\documentclass{article}

\usepackage{amsmath, amsthm, amssymb,float}
\usepackage{amsfonts}
\usepackage[linesnumbered,ruled]{algorithm2e}

\usepackage{graphicx}
\usepackage{caption}
\usepackage{subcaption}
\usepackage{centernot}

\usepackage{textcomp}
\usepackage{cases}
\usepackage{braket}
\usepackage[normalem]{ulem}

\usepackage{chngcntr}

\usepackage{xspace}

\usepackage{pifont}

\usepackage[margin=2.5cm]{geometry}

\usepackage[colorlinks=true, urlcolor=blue, anchorcolor=blue, citecolor=blue,filecolor=blue,linkcolor=blue,menucolor=blue,pagecolor=blue]{hyperref}

\usepackage[noabbrev,capitalise]{cleveref}
\crefformat{equation}{(#2#1#3)}

\theoremstyle{theorem}
\newtheorem{theorem}{Theorem}
\numberwithin{theorem}{section}
\theoremstyle{lemma}
\newtheorem{lemma}[theorem]{Lemma}
\theoremstyle{remark}
\newtheorem{remark}[theorem]{Remark}
\theoremstyle{corollary}

\theoremstyle{proposition}
\newtheorem{proposition}[theorem]{Proposition}
\theoremstyle{definition}

\theoremstyle{claim}

\newcommand{\PP}{\mathbb{P}} 
\newcommand{\E}{\mathbb{E}} 

\newcommand{\MIRLS}{{\texttt{Mix-IRLS}}\xspace}
\newcommand{\MIRLST}{{\texttt{Mix-IRLS:tuned}}\xspace}
\newcommand{\EM}{\texttt{EM}\xspace}
\newcommand{\AltMin}{\texttt{AltMin}\xspace}
\newcommand{\GD}{\texttt{GD}\xspace}
\newcommand{\ILTS}{\texttt{ILTS}\xspace}

\begin{document}

\title{Imbalanced Mixed Linear Regression}
\author{Pini Zilber \footnotemark[1]\thanks{Faculty of Mathematics and Computer Science, Weizmann Institute of Science \newline (\href{mailto:pini.zilber@weizmann.ac.il}{pini.zilber@weizmann.ac.il}, \href{mailto:boaz.nadler@weizmann.ac.il}{boaz.nadler@weizmann.ac.il})}
\and Boaz Nadler \footnotemark[1]}
\date{}

\maketitle
\nopagebreak

\begin{abstract}
We consider the problem of mixed linear regression (MLR), where each observed sample belongs to one of $K$ unknown linear models. 
In practical applications, the proportions of the $K$ components are often imbalanced.
Unfortunately, most MLR methods do not perform well in such settings. 
Motivated by this practical challenge, in this work we propose \MIRLS, 
a novel, simple and fast algorithm for MLR with excellent performance on both balanced and imbalanced mixtures.
In contrast to popular approaches that recover the $K$ models simultaneously, \MIRLS does it sequentially using tools from robust regression. 
Empirically, \MIRLS succeeds in a broad range of settings where other methods fail.
These include imbalanced mixtures, small sample sizes, presence of outliers, and an unknown
number of models $K$. In addition, \MIRLS outperforms competing methods on several real-world datasets, in some cases by a large margin. 
We complement our empirical results by deriving a recovery guarantee for
\MIRLS, which highlights its advantage on imbalanced mixtures.
\end{abstract}

{\footnotesize keywords: mixture regression model, mixture of linear models, robust regression, iteratively reweighted least squares}

\section{Introduction}
In this paper we consider a simple generalization of the linear regression problem, known as mixed linear regression (MLR) \cite[Chapter~14]{bishop2006pattern}. In MLR, each sample belongs to one of $K$ unknown linear models, and it is not known to which one. MLR can thus be viewed as a combination of linear regression and clustering.
Despite its simplicity, the presence of multiple linear components makes MLR highly expressive and thus a useful model for data representation in various applications, including
trajectory clustering \cite{gaffney1999trajectory},
health care analysis \cite{deb2000estimates},
market segmentation \cite{wedel2000market},
face recognition \cite{chai2007locally},
population clustering \cite{ingrassia2014model},
drug sensitivity prediction \cite{li2019drug}
and relating genes to disease phenotypes \cite{chang2021supervised,sun2022robust}.

Several methods were developed to solve MLR, including expectation maximization \cite{de1989mixtures,bishop2006pattern}, alternating minimization  \cite{yi2014alternating,yi2016solving} and gradient descent \cite{zhong2016mixed}.
These methods share three common features: they all (i) require as input the number of components $K$; (ii) estimate the $K$ models simultaneously; and (iii) tend to perform better on balanced mixtures, where the proportions of the $K$ models are approximately equal.
As illustrated in \cref{sec:simulations}, given data from an imbalanced mixture, these methods may fail.
In addition, most of the theoretical guarantees in the literature assume a balanced mixture.
Since imbalanced mixtures are ubiquitous in applications, it is of practical interest to develop MLR methods able to handle such settings, as well as corresponding recovery guarantees.

\begin{figure*}[ht]
\centering
	\includegraphics[width=1\linewidth]{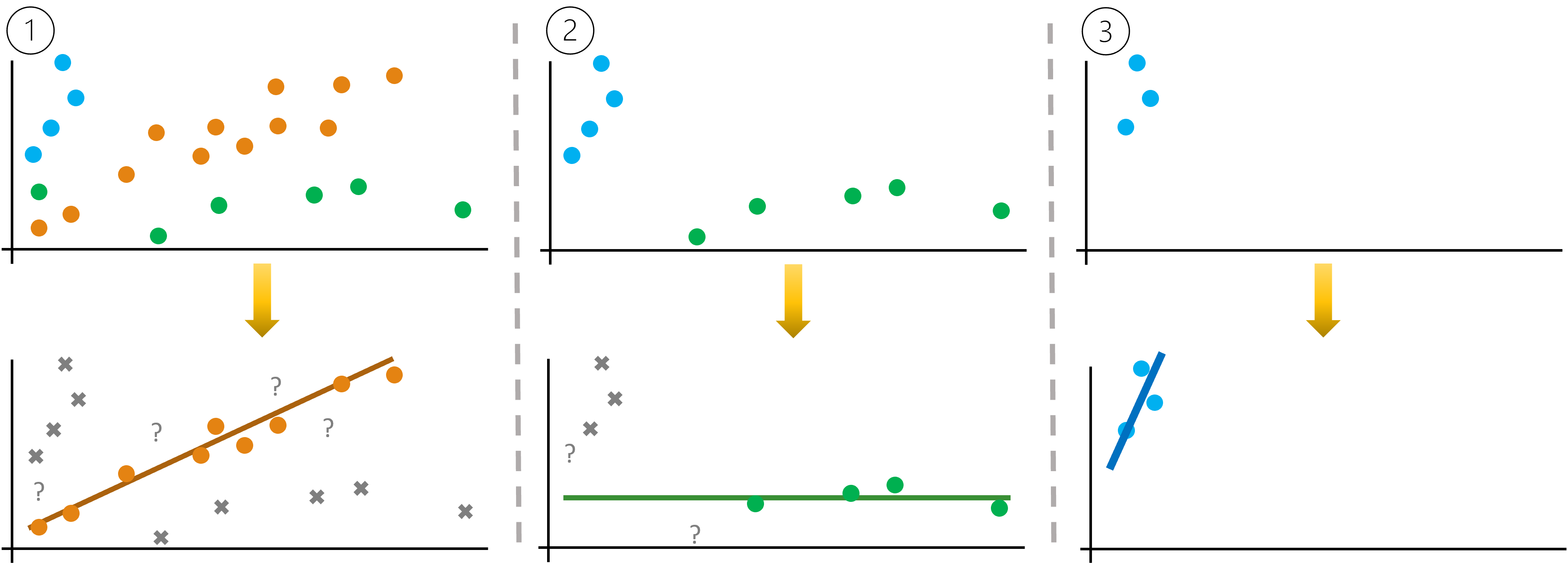}
	\caption{Illustration of \MIRLS. The data is a mixture of $K=3$ linear components. At each step, \MIRLS excludes samples with poor fit (marked with 'X') and with moderate fit (marked with '?'), and performs linear regression on the good-fit samples. The poor-fit samples are passed to the next step.}
\label{fig:MIRLS_illustration}
\end{figure*}

In this paper, we present \MIRLS, a novel and conceptually different iterative algorithm for MLR, able to handle both balanced and imbalanced mixtures.
\MIRLS is computationally efficient, simple to implement, and scalable to large problems.
In addition, \MIRLS can successfully recover the underlying components with only relatively few samples, is robust to noise and outliers, and does not require as input the number of components $K$.
In \cref{sec:simulations,sec:experiments} we illustrate the competitive advantage of \MIRLS over several other methods, on both synthetic and real data.

To motivate our approach, let us consider a highly imbalanced mixture, where most samples belong to one model. In this case, the samples that belong to the other models may be regarded as {\em outliers} with respect to the dominant one. The problem of finding the dominant model may thus be viewed as a specific case of robust linear regression, a well studied problem with a variety of effective solutions, see e.g.~\cite{huber1981robust,wilcox2011introduction}.
After finding the dominant model, we may remove its associated samples from the observation set and repeat the process to find the next dominant model. This way, the $K$ linear models are found {\em sequentially} rather than simultaneously as in the aforementioned methods.
This key difference makes our approach applicable to scenarios with imbalanced mixtures, and does not require to know a-priori the number of components.

As illustrated in \cref{fig:MIRLS_illustration}, to
facilitate the sequential recovery of \MIRLS, we found it important to allow in its intermediate steps an "I don't know" assignment to some of the samples.
Specifically, given coefficient estimates of the current model, we partition the samples to three classes, according to their fit to the found model: good, moderate and poor fit. The samples with good fit are used to re-estimate the model coefficients; those with poor fit are assumed to belong to a yet undiscovered model, and hence are passed to the next step; the moderate-fit samples, on whose model identity we have only low confidence ("I don't know"), are ignored, but used later in a refinement phase.

At each step, we perform robust regression to recover the currently dominant model.
A popular robust regression technique is iteratively reweighted least squares (IRLS) \cite{holland1977robust,chatterjee1997robust,mukhoty2019globally}. IRLS iteratively solves weighted least squares subproblems,
where the weight of each sample depends on its residual with respect to the current model estimate.
As the iterations progress, outliers are hopefully assigned smaller and smaller weights, and ultimately ignored.

On the theoretical front, in \cref{sec:theory} we present a recovery guarantee for our method. Specifically, we show that in a population setting with an imbalanced mixture of two components, \MIRLS successfully recovers the linear models. A key novelty in our analysis is that it holds for a sufficiently imbalanced mixture rather than a sufficiently balanced one (or even a perfectly balanced one) as is common in the literature \cite{yi2014alternating,balakrishnan2017statistical,kwon2021minimax}. In addition, unlike most available guarantees, our result allows an unknown $K$, and it is insensitive to the initialization, allowing it to be arbitrary.

To the best of our knowledge, our work is the first to specifically handle imbalance in the MLR problem, providing both a practical algorithm as well as a theoretical recovery guarantee.
The basic idea of \MIRLS, sequential recovery using robust regression, was also proposed by \cite{banks2009cherry}. They used random sample consensus (RANSAC) approach instead of IRLS, and without our "I don't know" concept. To find a component, \cite{banks2009cherry} randomly pick $(d+2)$ samples from the data and run ordinary least squares (OLS) on them, in hope that they all belong to the same component. As discussed by the authors, their approach is feasible only in low dimensional settings, as the probability that all chosen samples belong to the same component decreases exponentially with the dimension $d$; specifically, the authors studied only cases with $d\leq 5$. In addition, the authors did not provide a theoretical guarantee for their method. In contrast, our \MIRLS method is scalable to high dimensions, and theoretically justified.

\paragraph{Notation.}
For a positive integer $K$, denote $[K] = \{1, \ldots, K\}$, and the set of all permutations over $[K]$ by $[K]!$.
For a vector $u$, denote its Euclidean norm by $\|u\|$.
For a matrix $X$, denote its operator norm (a.k.a.~spectral norm) by $\|X\|$ and its smallest singular value by $\sigma_\text{min}(X)$.
Given a matrix $X\in \mathbb R^{n\times d}$ and an index set $S \subseteq [n]$, $X_S \in \mathbb R^{|S|\times d}$ is the submatrix of $X$ that corresponds to the rows in $S$.
Denote by $\text{diag}(w)$ the diagonal matrix whose entries are $W_{ii} = w_i$ and $W_{ij} = 0$ for $i\neq j$. Denote the probability of an event $A$ by $\PP[A]$. Denote the expectation and the variance of a random variable $x$ by $\E[x]$ and $\text{Var}[x]$, respectively. The cumulative distribution function of the standard normal distribution $\mathcal N(0,1)$ is $\Phi$.

\section{Problem Setup}
Let $\{(x_i, y_i)\}_{i=1}^n$ be $n$ pairs of explanatory variables $x_i \in \mathbb R^d$ and corresponding responses $y_i \in \mathbb R$.
In standard linear regression, one assumes a linear relation between the response and the explanatory variables, namely $y_i = x_i^\top \beta^* + \epsilon_i$ where $\epsilon_i\in \mathbb R$ are random noise terms with zero mean. A common goal is to estimate the vector of regression coefficients $\beta^* \in \mathbb R^d$.
In mixed linear regression (MLR), in contrast, the assumption is that each response $y_i$ belongs to one of $K$ different linear models $\{\beta^*_k\}_{k=1}^K$. Formally,
\begin{align}\label{eq:MLR}
y_i &= x_i^\top \beta^*_{c^*_i} + \epsilon_i, \quad i=1, \ldots, n,
\end{align}
where $c^* = (c_1^*, \ldots, c_n^*)^\top \in [K]^n$ is the label vector.
Importantly, we do not know to which component each pair $(x_i, y_i)$ belongs, namely $c^*$ is unknown.
For simplicity, we assume the number of components $K$ is known, and later on discuss the case where it is unknown.
Given the $n$ samples $\{(x_i, y_i)\}_{i=1}^n$, the goal is to estimate $\beta^* \equiv \{\beta_1^*, \dotsc, \beta_K^*\} \subset \mathbb R^d$, possibly by concurrently estimating $c^*$.
See \cref{fig:tone_perception} for a real-data visualization of MLR in the simplest setting of $d=1$ and $K=2$.

To make the recovery of the regression vectors $\beta^*$ feasible, sufficiently many samples must be observed. The minimal number of samples depends on the dimension and the mixture proportions.
Denote the vector of mixture proportions by $p = (p_1, \ldots, p_K)$, with $p_k = |\{i\in [n] : c^*_i = k\}|$.
Then the information limit on the sample size, namely the minimal number of observations required to make $\beta^*$ identifiable in the absence of noise, is $n_\text{inf} = d/\min(p)$.

The lack of knowledge of the labels $c^*$ makes MLR significantly more challenging than standard linear regression. Even in the simplified setting of $K=2$ with perfect balance ($p_1=p_2=1/2$) and no noise ($\epsilon = 0$), the problem is NP-hard without further assumptions \cite{yi2014alternating}.

\section{The Mix-IRLS Method}\label{sec:method}
For simplicity, we present our algorithm assuming $K$ is known; the case of an unknown $K$ is discussed in \cref{remark:unknown_K} below.
\MIRLS consists of two phases.
In its first (main) phase, \MIRLS sequentially recovers each of the $K$ components $\beta^*_1, \ldots, \beta^*_K$ by treating the remaining components as outliers. The sequential recovery is the core idea that distinguishes \MIRLS from most other methods. In the second phase, we refine the estimates of the first phase by optimizing them simultaneously, similar to existing methods.
As discussed below, accurate estimates are often already found in the first phase, in which case the second phase is unneeded.
For brevity, we defer the description of the second phase to \cref{sec:method_phaseII}.

Before we dive into details, let us give a brief overview of the main phase mechanism. At each round of the main phase, \MIRLS estimates the largest component present in the data using techniques from robust regression. Then, it partitions the samples to three subsets, according to their fit to the found component: good, moderate and poor. \MIRLS refines the component estimate using only the samples with good fit, and proceeds to the next round with the poor fit samples. The moderate fit samples are ignored in the main phase, as we have low confidence in their component assignment - they either may or may not belong to the found component.
The partition of the samples at each round is performed with the aid of two parameters: a threshold $0 \leq w_\text{th} < 1$ and an oversampling ratio $\rho \geq 1$.

\begin{algorithm}[tb]
\caption{\MIRLS: main phase} \label{alg:MIRLS_phaseI}
\SetKwInOut{Input}{input}
\SetKwInOut{Output}{output}
\Input{samples $\{(x_i, y_i)\}_{i=1}^n$, number of components $K$, parameters $w_\text{th}$, $\rho$, $\eta$, $T_1$}
\Output{estimates $\beta_1, \ldots, \beta_K$}
set $S_1 = [n]$ \\
\For{$k=1$ {\bfseries to} $K$}
{
	initialize $\beta_k$ randomly \\
	\For{$t=1$ {\bfseries to} $T_1$}
	{
		compute $r_{i,k} = |x_i^\top \beta_k - y_i|, \quad \forall i\in S_k$ \\
		compute $w_{i,k} = (1 + \eta r_{i,k}^2 / \bar{r}_k^2)^{-1}, \quad \forall i\in S_k$ \\
		compute $\beta_k = (X_{S_k}^\top W_k X_{S_k})^{-1} X_{S_k}^\top W_k\, y_{S_k}$
	}
	set $S_{k+1} = \{i\in S_k : w_{i,k} \leq w_\text{th} \}$ \\
	set $S'_k = \{\rho \cdot d \text{ samples in } S_k \text{ with largest } w_{i,k} \}$ \\
	\If{$k < K$ and $|S_{k+1}| < \rho\cdot d$ \label{alg:resetting_criterion}}
	{
		start \MIRLS over with $w_\text{th} \leftarrow w_\text{th} + 0.1$
	}
	compute $\beta_k = (X_{S'_k}^\top X_{S'_k})^{-1} X_{S'_k}^\top y_{S'_k}$
}
\end{algorithm}

Next, we give a detailed description of the main phase of \MIRLS. A pseudocode appears in \cref{alg:MIRLS_phaseI}.
We begin by initializing the set of active samples as the entire dataset, $S_1 = [n]$.
Next, we perform the following procedure for $K$ rounds.
At the beginning of round $k \in [K]$, we start from a random guess $\beta_k$ for the $k$-th vector.
Then, we run the following IRLS scheme for $T_1$ iterations:
\begin{subequations}\label{eq:IRLS}\begin{alignat}{3}
r_{i,k} &= \left|x_i^\top \beta_k - y_i\right|, \, \forall i\in S_k, &\text{(residuals)} \label{eq:IRLS_r} \\
w_{i,k} &= \frac{1}{1 + \eta \cdot r_{i,k}^2 / \bar r_k^2}, \, \forall i\in S_k, &\text{(weights)} \label{eq:IRLS_W} \\
\beta_k &= (X_{S_k}^\top W_k X_{S_k})^{-1} X_{S_k}^\top W_k\, y_{S_k}, \quad &\text{(estimate)} \label{eq:IRLS_beta}
\end{alignat}\end{subequations}
where $X = \begin{pmatrix} x_1 & \cdots & x_n \end{pmatrix}^\top$ and $y = \begin{pmatrix} y_1 & \cdots & y_n \end{pmatrix}^\top$, $\bar r_k = \text{median}\{r_{i,k} \mid i\in S_k\}$, $\eta>0$ is a parameter of \MIRLS, and $W_k = \text{diag}(w_{1,k}, w_{2,k}, \dotsc)$.
The estimate \eqref{eq:IRLS_beta} is the minimizer of the weighted objective $\|W^\frac{1}{2} \left(y_{S_k} - X_{S_k}\beta_k\right) \|^2$.
After $T_1$ iterations of \cref{eq:IRLS_r,eq:IRLS_W,eq:IRLS_beta}, we define the subset $S_{k+1}$ of 'poor fit' samples that seem to belong to another component,
\begin{subequations}\label{eq:index_subsets}
\begin{align}\label{eq:index_subsets_Sk+1}
S_{k+1} = \left\{i\in S_k : w_{i,k} \leq w_\text{th} \right\} .
\end{align}
This serves as the set of active samples for the next round.
In addition, we define the subset $S_k'$ of 'good fit' samples that seem to belong to the $k$-th component,
\begin{align}\label{eq:index_subsets_Sk'}
S'_k &= \left\{\rho\cdot d \text{ samples in } S_k \text{ with largest weights } w_{i,k} \right\}.
\end{align}
\end{subequations}
This subset is used to refine the estimate by performing OLS,
\begin{align}\label{eq:phase1_beta}
\beta_k^\text{phase-I} = (X_{S'_k}^\top X_{S'_k})^{-1} X_{S'_k}^\top y_{S'_k}.
\end{align}
The subsets $S_k', S_{k+1}$ are in general disjoint, unless the threshold $w_\text{th}$ is too low or the oversampling ratio $\rho$ is too large.
The choice for the value of $w_\text{th}$ is discussed in \cref{remark:reset_alg,sec:theory}.
A suitable value for $\rho$ depends on the ratio between the sample size $n$ and the information limit $n_\text{inf} = d/\min(p)$. In the challenging setting of $n \approx n_\text{inf}$, $\rho$ should be set close to $1$; otherwise, \MIRLS would reach one of the $K$ rounds with less than $d$ active samples, making the recovery of the yet undiscovered components impossible. If $n \gg n_\text{inf}$, then $\rho$ can be set to a higher value.

This concludes the main phase of \MIRLS.
The second (refinement) phase, described in \cref{sec:method_phaseII}, improves the estimates $\beta_k^\text{phase-I}$ using also the moderate-fit samples that were ignored in the first phase.
Yet, in many cases, empirically, the main phase is sufficient to accurately recover the components.
This is theoretically established in \cref{sec:theory}.

\begin{remark}[Parameter tuning]\label{remark:parameters}
\MIRLS has four input parameters:
$\eta, w_\text{th}, \rho$ and $T_1$.
As empirically demonstrated in \cref{sec:simulations,sec:experiments}, there is no need to carefully tune these parameters, as in a wide range of settings including both synthetic and real-world data, \MIRLS performs well with a fixed set of values, specified in \cref{sec:additional_details}.
Moreover, in many cases, tuning \MIRLS parameters only slightly improves its performance.
\end{remark}

\begin{remark}[Threshold adaptation]\label{remark:reset_alg}
If the input threshold $w_\text{th}$ is too low, the algorithm will fail to detect all the components, as the number of 'poor fit' samples passed to the next round in \eqref{eq:index_subsets_Sk+1} is too small. This happens when $|S_k| < \rho d$ for some $k$, as in this case there are not enough samples to confidently determine the $k$-th component using Eq.~\eqref{eq:index_subsets_Sk'}.
To handle this case, we increase the value of $w_\text{th}$ by $0.1$, and start \MIRLS over; see \cref{alg:resetting_criterion} in \cref{alg:MIRLS_phaseI}.
\end{remark}

\begin{remark}[Unknown/overestimated $K$]\label{remark:unknown_K}
If $K$ is unknown, or only an upper bound $K_\text{max}$ is given, we ignore the resetting criterion (\cref{alg:resetting_criterion} in \cref{alg:MIRLS_phaseI}), and instead run phase I until there are too few samples to estimate the next component, namely $|S_{k+1}| < \rho d$.
We then proceed to the second phase with $K$ set to the number of components with at least $\rho d$ associated samples.
\end{remark}

\section{Simulation Results}\label{sec:simulations}
We present simulation results on synthetic data in this section, and on real-world data in the next one. In both sections, we compare the performance of \MIRLS to the following three algorithms: (i) \AltMin\xspace - alternating minimization \cite{yi2014alternating,yi2016solving};
(ii) \EM\xspace - expectation maximization \cite[Chapter~14]{faria2010fitting,bishop2006pattern};
and (iii) \GD\xspace - gradient descent on a factorized objective \cite{zhong2016mixed}.
We implemented all methods in MATLAB.\footnote{MATLAB and Python code implementations of \MIRLS are available at \url{github.com/pizilber/MLR}.}
In some of the simulations, we additionally ran a version of \EM for which the mixture proportions $p$ are given as prior knowledge, but it hardly improved its performance and we did not include it in our results.

\begin{figure}[t]
\centering
	\subfloat{
		\includegraphics[width=0.5\linewidth]{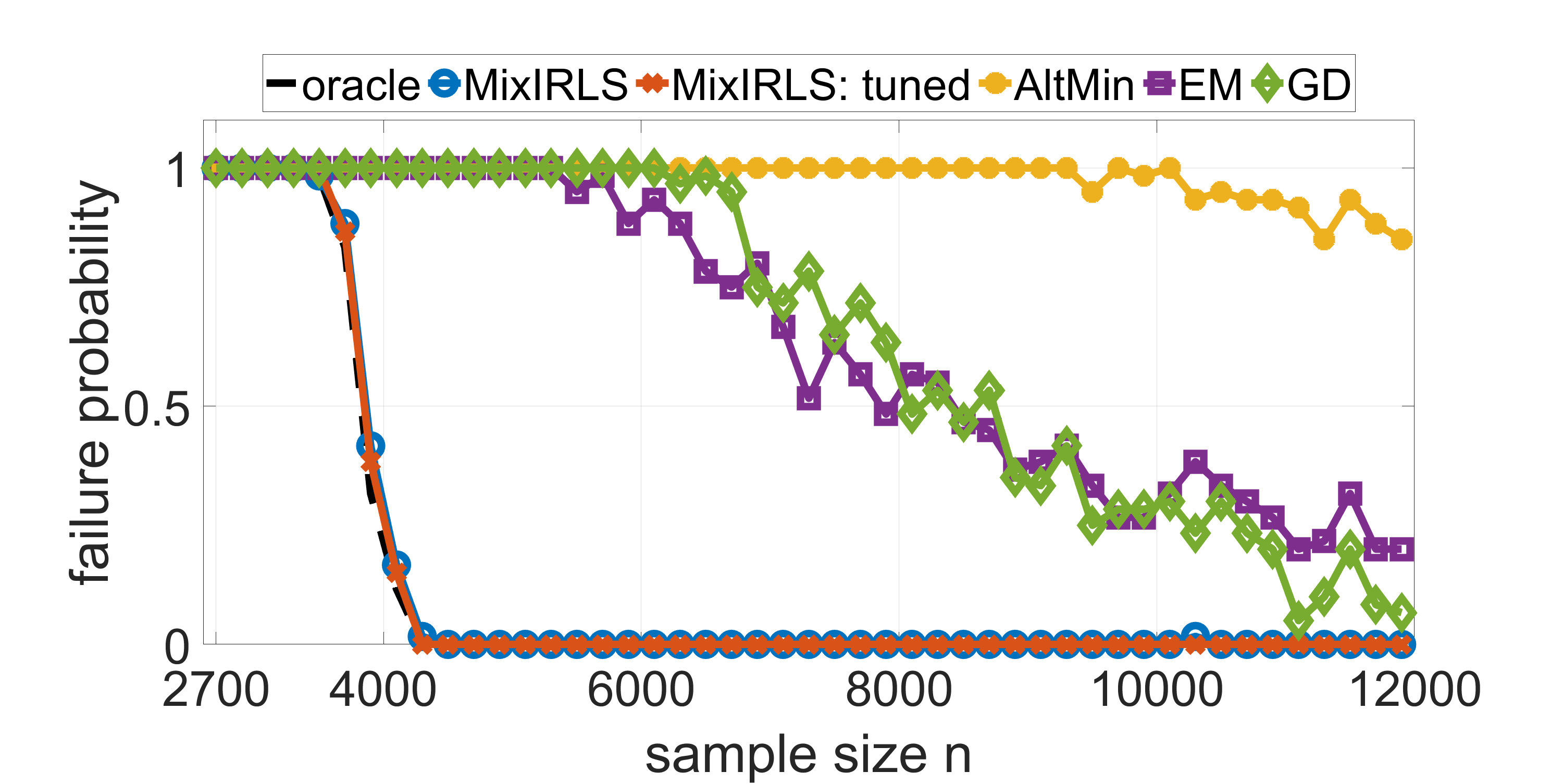}
	}
	\subfloat{
		\includegraphics[width=0.5\linewidth]{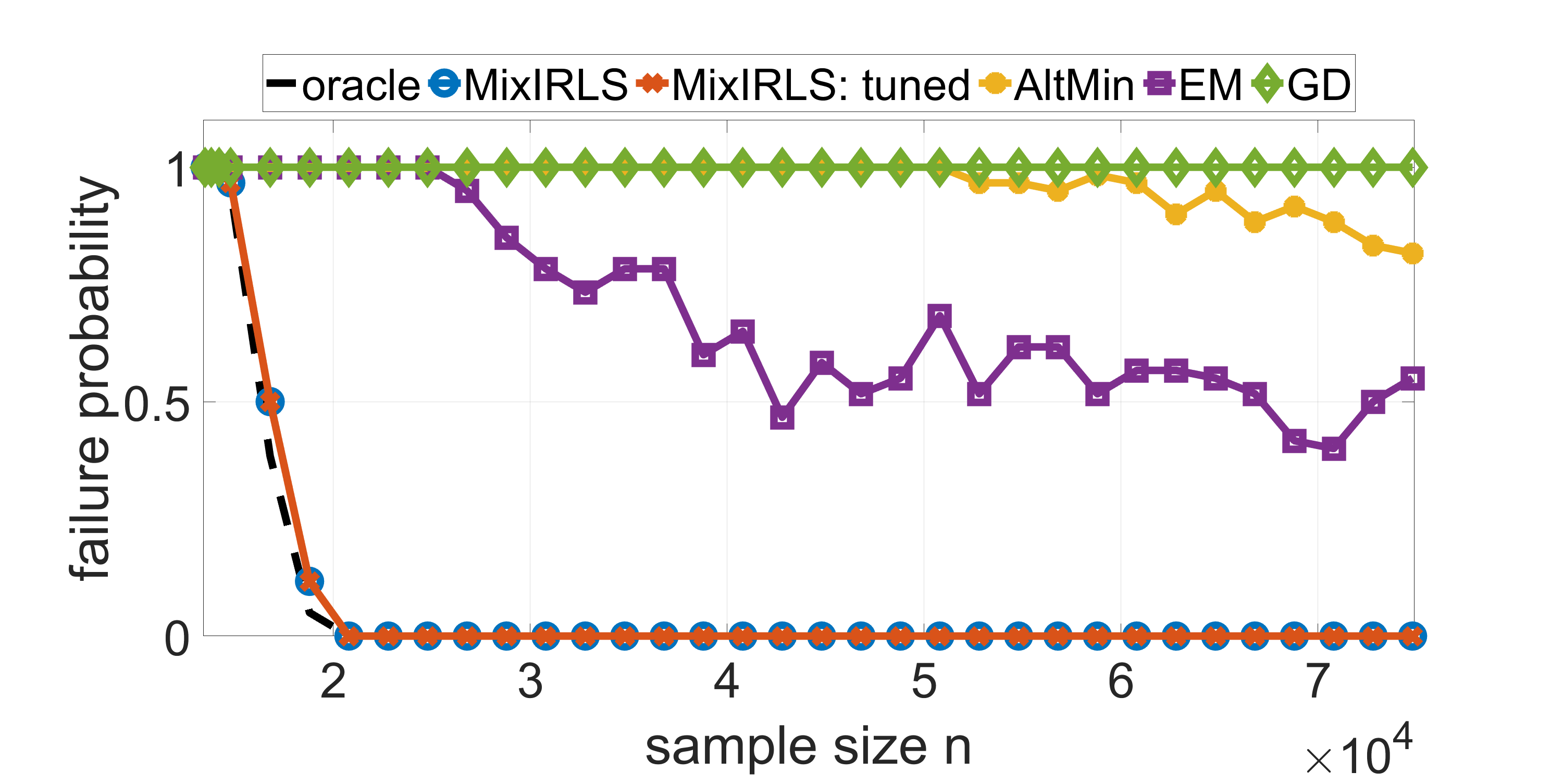}
	}
	\caption{A comparison of various MLR algorithm. Depicted is the percentage of runs, out of 50 random initializations, for which $F_\text{latent} > 2\sigma$ (see \eqref{eq:objective_latent}), as a function of the sample size $n$. The dimension and noise level are fixed at $d=300$ and $\sigma = 10^{-2}$. Mixture: $K=3$ with $p = (0.7, 0.2, 0.1)$ (left panel); $K = 5$ with $p = (0.63, 0.2, 0.1, 0.05, 0.02)$ (right panel).}
\label{fig:varying_n}
\end{figure}

All methods start from the same random initialization, as described shortly. Other initializations did not change the qualitative results; see \cref{sec:additional_details} for more details.
For \EM and \GD we added a single \AltMin refinement step at the end of each algorithm to improve their estimates.
In all simulations, we also run a version of \MIRLS with tuned parameters, denoted \MIRLST. In addition, we plot the performance of an oracle which is provided with the true labels $c^*$ and separately computes the OLS solution for each component.
Details on maximal number of iterations, early stopping criteria and parameter tuning of \MIRLST and \GD appear in \cref{sec:additional_details}.

Similar to \cite{zhong2016mixed,ghosh2020alternating}, in each simulation we sample the entries of the explanatory variables $X$ and of the regression vectors $\beta^*$ from the standard normal distribution $\mathcal N(0,1)$. In this section, the dimension is fixed at $d = 300$. The entries of the additive noise term $\epsilon$ are sampled from a normal distribution with zero mean and standard deviation $\sigma = 10^{-2}$. Additional simulations with other values of $d$ and $\sigma$ appear in \cref{sec:additional_simulations}.
As discussed in the introduction, a central motivation for the development of \MIRLS is dealing with imbalanced mixtures. Hence, in this section, the labels $c^*_i$ follow a multinomial distribution with highly imbalanced proportions. In \cref{sec:additional_simulations}, we present results for balanced and moderately imbalanced mixtures.

We measure the accuracy of an estimate $\beta \equiv \{\beta_1, \hdots, \beta_K\}$ by the following quantity:
\begin{align}\label{eq:objective_latent}
F_\text{latent}(\beta; \beta^*) &= \min_{\sigma\in [K]!} \frac 1K \sum_{k=1}^K \|\beta_{\sigma(k)} - \beta_k^*\|.
\end{align}
The minimization over all $K!$ permutations makes the accuracy measure invariant to the order of the regression vectors in $\beta$.
A similar objective was used by \cite{yi2016solving,zhong2016mixed}.

All algorithms are initialized with the same vector $\beta$, whose entries are sampled from the standard normal distribution $\mathcal N(0,1)$.
For each simulation, we run 50 independent realizations, each with a different random initialization $\beta$, and report the median errors and the failure probability. The latter quantity is defined as the percentage of runs whose error $F_\text{latent}$ \eqref{eq:objective_latent} is above $2\sigma$; see an explanation for this choice in \cref{sec:additional_details}.
Due to space considerations, some figures appear in \cref{sec:additional_simulations}.

In the first simulation, we examine the performance of the algorithms as a function of the sample size $n$. The results are depicted in \cref{fig:varying_n}, and the corresponding runtimes in \cref{fig:varying_n_time} (\cref{sec:additional_simulations}).
Given moderate sample sizes, all competing methods get stuck in bad local minima.
Importantly, this behavior is not due to the presence of noise, and as shown in \cref{sec:additional_simulations}, the qualitative result does not change in a noiseless setting.
\MIRLS, in contrast, recovers the components with sample size very close to the oracle's minimum.
Moreover, as shown in \cref{fig:heatmap} (\cref{sec:additional_simulations}), the nearly optimal performance of \MIRLS is invariant to the dimension $d$. 
It does depend, however, on the mixture proportions: for a moderately imbalanced mixture, the oracle performs reasonably better than \MIRLS.
Yet, even in this case \MIRLS markedly outperforms the other methods; see \cref{fig:varying_n_K5_moderateImbalance} (\cref{sec:additional_simulations}).

\begin{figure}[t]
\centering
	\subfloat{
		\includegraphics[width=0.5\linewidth]{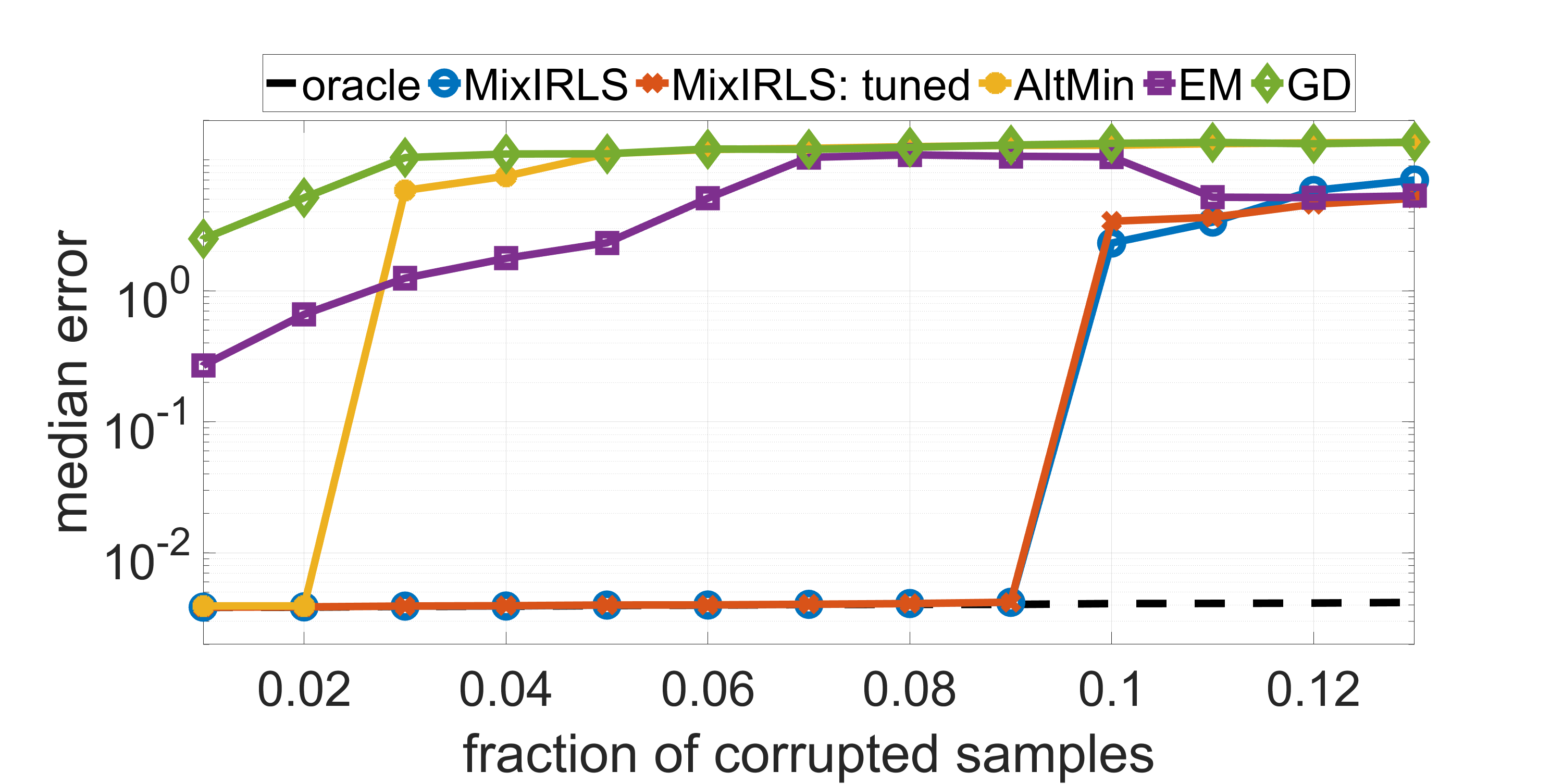}
	}
	\subfloat{
		\includegraphics[width=0.5\linewidth]{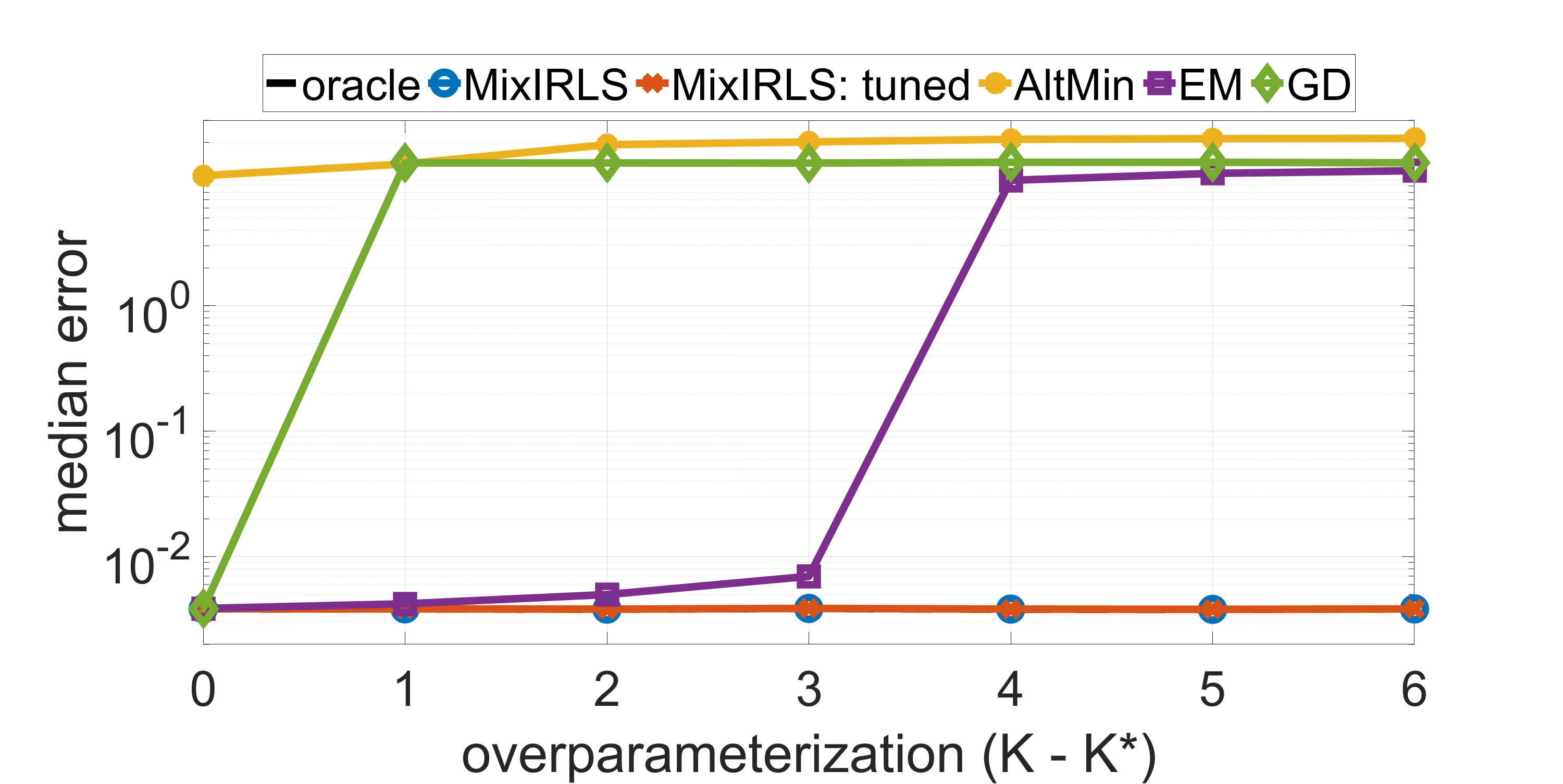}
	}
	\caption{Comparison of the robustness of several MLR algorithms to outliers and to overparameterization, for the same values of $d, \sigma, K$ and $p$ as in \cref{fig:varying_n}(left), and $n = 12000$. Qualitatively similar results for a balanced mixture appear in \cref{fig:balanced_corruptions,fig:balanced_overparam} (\cref{sec:additional_simulations}). X-axis: fraction of outliers (left panel); the difference between the number of components $K$ given to the algorithms and the true $K^*$ (right panel).}
\label{fig:robustness}
\end{figure}

Another observation from \cref{fig:varying_n} is the negligible advantage of the tuned \MIRLS variant compared to its untuned one. In fact, in all our simulations the tuning variant had only a small advantage over the untuned one,
implying that \MIRLS can be viewed as a tuning-free method.

Next, we explore the robustness of the algorithms to additive noise, outliers and overestimation.
\Cref{fig:noise} (\cref{sec:additional_simulations}) shows that all algorithms are stable to additive noise, but only \MIRLS matches the oracle performance in all runs.

To study robustness to outliers, in the following simulation we artificially corrupt a fraction $f\in (0,1)$ of the observations. A corrupted response $\tilde y_i$ is sampled from a normal distribution with zero mean and standard deviation $\sqrt{\tfrac{1}{n} \sum_{j=1}^n y_j^2}$, independently of the original value $y_i$.
\cref{fig:robustness}(left) shows the error of the algorithms as a function of the corruption fraction $f$.
To let the algorithms cope with outliers while keeping the comparison fair, we made the \textit{same} modification in all of them: at each iteration, the estimate $\beta$ is calculated based on the $\lceil (1-f)n \rceil$ samples with smallest residuals. In \MIRLS, we implemented this modification only in the second phase. 
As shown in \cref{sec:additional_simulations}, empirically, \MIRLS can deal with a corruption fraction of $f=0.09$, which is over $4$ times more corrupted samples than the other algorithms. In the balanced setting, \MIRLS can deal with roughly twice as large corrupted samples ($f=0.17$), which is almost 6 times more outliers than other methods.
This should not be surprising given that robust regression is at the heart of {\MIRLS}'s mechanism.

The final simulation considers the case of an unknown number of components. Specifically, the various algorithms are given as input a number $K$ equal to or larger than the true number $K^*$.
Here, the error is defined
similar to \eqref{eq:objective_latent}, but with $K^*$ instead of $K$.
Namely, the error is calculated based on the best $K^*$ vectors in $\beta$, ignoring its other $K-K^*$ vectors.
\Cref{fig:robustness}(right) shows that most algorithms have similar performance at the correct parameter value $K = K^*$, and \EM succeeds also at small overparameterization, $K-K^* \leq 3$.
Only \MIRLS, in both its tuned and untuned variants, is insensitive to the overparameterization, and succeeds with unbounded $K$. This feature is attained thanks to the sequential nature of \MIRLS (\cref{remark:unknown_K}). Similar results hold in the case of a balanced mixture; see \cref{sec:additional_simulations}.

\section{Real-World Datasets}\label{sec:experiments}
We begin with the classic problem of music perception, based on Cohen's standard dataset \cite{cohen1980influence}. In her thesis, Cohen investigated the human perception of tones by using newly available electronic equipment.
The $n=150$ data points acquired in her experiment are shown in \cref{fig:tone_perception}.
Cohen discussed two music perception theories:
One theory predicted that in this experimental setting, the perceived tone (y-axis) would be fixed at $2.0$, while the other theory predicted an identity function ($y=x$).
The results, depicted in \cref{fig:tone_perception}, support both theories.
As a mathematical formulation of this finding, Cohen proposed the MLR model \eqref{eq:MLR} with $K=2$, where the labels $c^*$ are i.i.d.~according to a Bernoulli distribution; see also \cite{de1989mixtures}.

\begin{figure}[t]
\centering
	\subfloat{
		\includegraphics[width=0.6\linewidth]{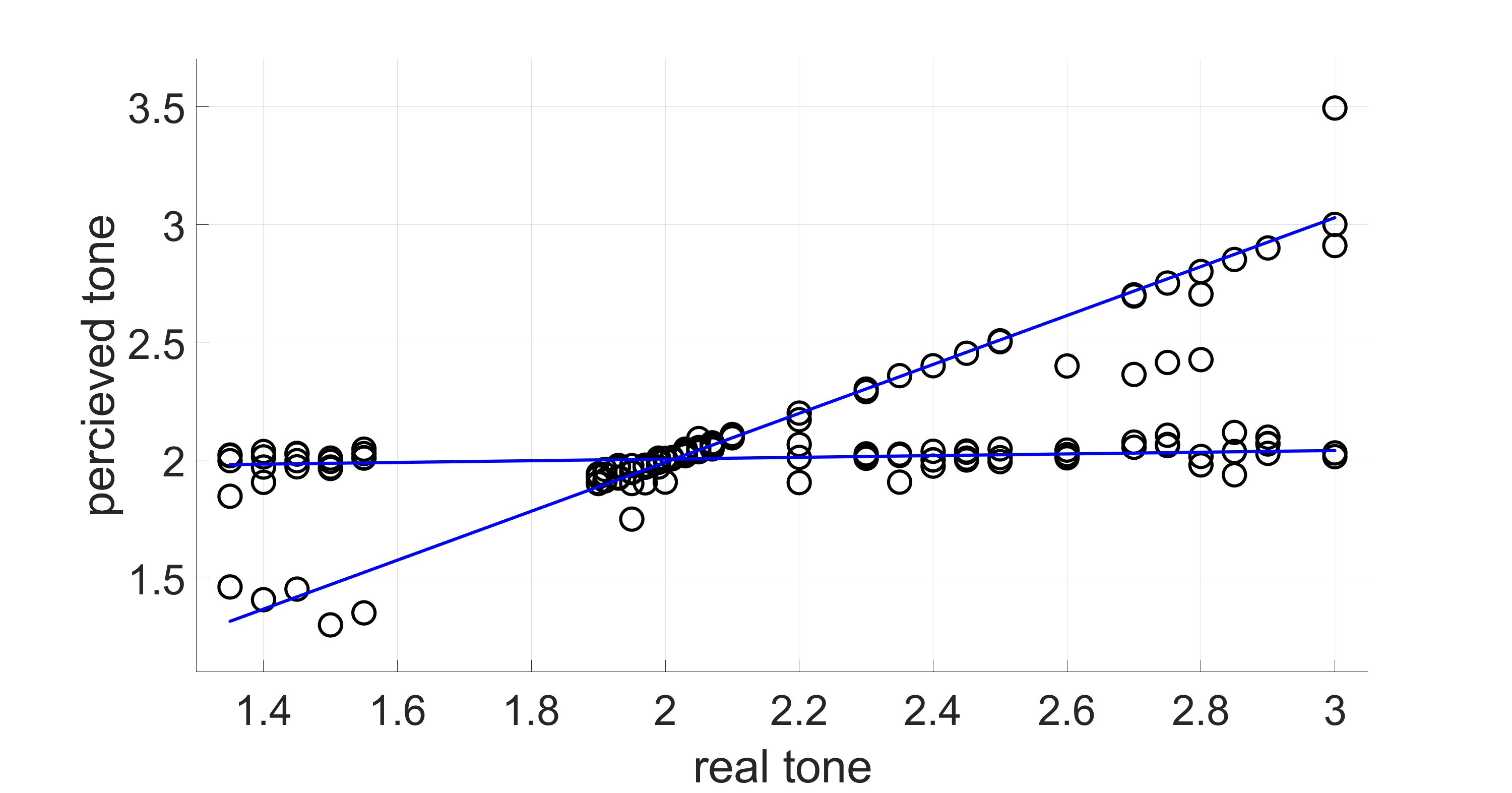}
	}
	\caption{Black circles: music perception data \cite{cohen1980influence} ($d=1$,  $n=150$). Blue lines: \MIRLS estimate.}
\label{fig:tone_perception}
\end{figure}

In \cref{fig:tone_perception}, the untuned version of \MIRLS is shown to capture the two linear trends in the data. Notably, untuned \MIRLS was not given the number of components, but automatically stopped at $K=2$ with its default parameters.
By increasing the sensitivity of \MIRLS to new components via the parameter $w_\text{th}$, it is possible to find $K=3$ or even more components; see \cref{fig:tone_perception_outliersK3} (\cref{sec:additional_experiments}).

Next, we compare the performance of \MIRLS to the algorithms listed in the previous section on four of the most popular benchmark datasets for multi-linear regression, all of which are available on Kaggle (see \cref{sec:additional_details}): medical insurance cost,
red wine quality,
World Health Organization (WHO) life expectancy,
and fish market.
The task in each dataset is to predict, respectively: medical insurance cost from demographic details; wine quality from its physicochemical properties; life expectancy from demographic and medical details; and fish weight from its dimensions.
For these datasets, MLR is at best an approximate model, and its regression vectors $\beta^*$ are unknown. Hence, we replace \eqref{eq:objective_latent} by the following quality measure which represents the fit of an MLR model to the data:
\begin{align}\label{eq:objective_observed}
F_\text{real}(\beta; X,y) &= \frac{1}{\text{Var}[y]} \cdot \frac 1n \sum_{i=1}^n \min_{j\in [K]} (x_i^\top \beta_j - y_i)^2 ,
\end{align}
resembling the K-means objective for clustering \cite{hastie2009elements}.
To illustrate the improvement of a multi-component model, we also report the error of a single component ordinary least squares (OLS) solution.

\begin{figure}[t]
\centering
	\subfloat{
		\includegraphics[width=0.5\linewidth]{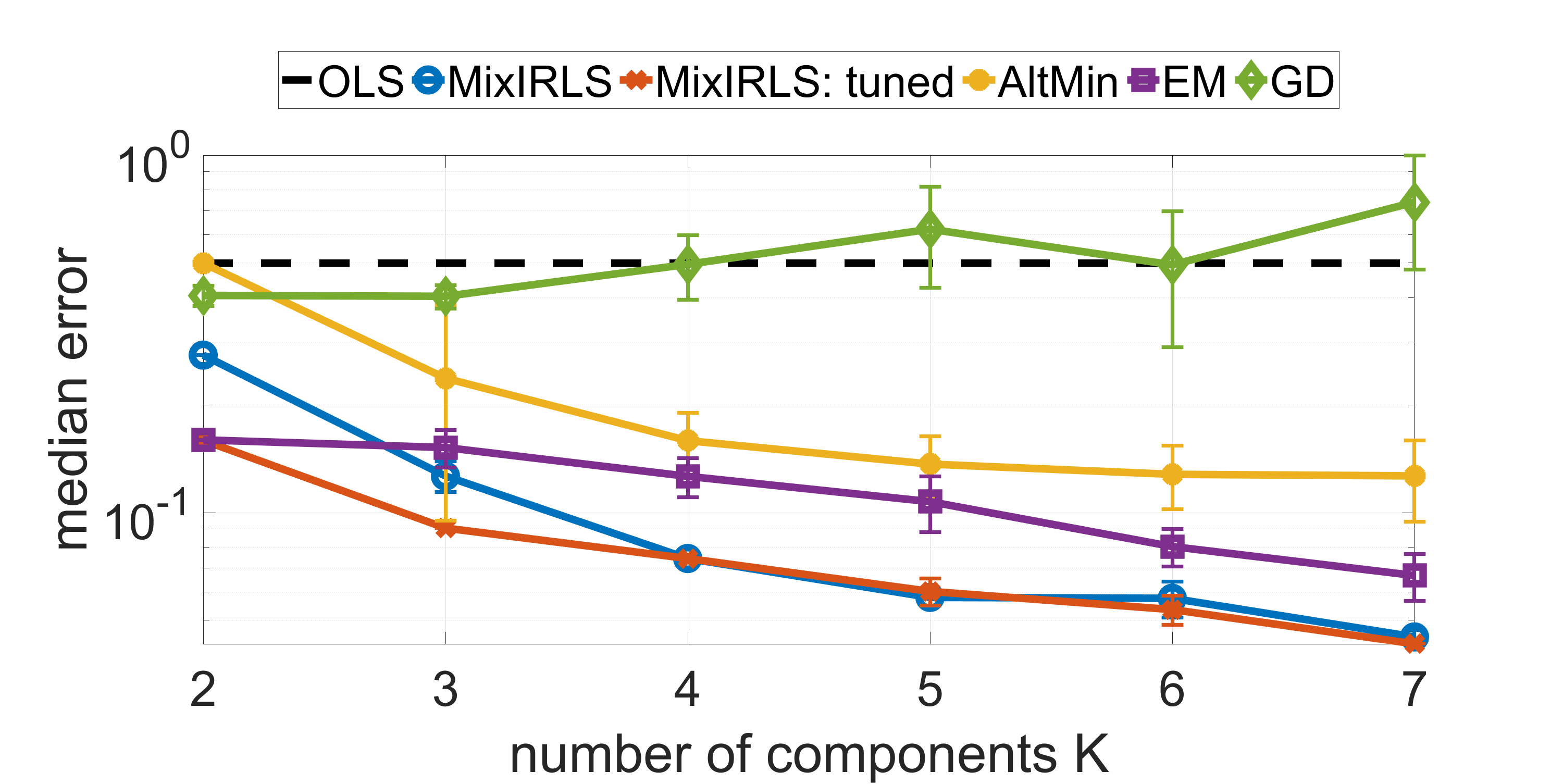}
	}
	\subfloat{
		\includegraphics[width=0.5\linewidth]{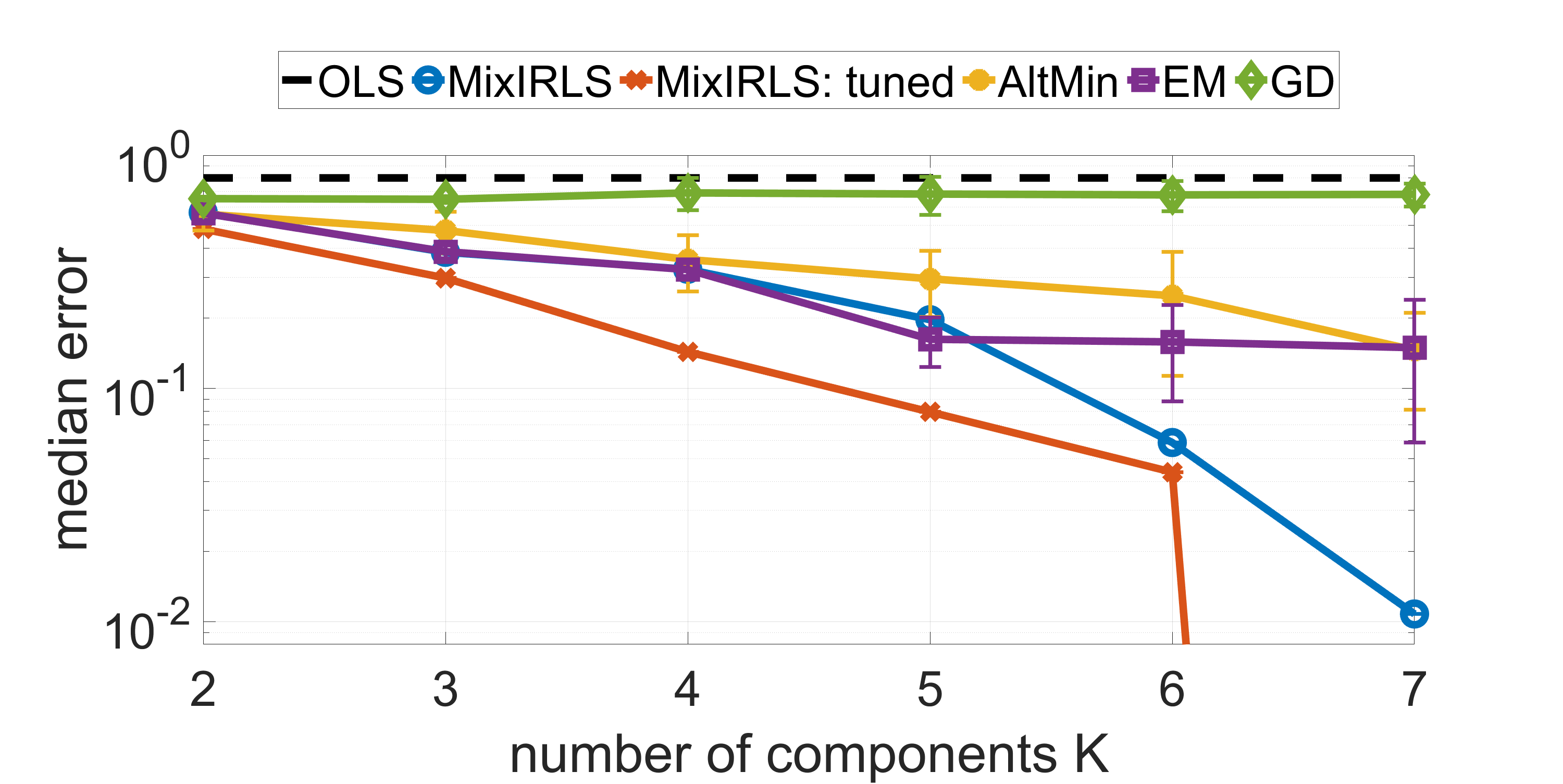}
	}
	\caption{Comparison of several MLR algorithms on the medical insurance (left) and the wine quality (right) datasets. Note that the y-axis is on a log scale. Median estimation errors are calculated across 50 random initializations according to \eqref{eq:objective_observed}, and error bars correspond to the median absolute deviation. While not shown in the figure, in the right panel the median error of \MIRLS with $K=7$ components is $6\cdot 10^{-15}$ (machine precision).}
\label{fig:medical_wine}
\end{figure}

As the number of components $K$ is unknown, we explore the algorithms' performance given different values of $K$, ranging from 2 to 7.
All algorithms start from the same random initialization.
In each experiment, we report the median error \eqref{eq:objective_observed} across 50 different random initializations; the minimal errors across the realizations are reported in \cref{sec:additional_experiments}.
Additional details appear in \cref{sec:additional_details}.

\Cref{fig:medical_wine} shows the performance of the algorithms on the medical insurance and the wine quality datasets. The results for the other two datasets are deferred to \cref{sec:additional_experiments}.
In general, both \MIRLS and \MIRLST improve upon the other methods, sometimes by 30\% or more. In addition, \MIRLST clearly outperforms \MIRLS. 
Most significantly is the case of the wine quality dataset with $K=7$ components. Unlike the other datasets, the response $y$ in this dataset is discrete, taking the values from 3 to 8. Hence, with $K\geq 6$, it is possible to perfectly fit an MLR model up to machine precision error. Notably, \MIRLST is the only algorithm that achieves this error with $K=7$ in at least half of the realizations.

\section{Recovery Guarantee for Mix-IRLS}\label{sec:theory}
In this section, we theoretically analyze \MIRLS in a population setting with an infinite number of samples. For simplicity, we assume the explanatory variables are normally distributed with identity covariance,
\begin{align}\label{eq:gaussian_assumption}
x_i \sim \mathcal N(0, I_d) .
\end{align}
The responses $y_i$ are assumed to follow model \eqref{eq:MLR} with $K=2$ components $(\beta_1^*, \beta_2^*)$ and labels $c^*_i$ generated independently of $x_i$ according to mixture proportions $p_1 \geq p_2 > 0$.
This setting was considered in several previous works on the \texttt{EM} method for MLR \cite{balakrishnan2017statistical,daskalakis2017ten, klusowski2019estimating,kwon2019global}.

We assume the noise terms $\epsilon_i$ are all i.i.d., zero-mean random variables, independent of $x_i$ and $c_i^*$. We further assume they are bounded and follow a symmetric distribution,
\begin{align}\label{eq:noise_assumptions}
|\epsilon_i| &\leq \sigma_\epsilon
\quad \text{and} \quad
\PP[\epsilon_i] = \PP[-\epsilon_i].
\end{align}

For analysis purposes, we consider a slightly modified variant of \MIRLS, described in \cref{sec:proof}. In this variant, \MIRLS excludes samples $x_i$ with large magnitude, $\|x_i\|^2 > R$ where $R$ is a fixed parameter. A natural choice for its value is $R \sim \E[\|x_i\|^2] = d$, e.g.~$R=2d$.
In a high-dimensional setting with $d\gg 1$, such a choice excludes an exponentially small in $d$ proportion of the samples.
For simplicity, we present our result assuming $R$ is large, corresponding to large dimension $d$; the general result appears in \cref{lem:K2_population}. 

The following theorem states that given a sufficiently imbalanced mixture, namely $p_2/p_1$ is small enough, \MIRLS successfully recovers the underlying vectors $\beta_1^*$ and $\beta_2^*$.

\begin{theorem}\label{thm:K2_population}
Let $\{(x_i,y_i)\}_{i=1}^\infty$ be i.i.d.~from a mixture of $K=2$ components with proportions $(p_1,p_2)$ and regression vectors $(\beta_1^*, \beta_2^*)$, and denote their separation $\Delta = \beta_1^* - \beta_2^*$.
Suppose the noise terms $\epsilon_i$ follow \eqref{eq:noise_assumptions} with $\sigma_\epsilon$ satisfying
\begin{align}\label{eq:bounded_balance_thm}
q \equiv \gamma + \left(\frac{1}{p_1} + \frac{1}{\sqrt R}\right) \frac{\sigma_\epsilon}{\|\Delta\|} < \frac 12,
\end{align}
where $\gamma = 5p_2 / (4p_1)$.
Assume that \MIRLS's parameters $\rho$, $w_\text{th}$ and $R$ satisfy $\rho=\infty$,
\begin{align}\label{eq:wth_assumption}
\frac{1}{1 + \eta (1-q)^2 \|\Delta\|^2} < w_\text{th} < \frac{1}{1 + \eta q^2 \|\Delta\|^2} ,
\end{align}
and that $R$ is sufficiently large.
Then starting from an arbitrary initialization, the first phase of \MIRLS with at least one iteration ($T_1\geq 1$) recovers $\beta^*$ up to an error that decreases with increasing $R$,
\begin{align}
\max_{k=1,2}\|\beta_k - \beta_k^*\| \leq \frac{1}{\sqrt R} \frac{\sigma_\epsilon}{\sigma_\epsilon + \gamma \|\Delta\|} .
\end{align}
Specifically, in the absence of noise ($\sigma_\epsilon = 0$), the true regression vectors are perfectly recovered.
\end{theorem}
\Cref{thm:K2_population} considers only the first phase of \MIRLS, as it is sufficient to recover the regression vectors in the described setting. Indeed, empirically, the second phase is often unnecessary.
The choice of an oversampling ratio $\rho = \infty$ is suited to population setting where $n=\infty$; see the discussion following Eq.~\eqref{eq:phase1_beta}.
The theorem proof appears in \cref{sec:proof}.

\begin{remark}[Required imbalance]\label{remark:theory_imbalance}
Due to \eqref{eq:bounded_balance_thm}, \cref{thm:K2_population} holds only for a sufficiently imbalanced mixture.
In the absence of noise ($\sigma_\epsilon=0$), \eqref{eq:bounded_balance_thm} imposes the constraint ${\gamma < 1/2}$. The numerical factor $5/4$ in the definition of $\gamma$ is not strict, and $\gamma$ can actually get as close as desired to $p_2/p_1$, at the expense of increasing $R$ by a constant factor. Hence, the constraint essentially reads $p_2/p_1 < 1/2$. It other words, the most balanced mixture allowed by our guarantee is
$$
p = (2/3, 1/3).
$$
More generally, there is a trade-off between balance and noise: the required imbalance ratio $p_1/p_2$ increases with the noise level $\sigma_\epsilon$.

We emphasize that empirically, \MIRLS works well also on balanced mixtures $1/2 \leq p_2/p_1 \leq 1$; see \cref{sec:additional_simulations}. Hence it is an open problem to provide theoretical guarantees in this regime as well.
\end{remark}

\begin{remark}[Allowed range for $w_\text{th}$]\label{remark:theory_allowed_range_wth}
\cref{thm:K2_population} limits the range of $w_\text{th}$, see \eqref{eq:wth_assumption}.
This range depends on $q$, which in turn depends on the noise level and the mixture imbalance. For example, at a noise level $\sigma_\epsilon = 10^{-2}$, proportions $p = (4/5, 1/5)$, separation of $\|\Delta\| = 1$ and parameter choice of $\eta = 1$, the allowed range is $0.69 \leq w_\text{th} \leq 0.9$.
\end{remark}

\begin{remark}[Overparameterization / unknown $K$]\label{remark:theory_overparameterization}
In practical scenarios, the number of components $K$ is often unknown. Remarkably, \cref{thm:K2_population} can be extended to an overparameterized setting, where \MIRLS is given an overestimate for the number of components $K > 2$, together with a corresponding (arbitrary) initialization $(\beta_1, \dotsc, \beta_K)$.
This is explicitly discussed in \cref{sec:theory_unknownK} (\cref{proposition:unknownK}), and also demonstrated empirically in \cref{fig:robustness}(right).
\end{remark}

\Cref{thm:K2_population} and \cref{remark:theory_overparameterization} theoretically support several empirical findings from previous sections: unlike other methods, \MIRLS performs better on imbalanced mixtures than on balanced ones; it copes well with an overparameterized $K$; and it works well starting from a random initialization.
Our analysis (\cref{sec:proof}) sheds light on the inner mechanism of \MIRLS that enables these features.

\paragraph*{Comparison to prior work.}
Several works derived recovery guarantees for \AltMin \cite{yi2014alternating,yi2016solving,ghosh2020alternating} and \GD \cite{zhong2016mixed,li2018learning} in a noiseless setting.
More related to our \cref{thm:K2_population} are works that studied the population \EM algorithm in the presence of noise \cite{balakrishnan2017statistical,daskalakis2017ten, klusowski2019estimating,kwon2019global,kwon2020converges}.
These latter works assumed a perfectly balanced mixture, $p_1 = p_2 = 1/2$. An exception is \cite{kwon2020converges}, who allowed for $K>2$ and an imbalanced mixture. However, their allowed imbalance is limited. In addition, they required a sufficiently accurate initialization.
A key novelty in our result is not only that we allow for highly imbalanced mixtures, but that large imbalance actually makes recovery \textit{easier} for \MIRLS: since $q$ is monotonically decreasing with the mixture imbalance, the allowed range \eqref{eq:wth_assumption} of the parameter $w_\text{th}$ increases with the imbalance. Furthermore, our result holds for an arbitrary initialization.
The downside is that \cref{thm:K2_population} requires sufficient imbalance (see \cref{remark:theory_imbalance}), and does not provide a recovery guarantee for our method on a balanced mixture, even though empirically, our method works well also on balanced mixtures.

Our result is novel in another aspect as well. In contrast to most existing guarantees, \cref{thm:K2_population} holds also for an arbitrary input number of components $K \geq 2$; see \cref{remark:theory_overparameterization}.

\MIRLS recovers the first component by treating the samples from the second component as outliers. Our guarantee allows the second component to consist up to $1/3$ of the data; see \cref{remark:theory_imbalance}.
For comparison, in the context of robust regression, \cite{mukhoty2019globally} recently analyzed an IRLS method, and allowed less than $1/5$ corrupted samples. Our higher tolerance is possible thanks to the strong linearity assumption of MLR \eqref{eq:MLR}.

\section{Summary and Discussion}
In this work, we presented a novel method to solve MLR, \MIRLS, that handles both imbalanced and balanced mixtures.
\MIRLS is also robust to outliers and to an overestimated number of components.
In particular, under suitable conditions, it can be run with an overestimated $K > K^*$ and will automatically find the true number of components $K^*$.

The basic idea of {\MIRLS} - sequential recovery using tools from robust regression - was also proposed by Banks et al.~\cite{banks2009cherry}. Several important differences between their method and \MIRLS were listed in the introduction; first and foremost is the scalability of the algorithm.
It is interesting to note that \cite{aitkin1980mixture} also made a connection between MLR and robust regression, but the other way around: as a simplified theoretical model, they assumed the outliers follow a linear model, and applied the \EM algorithm to the obtained MLR problem to detect them.

Although stronger than several existing results in certain aspects, our current theoretical analysis suffers from two limiting assumptions: only $K=2$ components, and a population setting with infinitely many samples. While the first assumption is common in the literature (e.g., \cite{balakrishnan2017statistical,kwon2019global}, and many others), population analysis is usually accompanied by a finite-sample one.
We plan to extend our analysis to finite-sample setting in the future.
Another interesting future research direction is to formally prove \MIRLS's robustness to outliers, as was empirically demonstrated in \cref{sec:simulations}.

Additional related work, as well as future methodological research problems such as dealing with non-Gaussian noise, model-based clustering (where $c_i^*$ is a function of $x_i$) and non-linear regression, are discussed in \cref{sec:additional_references}.

\section*{Acknowledgements}
The research of P.Z. was partially supported by a fellowship for data science from the Israeli Council for Higher Education (CHE). B.N.~is the incumbent of the William Petschek Professorial Chair of Mathematics. We thank Yuval Kluger and Ofir Lindenbaum for interesting discussions. We thank the authors of \cite{zhong2016mixed} for sharing their code with us.

\bibliographystyle{alpha}
\bibliography{MLR}

\newpage
\appendix

\section{Additional Related Work and Future Research Directions}\label{sec:additional_references}
The most popular approach to solve MLR is, arguably, expectation-maximization and its variants.
In recent years, this approach was extensively studied both theoretically and empirically.
Most of the works on \EM \cite{faria2010fitting,huang2012mixture,balakrishnan2017statistical,daskalakis2017ten,
klusowski2019estimating,kwon2019global,kwon2020converges,zhang2020estimation,kwon2021minimax} made two simplifying assumptions:
(i) Gaussian noise, $\epsilon \sim \mathcal N(0, \sigma^2 I)$; and (ii) model-free clustering. The second assumption means that the cluster assignment $c_i^*$ is random and independent of the sample position in space $x_i$.
In this work, we made similar assumptions in our simulations (\cref{sec:simulations}) and in our theoretical analysis (\cref{sec:theory}).
Several other works on expectation-maximization extended this setting in both directions: (i) non-Gaussian noise
\cite{khalili2007variable,hunter2012semiparametric,song2014robust,yao2014robust,
hu2017robust,barazandeh2021efficient},
and (ii) model-based clustering, where $c_i^*$ potentially depends on $x_i$
\cite{young2010mixtures,hunter2012semiparametric,ingrassia2014model}.
While nothing in our approach explicitly assumes Gaussian noise or model-free clustering, we did not explicitly address the non-Gaussian and model-based clustering settings in our work, and we leave it for a future research.

Besides expectation-maximization, other approaches proposed in the literature are alternating minimization
\cite{yi2014alternating,yi2016solving,ghosh2020alternating,pal2022learning},
convex relaxation
\cite{chen2014convex,hand2018convex,jiang2021nonparametric},
and gradient descent applied to a suitable objective \cite{zhong2016mixed,li2018learning,diamandis2021wasserstein}.
These methods, as well as expectation-maximization, recover the linear models simultaneously, and the corresponding works did not pay specific attention to the imbalanced MLR setting.
In particular, all the available theoretical guarantees in the literature either assume a perfectly balanced mixture \cite{yi2014alternating,chen2014convex,balakrishnan2017statistical,daskalakis2017ten,
klusowski2019estimating,kwon2019global,ghosh2020alternating,
diamandis2021wasserstein,kwon2021minimax}, or at least a sufficiently balanced one \cite{yi2016solving,zhong2016mixed,li2018learning,chen2020learning,kwon2020converges}.
In contrast, our guarantee holds for a sufficiently \textit{imbalanced} mixture.

Under certain assumptions (e.g.~model-free clustering), the clustering task in MLR can be viewed as a special case of subspace clustering; see \cite{zhong2016mixed,pimentel2017mixture}.
More generally, MLR is a special case of finite mixture models; see \cite{mclachlan2019finite} for a comprehensive review on this broader field. The framework presented in this paper can, in principle, be applied to non-linear mixture models: given a robust non-linear regressor, we can use it to separate the components of the mixture. This is another appealing direction for future research.

\section{The Second Phase of \MIRLS}\label{sec:method_phaseII}
The first phase of \MIRLS calculates estimates $\beta_1^\text{phase-I}, \ldots, \beta_K^\text{phase-I}$ for the regression vectors.
In the second phase, we initialize $\beta = \beta^\text{phase-I}$, and then run the following scheme for $T_2$ iterations. A pseudocode appears in \cref{alg:MIRLS_phaseII}. First, we calculate the following residuals and modified weights,
\begin{subequations}\label{eq:phase2_IRLS}\begin{align}
r_{i,k} &= \left|x_i^\top \beta_k - y_i\right|, \quad \forall i\in [n],\, \forall k\in [K], \label{eq:phase2_residuals} \\
\tilde w_{i,k} &= \frac{1/(r_{i,k}^2 + \epsilon_\text{mp})}{\sum_{k'=1}^K 1/(r_{i,k'}^2 + \epsilon_\text{mp})}, \quad \forall i\in [n],\, \forall k\in [K], \label{eq:phase2_weights}
\end{align}\end{subequations}
where $\epsilon_\text{mp}$ is the machine precision.
Next, we binarize some of the weights in a two-step scheme.
Let
\begin{align}\label{eq:H}
H &= \left\{i\in [n]: \exists k\in [K] \,\text{ s.t. } \tilde w_{i,k} \geq \frac{2}{3}\right\}
\end{align}
be the subset of samples with a single dominant weight.
The numerical constant $2/3$ is arbitrary, and the performance of \MIRLS is insensitive to its exact value.
(i) For each sample in $H$, we set its highest weight to $1$ and zero out the others; (ii) for the samples outside $H$, we zero out the weights smaller than $1/K$, and renormalize $\tilde w_{i,k} = \tilde w_{i,k}/\sum_{k'=1}^K \tilde w_{i,k'}$. Finally, we calculate a weighted least squares,
\begin{align}\label{eq:phase2_beta}
\beta_k &= (X^\top \tilde W_k X)^{-1} X^\top \tilde W_k y, \quad \forall k\in [K],
\end{align}
where $\tilde W_k = \text{diag}(\tilde w_{1,k}, \tilde w_{2,k}, \dotsc)$.
We iterate \cref{eq:phase2_IRLS,eq:H,eq:phase2_beta} $T_2$ times.
This concludes the second phase of \MIRLS.
Note that for $K=2$, the second phase coincides with the alternating minimization algorithm \cite{yi2014alternating}. 
The final output of \MIRLS is $\beta = (\beta_1, \dotsc, \beta_K)$.

\begin{algorithm}[tb]
\caption{\MIRLS: refinement phase (often unnecessary)} \label{alg:MIRLS_phaseII}
\SetKwInOut{Input}{input}
\SetKwInOut{Output}{output}
\Input{samples $\{(x_i, y_i)\}_{i=1}^n$, number of components $K$, number of iterations $T_2$, phase I estimates $\{\beta_k^\text{phase-I}\}_{k=1}^K$}
\Output{estimates $\beta_1, \ldots, \beta_K$ such that $y_i \approx x_i^\top \beta_{k(i)}$ for some function $k: [n]\to [K]$}
initialize $\beta_k = \beta^\text{phase-I}_k, \quad \forall k\in [K]$ \\
\For{$t=1$ {\bfseries to} $T_2$}
{
	compute $r_{i,k} = |x_i^\top \beta_k - y_i|, \quad \forall i\in [n],\, \forall k\in [K]$ \\
	compute $\tilde w_{i,k} = (r_{i,k}^2 + \epsilon_\text{mp})^{-1} / \sum_{k'=1}^K (r_{i,k'}^2 + \epsilon_\text{mp})^{-1}, \quad \forall i\in [n],\, \forall k\in [K]$ \\
	set $H = \{i\in [n]: \exists k\in [K] \,\text{ s.t. } \tilde w_{i,k} \geq 2/3\}$ \\
	set $\tilde w_{i,k} = 1$ if $\tilde w_{i,k} = \max_{k'} \tilde w_{i,k'}$ and $0$ otherwise, $\quad \forall i\in H, \, k\in [K]$ \\
	set $\tilde w_{i,k} = 0$, $\quad \forall i,k \text{ s.t. } \tilde w_{i,k} < 1/K$ \\
	compute $\tilde w_{i,k} = \tilde w_{i,k}/\sum_{k'=1}^K \tilde w_{i,k'}, \quad \forall i\in [n],\, \forall k\in [K]$ \\
	compute $\beta_k = (X^\top \tilde W_k X)^{-1} X^\top \tilde W_k y, \quad \forall k\in [K]$
}
\end{algorithm}

Both phases of \MIRLS employ an IRLS approach. However, as discussed earlier, they are fundamentally different: the first phase estimates the components sequentially, while the second one does it simultaneously.

\begin{remark}[Computational complexity]\label{remark:complexity}
Each round of the first phase of \MIRLS is dominated by the weighted least squares problem \eqref{eq:IRLS_beta}, whose complexity is $\mathcal O(n^2 d)$. The complexity of the first phase is thus $\mathcal O(n^2 d K T_1)$. Similarly, as each round in the second phase is dominated by the weighted least squares computation \eqref{eq:phase2_beta}, its complexity is $\mathcal O(n^2 d T_2)$. The overall number of operations in \MIRLS is thus $\mathcal O\left( n^2 d (K T_1 + T_2) \right)$.
\end{remark}

\section{Theoretical Guarantee with an Unknown $K$}\label{sec:theory_unknownK}
In \cref{remark:unknown_K} of the main text, we claimed that \cref{thm:K2_population} can be extended to the unknown $K$ setting. \Cref{proposition:unknownK} formulates this claim.

\begin{proposition}\label{proposition:unknownK}
Assume the conditions of \cref{thm:K2_population}, but with $K=2$ not given as input to \MIRLS. Then \MIRLS would correctly stop the IRLS scheme \eqref{eq:IRLS} after two rounds according to the stopping criterion described in \cref{remark:unknown_K}.
\end{proposition}
Intuitively, this happens as the second round approximately recovers the second regression vector, so that removing the samples with good and moderate fit actually removes all the samples and leaves no active samples for a third round.
The formal proof appears in \cref{sec:proof_lemmaK2pop_propositionUnknownK}.

\section{Proof of \cref{thm:K2_population}}\label{sec:proof}
Let us first describe the modified algorithm, for which our theoretical analysis holds.
For simplicity, we suit it to the assumptions of \cref{thm:K2_population}, namely $K=2$ components and $T_1 = 1$ iterations. For clarity, a pseudocode is presented in \cref{alg:MIRLS_modified}.

\begin{algorithm}[tb]
\caption{\MIRLS: modified main phase for analysis purposes} \label{alg:MIRLS_modified}
\SetKwInOut{Input}{input}
\SetKwInOut{Output}{output}
\Input{samples $\{(x_i, y_i)\}_{i=1}^n$, parameters $\eta, w_\text{th}, R$}
\Output{estimates $\beta_1^\text{(phase-I)}, \beta_2^\text{(phase-I)}$}
set $S_1 = [n]$ \\
initialize $\beta_1$ randomly \\
compute $r_{i,1} = |x_i^\top \beta_1 - y_i|, \quad \forall i\in S_1$ \label{alg:MIRLS_modified_w1IRLS} \\
compute $w_{i,1} = (1 + \eta r_{i,1}^2 / R)^{-1}, \quad \forall i\in S_1$ \label{alg:MIRLS_modified_r1IRLS} \\
compute $\beta_1 = (X_{S_1}^\top W_1 X_{S_1})^{-1} X_{S_1}^\top W_1\, y_{S_1}$ \quad // $W_1 = \text{diag}(w_{1,1}, w_{2,1}, \ldots)$ \label{alg:MIRLS_modified_beta1} \\
compute $r_{i,1} = |x_i^\top \beta_1 - y_i|, \quad \forall i\in S_1$ \\
compute $w_{i,1} = (1 + \eta r_{i,1}^2 / R)^{-1}, \quad \forall i\in S_1$ \label{alg:MIRLS_modified_weight1} \\
set $S_2 = S_2' = \{i\in S_1 : \|x_i\|^2 \leq R \text{ and } w_{i,1} \leq w_\text{th} \}$ \label{alg:MIRLS_modified_S2} \\
\If{$|S_{2}| < \infty$ \label{alg:MIRLS_modified_resettingCriterion}} 
{
	start \MIRLS over with $w_\text{th} \leftarrow w_\text{th} + 0.1$
}
compute $\beta_2^\text{(phase-I)} = (X_{S_2'}^\top X_{S_2'})^{-1} X_{S_2'}^\top y_{S_2'}$ \label{alg:MIRLS_modified_beta2phaseI} \\
compute $r_{i,2} = |x_i^\top \beta_2^\text{(phase-I)} - y_i|, \quad \forall i\in S_1$ \\
compute $w_{i,2} = (1 + \eta r_{i,2}^2 / R)^{-1}, \quad \forall i\in S_1$ \\
set $S'_1 = \{i\in S_1 : \|x_i\|^2 \leq R \text{ and } w_{i,2} \leq w_\text{th} \}$ \label{alg:MIRLS_modified_S1'} \\
compute $\beta_1^\text{(phase-I)} = (X_{S_1'}^\top X_{S_1'})^{-1} X_{S_1'}^\top y_{S_1'}$ \label{alg:MIRLS_modified_beta1phaseI}
\end{algorithm}

First, we replace the original definition of the weights in \eqref{eq:IRLS_W}. Instead of scaling the residuals by the square median residual $\bar r_k^2$, we assume the following formula:
\begin{align}\label{eq:modified_w}
w_{i,k} &= \frac{1}{1 + \eta r_{i,k}^2 / R},
\end{align}
where $R\geq 1$ is a constant.
In addition, we change the definition of the subsets in \eqref{eq:index_subsets} as follows:
\begin{subequations}\label{eq:modified_index_subsets}
\begin{align}
S_2 = S_2' &= \{i\in S_1 : \|x_i\|^2 \leq R \text{ and } w_{i,1} \leq w_\text{th}\}, \label{eq:modified_index_subsets_S2} \\
S_1' &=  \{i\in S_1: \|x_i\|^2 \leq R \text{ and } w_{i,2} \leq w_\text{th}\}. \label{eq:modified_index_subsets_S1'}
\end{align}
\end{subequations}
The equality $S_2 = S_2'$ corresponds to taking $\rho = \infty$ in \eqref{eq:index_subsets_Sk'}. As discussed after \eqref{eq:index_subsets_Sk'}, the oversampling ratio $\rho$ is related to the sample size $n$; in our population setting with $n=\infty$, we thus take $\rho=\infty$.
For the same reason, the stopping criterion (\cref{alg:MIRLS_modified_resettingCriterion} in \cref{alg:MIRLS_modified}) reads $|S_2| < \infty$.

Definition \eqref{eq:modified_index_subsets} contains two modifications with respect to the original \eqref{eq:index_subsets}.
First, we consider only samples with bounded norm $\|x_i\|^2 \leq R$. Otherwise, with small probability, a sample $x_i$ may have large magnitude $\|x_i\|$ and consequently have large residual $r_{i,1}$, even if the estimate $\beta_1$ is close to the true $\beta_1^*$.
Second, to uniformize the definitions of $S_1'$ and $S_2'$, we added the condition $w_{i,2} \leq w_\text{th}$ to the definition of $S_1'$ \eqref{eq:modified_index_subsets_S1'}, where $w_{i,2} = 1/(1 + \eta (x_i^\top \beta_2^\text{phase-I} - y_i)^2/R)$.
Since $w_{i,2}$ is calculated based on $\beta_2^\text{phase-I}$, it needs to be calculated before $\beta_1^\text{phase-I}$. 

\cref{thm:K2_population} is formulated in the large-$R$ regime.
The following lemma is similar to \cref{thm:K2_population}, but with the exact dependence on $R$. With this lemma in hand, \cref{thm:K2_population} immediately follows.
In this section, we use the following notation for convenience:
\begin{align}\label{eq:xi}
\xi = \frac{\sigma_\epsilon}{\|\Delta\|}
\quad \text{and} \quad
\tilde\xi = \frac{\xi}{\sqrt R} .
\end{align}

\begin{lemma}\label{lem:K2_population}
Let $\{(x_i,y_i)\}_{i=1}^\infty$, $K$, $p_1, p_2$, $\beta_1^*, \beta_2^*, \gamma$ and $q$ be defined as in \cref{thm:K2_population}.
Let $(\beta_1, \beta_2)$ be an arbitrary initialization to \MIRLS, and denote $D = \|\beta_1 - \beta_1^*\|/\|\Delta\|$.
Assume the parameters of \MIRLS satisfy \eqref{eq:wth_assumption}, $\rho = \infty$, and
\begin{align}\label{eq:R_bound}
R > \max\left\{\frac{1}{(q-\tilde\xi)^2 \|\Delta\|^2}, 5(3\max\{D, 1/2\} + \xi)^2 \|\Delta\|^2  \eta \right\} .
\end{align}
Then the first phase of \MIRLS with at least one iteration ($T_1\geq 1$) approximately recovers $\beta^*$,
\begin{align}\label{eq:approximate_recovery}
\max_{k=1,2}\|\beta_k - \beta_k^*\| < \frac{1}{\sqrt R} \frac{\xi}{\xi + \gamma} .
\end{align}
Specifically, in the absence of noise ($\xi = 0$), \MIRLS perfectly recovers the two components, $\beta_k = \beta_k^*$ for $k=1,2$.
\end{lemma}

\begin{proof}[Proof of \cref{thm:K2_population}]
The theorem follows by taking large enough $R$ in \cref{lem:K2_population}.
\end{proof}

\subsection{Proof of \cref{lem:K2_population} and \cref{proposition:unknownK}}\label{sec:proof_lemmaK2pop_propositionUnknownK}
To prove \cref{lem:K2_population}, we will use the following three auxiliary lemmas. Their proof appears in the next subsections.
In the following, unless otherwise stated, expectations are taken over all the random variables (typically $x_i$, $\epsilon_i$ and $c_i^*$).

\begin{lemma}\label{lem:xxt}
Let $x\sim \mathcal N(0, I_d)$ and $\epsilon$ be independent random variables. Suppose $\epsilon$ has a symmetric distribution, $\PP[\epsilon] = \PP[-\epsilon]$. Let $u \in \mathbb R^d$ be a fixed vector, and denote $x_u = \tilde u^\top x$ where $\tilde u = u/\|u\|$. Denote the events
\begin{subequations}
\begin{align}
\mathcal E &= \{|x_u| \leq s_1\}, \\
S &= \{\|x\|^2 \leq R \text{ and } (x_u - \epsilon)^2 \geq s_2\},
\end{align}
\end{subequations}
for some fixed positive scalars $s_1,s_2,R$. Let $U = \tilde u \tilde u^\top$ and $x_\perp = x - x_u \tilde u$. Then
\begin{align}
\E\left[xx^\top \mid \mathcal E\right] &= I_d - \left(1 - \E\left[x_u^2 \mid \mathcal E\right]\right) U, \label{eq:xxt_E} \\
\E\left[xx^\top \mid S\right]^{-1} &= \frac{1}{\E[\|x_\perp\|^2 \mid S]} \left(I_d  - U\right) + \frac{1}{\E[x_u^2 \mid S]} U . \label{eq:xxt_S}
\end{align}
\end{lemma}

\begin{lemma}\label{lem:Dt}
Assume the conditions of \cref{lem:K2_population}. Denote $\tilde\eta = \eta \|\Delta\|^2/R$.
Let $\beta_1^{(t)}$ be the $t$-th iterate of \eqref{eq:IRLS_beta} for the first component ($k=1$), and denote $D_t = \|\beta_1^{(t)} - \beta_1^*\|/\|\Delta\|$.
Then
\begin{align}
D_{t+1} &\leq \frac{5(q-\tilde\xi)}{6} \left(1 + \tilde\eta (3D_t + \xi)^2\right). \label{eq:Dt+1}
\end{align}
\end{lemma}

\begin{lemma}\label{lem:P}
Let $u, \Delta\in \mathbb R^d$, $\eta, R > 0$ and $q \in (0, 1/2)$ be fixed.
Let $x\sim \mathcal N(0, I_d)$ and $\epsilon$ be independent random variables. Suppose $\epsilon$ is bounded, $|\epsilon| < \xi\|\Delta\|$ where $\xi < q \sqrt R$. Denote $w(x,\epsilon) = \left(1 + \eta \left(x^\top u + \epsilon\right)^2/R\right)^{-1}$.
Further denote
\begin{align}
P = \PP\left[w(x,\epsilon) < w_\text{th} \mid \|x\|^2 \leq R\right]
\end{align}
where $w_\text{th}$ satisfies \eqref{eq:wth_assumption}.
Then $P = 0$ if $\|u\|/\|\Delta\| \leq q - \tilde\xi$ and $P>0$ if $\|u\|/\|\Delta\| \geq 1-q + \tilde\xi$.
\end{lemma}

Let us briefly sketch the proof idea before we present it formally.
\cref{lem:xxt} is a technical result, used occasionally throughout the proof.
Using \cref{lem:Dt}, we show that $\beta_1$ of \cref{alg:MIRLS_modified_beta1} in \cref{alg:MIRLS_modified} is a good approximation for the regression vector $\beta_1^*$. As a result, any sample with bounded norm that was generated from the first component has a small residual, and thus a large weight (\cref{alg:MIRLS_modified_weight1} in \cref{alg:MIRLS_modified}). By removing all the samples with large and moderate weights (i.e., constructing the set $S_2$, \cref{alg:MIRLS_modified_S2} in \cref{alg:MIRLS_modified}), we are left with active samples from the second component only, as follows by \cref{lem:P}. Thus, $\beta_2^\text{phase-I}$ of \cref{alg:MIRLS_modified_beta2phaseI} in \cref{alg:MIRLS_modified} accurately estimates the regression vector $\beta_2^*$. Then, we similarly show that $\beta_1^\text{phase-I}$ of \cref{alg:MIRLS_modified_beta1phaseI} accurately estimates the first vector $\beta_1^*$ as well. 
It is worth mentioning that due to the assumed imbalance, our method will indeed find $\beta_1^*$ as its first component and $\beta_2^*$ as its second.

\begin{proof}[Proof of \cref{lem:K2_population}]
As in \cref{lem:Dt}, let $\beta_1^{(t)}$ be the $t$-th iterate of \eqref{eq:IRLS_beta} for the first component ($k=1$), and denote $\tilde\eta = \eta \|\Delta\|^2/R$ and $D_t = \|\beta_1^{(t)} - \beta_1^*\|/\|\Delta\|$. In particular, $D_0 = D$.
We shall prove by induction that for any $t\geq 1$,
\begin{align}\label{eq:Dt_result}
D_t < q - \tilde\xi \leq \gamma + \frac{\xi}{p_1}.
\end{align}
Combined with the assumption $q < 1/2$, the first inequality in \eqref{eq:Dt_result} implies that $\beta_1^{(t)}$ is closer to $\beta_1^*$ than to $\beta_2^*$.
As a consequence, after removal of samples with good to moderate fit, the remaining (poor fit) samples are all belong to the second component, as we prove below.

Let $t=1$.
By \cref{lem:Dt}, after one iteration of the IRLS scheme (Eq.~\eqref{eq:IRLS}, or \cref{alg:MIRLS_modified_r1IRLS,alg:MIRLS_modified_w1IRLS,alg:MIRLS_modified_beta1} in \cref{alg:MIRLS_modified}), we have
\begin{align*}
D_{1} \leq \frac{5(q-\tilde\xi)}{6} \left(1 + \tilde\eta (3D_0 + \xi)^2\right) < \frac{5(q-\tilde\xi)}{6}\left(1 + \frac 15\right) = q - \tilde\xi,
\end{align*}
where in the second inequality we used $\tilde\eta = \eta \|\Delta\|^2/R \leq (1/5)/(3D_0 + \xi)^2$, see \eqref{eq:R_bound}.
This proves \eqref{eq:Dt_result} at $t=1$.
For the induction step, suppose \eqref{eq:Dt_result} holds for some $t\geq 1$, namely $D_t < q-\tilde\xi$. Recall that $q > \tilde\xi$.
Invoking \cref{lem:Dt} again yields
\begin{align*}
D_{t+1} \leq \frac{5(q-\tilde\xi)}{6} \left(1 + \tilde\eta (3D_t + \xi)^2\right) < \frac{5(q-\tilde\xi)}{6}\left(1 + \tilde\eta(3/2 + \xi)^2 \right) \leq \frac{5(q-\tilde\xi)}{6}\left(1 + \frac 15 \right) = q - \tilde\xi,
\end{align*}
where in the last inequality we used $\tilde\eta \leq (1/5)/(3/2+\xi)^2$, see \eqref{eq:R_bound}.
This proves \eqref{eq:Dt_result}.

Applying \eqref{eq:Dt_result} for $t = T_1 \geq 1$ yields
\begin{subequations}\begin{align}
\frac{\|\beta_1^{(T_1)} - \beta_1^*\|}{\|\Delta\|}
&= D_{T_1} < q - \tilde\xi, \label{eq:DT1_result}\\
\frac{\|\beta_1^{(T_1)} - \beta_2^*\|}{\|\Delta\|}
&\geq \frac{\|\beta_1^* - \beta_2^*\|}{\|\Delta\|} - \frac{\|\beta_1^{(T_1)} - \beta_1^*\|}{\|\Delta\|} = 1-D_{T_1} > 1-q + \tilde\xi. \label{eq:1-DT1_result}
\end{align}\end{subequations}
Hence, for any sample $x$ with $\|x\|^2 \leq R$ and whose response belongs to the second component ($c^*=2$), the corresponding weight $w_1$ (Eq.~\eqref{eq:modified_w}, or \cref{alg:MIRLS_modified_weight1} in \cref{alg:MIRLS_modified}) satisfies
\begin{align*}
\PP\left[w_1 < w_\text{th} \mid c^*=2 \text{ and } \|x\|^2\leq R\right]
&= \PP\left[\frac{1}{1 + \eta\left(x^\top \beta_1^{(T_1)} - x^\top \beta_2^* - \epsilon\right)^2/R} < w_\text{th} \mid \|x\|^2\leq R\right]
> 0,
\end{align*}
as follows by combining \eqref{eq:1-DT1_result} and \cref{lem:P} with $u = \beta_1^{(T_1)} - \beta_2^*$.
In contrast, combining \eqref{eq:DT1_result} with the same lemma for $u = \beta_1^{(T_1)} - \beta_1^*$, gives that for any sample $x$ whose response belongs to the first component ($c^*=1$),
\begin{align*}
\PP\left[w_1 < w_\text{th} \mid c^*=1\text{ and } \|x\|^2\leq R\right]
&= \PP\left[\frac{1}{1 + \eta\left(x^\top \beta_1^{(T_1)} - x^\top \beta_1^* - \epsilon\right)^2/R} < w_\text{th} \mid \|x\|^2\leq R\right]
= 0 .
\end{align*}
Hence, by \eqref{eq:modified_index_subsets_S2}, all the (infinite number of) samples in $S_2$ belong to the second component $c^* = 2$. In other words, the choice of the threshold $w_\text{th}$ allows to detect a subset of samples $(x,y)$ that all belong to the second component.
Note that the resetting criterion $|S_2| < \infty$ (\cref{alg:MIRLS_modified_resettingCriterion} in \cref{alg:MIRLS_modified}) does not hold, as $S_2$ is an infinite set.

Since all the samples in $S_2 = S_2'$ belong to the second component, their responses follow the relation $y = x^\top \beta_2^* + \epsilon$. The final estimate of the first phase for the second component (Eq.~\eqref{eq:phase1_beta}, or \cref{alg:MIRLS_modified_beta2phaseI} in \cref{alg:MIRLS_modified}) is thus
\begin{align*}
\beta_2^\text{phase-I} &=
\E\left[xx^\top \,\mid\, S_2 \right]^{-1} \E\left[xy \,\mid\, S_2 \right]
= \E\left[xx^\top \,\mid\, S_2 \right]^{-1} \E\left[x (x^\top \beta_2^* + \epsilon) \,\mid\, S_2 \right] \\
&= \beta_2^* + \E\left[xx^\top \,\mid\, S_2 \right]^{-1} \E\left[\epsilon\cdot x \,\mid\, S_2 \right],
\end{align*}
as follows by the weak law of large numbers.
Rearranging and taking the norm of both sides gives
\begin{align}
\left\|\beta_2^\text{phase-I} - \beta_2^*\right\| = \left\|\E\left[xx^\top \,\mid\, S_2 \right]^{-1} \E\left[\epsilon\cdot x \,\mid\, S_2 \right]\right\|.
\label{eq:beta2_phaseI_temp}
\end{align}
To upper bound the RHS of \eqref{eq:beta2_phaseI_temp}, we shall analyze each of the two terms $\E\left[xx^\top \,\mid\, S_2 \right]^{-1}$ and $\E\left[\epsilon\cdot x \,\mid\, S_2 \right]$.
Let $u = \beta_1^{(T_1)} - \beta_1^*$, and decompose $x = x_u \tilde u + x_\perp$ where $u \perp x_\perp$ and $\tilde u = u/\|u\|$.
Invoking \cref{lem:xxt} implies that the first term satisfies
\begin{align*}
\E\left[xx^\top \,\mid\, S_2 \right]^{-1} &= \frac{1}{\E\left[\|x_\perp\|^2 \mid S_2\right]}\left(I_d - \tilde u \tilde u^\top \right) + \frac{1}{\E\left[x_u^2 \,\mid\, S_2 \right]} \tilde u \tilde u^\top .
\end{align*}
To analyze the second term on the RHS of \eqref{eq:beta2_phaseI_temp}, recall the definition of $S_2$ in \eqref{eq:modified_index_subsets_S2}. The weight condition of $S_2$, $w_2 \leq w_\text{th}$, involves only the $x_u$ part of $x$.
Together with the isotropic distribution assumption on $x$ \eqref{eq:gaussian_assumption}, it follows that $x_\perp$ is isotropically distributed even when conditioned on $S_2$. Hence, the second term on the RHS of \eqref{eq:beta2_phaseI_temp} satisfies
\begin{align*}
\E\left[\epsilon\cdot x \,\mid\, S_2 \right] &= \E\left[\epsilon\cdot x_u \,\mid\, S_2 \right]\, \tilde u .
\end{align*}
Inserting these two equalities into \eqref{eq:beta2_phaseI_temp} yields
\begin{align}
\left\|\beta_2^\text{phase-I} - \beta_2^*\right\| &= \frac{\left|\E\left[\epsilon\cdot x_u \,\mid\, S_2 \right]\right|}{\E\left[x_u^2 \,\mid\, S_2 \right]} \,
\leq \xi \|\Delta\|\cdot \frac{\E\left[|x_u| \,\mid\, S_2 \right]}{\E\left[x_u^2 \,\mid\, S_2 \right]}
\leq \frac{\xi \|\Delta\|}{\sqrt{\E\left[x_u^2 \,\mid\, S_2 \right]}},
\label{eq:beta2_phaseII_temp}
\end{align}
where the first inequality follows by the bounded noise assumption \eqref{eq:noise_assumptions}, and the second by Jensen's inequality $\E[|x_u|]^2 \leq \E[x_u^2]$.
We shall now lower bound $\E\left[x_u^2 \,\mid\, S_2 \right]$.
For any pair $(x,y)\in S_2$, the weight satisfies
\begin{align*}
\frac{1}{1 + \eta \left(x_u - \epsilon\right)^2/R} = w_1 \leq w_\text{th} < \frac{1}{1 + \eta q^2 \|\Delta\|^2},
\end{align*}
where the second inequality follows by \eqref{eq:wth_assumption}.
Rearranging and taking the square root gives that
\begin{align*}
|x_u - \epsilon| > q \|\Delta\| \sqrt R .
\end{align*}
By the triangle inequality,
\begin{align*}
|x_u| &> q \|\Delta\| \sqrt{R} - |\epsilon|
\stackrel{(a)}{\geq} (q \sqrt R - \xi) \|\Delta\| 
\stackrel{(b)}{\geq} \left(\gamma \sqrt{R} + \frac{\xi \sqrt R}{p_1} \right) \|\Delta\|
\stackrel{(c)}{\geq} (\gamma + \xi) \|\Delta\| \sqrt R ,
\end{align*}
where (a) follows by combining \eqref{eq:noise_assumptions} and \eqref{eq:xi}, (b) by the definition of $q$ \eqref{eq:bounded_balance_thm}, and (c) by $p_1 \leq 1$.
In particular, $\E\left[x_u^2 \,\mid\, S_2 \right] > \left(\gamma + \xi \right)^2 \|\Delta\|^2 R$.
Plugging this into \eqref{eq:beta2_phaseII_temp} yields
\begin{align}\label{eq:beta2_phaseI_approximateRecovery}
\left\|\beta_2^\text{phase-I} - \beta_2^*\right\| < \frac{1}{\sqrt R} \frac{\xi}{\gamma + \xi} .
\end{align}
This completes the analysis of the second round of the IRLS scheme, and proves \eqref{eq:approximate_recovery} at $k=2$.

Finally, we need to prove \eqref{eq:approximate_recovery} at $k=1$, by deriving a similar bound for the estimate $\beta_1^\text{phase-I}$ of the first component.
Dividing \eqref{eq:beta2_phaseI_approximateRecovery} by $\|\Delta\|$ gives
\begin{align*}
\frac{\left\|\beta_2^\text{phase-I} - \beta_2^*\right\|}{\|\Delta\|} < \frac{1}{\|\Delta\| \sqrt R} \frac{\xi}{\gamma + \xi} \leq \frac{1}{\|\Delta\| \sqrt R} < q - \tilde\xi,
\end{align*}
where the last inequality follows by \eqref{eq:R_bound}.
This bound is identical to \eqref{eq:DT1_result}, but now for the accuracy of the second component rather than the first one.
Since the subset $S_1'$, which is calculated using $\beta_2^\text{phase-I}$, is defined similarly to $S_2'$ (see \eqref{eq:modified_index_subsets}, or \cref{alg:MIRLS_modified_S2,alg:MIRLS_modified_S1'} in \cref{alg:MIRLS_modified}), the rest of the argument follows the lines of the second component analysis described above, and we omit its details.
\end{proof}

\begin{proof}[Proof of \cref{proposition:unknownK}]
For any $i\in S_2$,
\begin{align*}
w_{i,2} &= \frac{1}{1 + \eta (x_i^\top (\beta_2^\text{phase-I} - \beta_2^*) + \epsilon_i)^2/R}
\geq \frac{1}{1 + \eta \left(\|x_i\|/\sqrt R + \xi \|\Delta\| \right)^2 /R} \\
&\geq \frac{1}{1 + \eta \left(1 + \xi \|\Delta\| \right)^2 /R},
\end{align*}
where the first inequality follows by \eqref{eq:beta2_phaseI_approximateRecovery} and \eqref{eq:noise_assumptions}, and the second by the condition $\|x_i\|^2 \leq R$ in the definition of $S_2$.
For large enough $R$, we get
\begin{align*}
w_{i,2} \geq \frac{1}{1 + \eta q^2 \|\Delta\|^2} > w_\text{th},
\end{align*}
where the second inequality follows by \eqref{eq:wth_assumption}.
As a result, $S_3 = \emptyset$ according to the definition of $S_3$ in \eqref{eq:index_subsets_Sk+1}.
\end{proof}

\subsection{Proof of \cref{lem:xxt}}
\begin{proof}[Proof of \cref{lem:xxt}]
Given that $x = x_u \tilde u + x_\perp$ and $U = \tilde u \tilde u^\top$,
\begin{align*}
xx^\top = x_u^2 U + x_u \tilde u x_\perp^\top + x_u x_\perp \tilde u^\top + x_\perp x_\perp^\top .
\end{align*}
First, let us show that the two middle terms, $x_u \tilde u x_\perp^\top$ and $x_u x_\perp \tilde u^\top$, vanish in expectation conditional on either $\mathcal E$ or $S$.
The case of $\mathcal E$ is simpler: as $x_\perp$ is independent of $x_u$ and thus also independent of $\mathcal E$,
\begin{align*}
\E\left[x_u x_\perp \tilde u^\top \mid \mathcal E\right] = \E[x_u \mid \mathcal E]\cdot \E[x_\perp] \tilde u^\top = 0 .
\end{align*}
Similarly, $\E\left[x_u \tilde u x_\perp^\top\right] = 0$ as well.

Next, let us analyze $\E\left[x_u x_\perp \tilde u^\top \mid S\right]$.
Unconditioned on $S$, $x_u$ is normally distributed around zero, and in particular symmetric.
In addition, unconditioned on $S$, $\epsilon$ is symmetric and independent of $x_u$. Hence, $\PP\left[(x_u, \epsilon)\right] = \PP[x_u] \cdot \PP[\epsilon] = \PP[-x_u] \cdot \PP[-\epsilon] = \PP[(-x_u, -\epsilon)]$.
Together with Bayes' theorem and the fact that $\PP\left[S \mid (x_u, \epsilon)\right] = \PP[S \mid (-x_u, -\epsilon)]$, we conclude
\begin{align*}
\PP[(x_u, \epsilon) \mid S] &= \frac{\PP[S \mid (x_u, \epsilon)] \cdot \PP[(x_u, \epsilon)]}{\PP[S]}
= \frac{\PP[S \mid (-x_u, -\epsilon)] \cdot \PP[(-x_u, -\epsilon)]}{\PP[S]} \nonumber \\
&= \PP[(-x_u, -\epsilon) \mid S] .
\end{align*}
As a result, the marginal distribution of $x_u$ conditional on $S$ is also symmetric,
\begin{align*}
\PP[x_u \mid S] &= \E_\epsilon[\PP[(x_u, \epsilon) \mid S]] = \E_\epsilon[\PP[(-x_u, -\epsilon) \mid S]] = \E_\epsilon[\PP[(-x_u, \epsilon) \mid S]] \\
&= \PP[-x_u \mid S].
\end{align*}
Further, $x_u$ and $x_\perp$ are independent when unconditioned on $S$, and the coupling between $x_u$ and $x_\perp$ under the event $S$ is only by the inequality $\|x_\perp\|^2 \leq R - x_u^2$, namely it depends only on the magnitudes $|x_u|$ and $\|x_\perp\|$. Hence,
\begin{align*}
\PP[(x_u, x_\perp) \mid S] &= \frac{\PP[S \mid (x_u, x_\perp)]\cdot \PP[(x_u, x_\perp)]}{\PP[S]}
= \frac{\PP[S \mid (x_u, x_\perp)]\cdot \PP[x_u] \cdot \PP[x_\perp]]}{\PP[S]} \\
&= \frac{\PP[S \mid (-x_u, x_\perp)]\cdot \PP[-x_u] \cdot \PP[x_\perp]}{\PP[S]}
= \frac{\PP[S \mid (-x_u, x_\perp)]\cdot \PP[(-x_u, x_\perp)]}{\PP[S]} \\
&= \PP[(-x_u, x_\perp) \mid S].
\end{align*}
This implies that $\E[x_u x_\perp \mid S] = 0$, so that $\E[x_u x_\perp \tilde u^\top \mid S] = 0$. Similarly, $\E[x_u \tilde u x_\perp^\top] = 0$.

Next, we analyze the last term $x_\perp x_\perp^\top$.
Again, the case of $\mathcal E$ is simple: since $x_\perp$ is independent of $\mathcal E$, $\E\left[x_\perp x_\perp^\top \mid \mathcal E\right] = \E\left[x_\perp x_\perp^\top \right] = I_d - U$. This proves \eqref{eq:xxt_E}.

Finally, to prove \eqref{eq:xxt_S}, we analyze the case of $S$.
Let $e_i$ be the $i$-th standard basis vector.
W.l.o.g., assume $\tilde u = e_d = \begin{pmatrix}0, & \ldots, & 0, & 1\end{pmatrix}^\top$.
Decompose $x_\perp = \sum_{i=1}^{d-1} a_i e_i$.
Recall that conditional on the event $S$, $x_\perp$ still has a spherically symmetric distribution. In particular, for any value of $\|x_\perp\| = t$, the vector $x_\perp$ is uniformly distributed on the sphere of radius $t$. Hence, for any $i\neq j$, $\E[a_i a_j \mid S] = \E_t\left[ \E[a_i a_j \mid S, \|x_\perp\| = t ] \right]$ vanishes, which implies
\begin{align*}
\E[x_\perp x_\perp^\top \mid S] &=
\E\left[\sum_{i,j=1}^{d-1} a_i a_j e_i e_j^\top \right]
= \sum_{i=1}^{d-1} \E[a_i^2] e_i e_i^\top \nonumber \\
&= \E[a_1^2 \mid S] \sum_{i=1}^{d-1} e_i e_i^\top 
= \E[a_1^2 \mid S] (I_d - U) .
\end{align*}
Let $\alpha_u = \E[x_u^2 \mid S]$ and $\alpha_\perp = \E[a_1^2 \mid S]$. Then
\begin{align*}
\E[xx^\top \mid S] &= \alpha_u U + \alpha_\perp (I_d - U)
= \alpha_\perp \left(I_d - \frac{\alpha_\perp - \alpha_u}{\alpha_\perp} U \right) .
\end{align*}
Its inverse is
\begin{align*}
\E[xx^\top \mid S]^{-1} &= \frac{1}{\alpha_\perp} \left(I_d + \frac{\alpha_\perp - \alpha_u}{\alpha_u} U\right) = \frac{1}{\alpha_\perp} (I_d - U) + \frac{1}{\alpha_u} U.
\end{align*}
\end{proof}

\subsection{Proof of \cref{lem:Dt}}
To prove \cref{lem:Dt}, we state and prove the following auxiliary result.
In this subsection, we use the notation $A \succeq B$ to indicate that a pair of matrices $A,B$ satisfies the semidefinite positive cone inequality, namely $A-B$ is positive semidefinite.

\begin{lemma}\label{lem:E_bounds}
Let $u\in \mathbb R^d$ and $x\sim \mathcal N(0, I_d)$. Let $z\in \mathbb R$ be a bounded random variable, $|z| \leq z_\text{max}$, independent of $x$. Denote $w(x,z) = \left(1 + \left(x^\top u + z\right)^2\right)^{-1}$.
Then
\begin{align}\label{eq:E_wx_bound}
\left\|\E\left[w(x,z)\cdot zx \right]\right\| \leq \sqrt\frac{2}{\pi} z_\text{max},
\end{align}
and
\begin{align}\label{eq:E_wxx_bound}
\frac{24/25}{1+(3\|u\| + z_\text{max})^2}
\leq \sigma_\text{min}\left(\E\left[w(x,z)\cdot xx^\top \right]\right)
\leq \left\|\E\left[w(x,z)\cdot xx^\top \right]\right\|
\leq 1 .
\end{align}
\end{lemma}

\begin{proof}
If $u=0$ then the lemma holds, as in this case $w = 1/(1+z^2) \in (0,1]$, it is independent of $x$, and $\E[x] = 0$.
Hence, we assume $u\neq 0$. Decompose $x = x_u \tilde u + x_\perp$ where $u \perp x_\perp$ and $\tilde u = u/\|u\|$.
Since $x \sim \mathcal N(0, I_d)$, we have $x_u \sim \mathcal N(0, 1)$ and $x_\perp \sim \mathcal N(0, I_d - \tilde u \tilde u^\top)$. Furthermore, $x_u$ and $x_\perp$ are independent, and $\E[x_\perp] = 0$.
Hence,
\begin{align*}
\left\|\E\left[w(x,z)\cdot zx \right]\right\| &= \|\E\left[w(x_u \tilde u,z)\cdot z(x_u \tilde u + x_\perp)\right] \| = |\E\left[w(x_u \tilde u,z)\cdot zx_u\right]| + \|\E\left[w(x_u \tilde u,z)\cdot z\right] \cdot \E[x_\perp]\| \\
&\leq z_\text{max}\cdot \E[|x_u|] = \sqrt\frac{2}{\pi} z_\text{max},
\end{align*}
where the inequality follows by $0 \leq w(x,z) \leq 1$ and $\E[x_\perp] = 0$.
This proves \eqref{eq:E_wx_bound}.

Next, we prove the lower bound on $\sigma_\text{min}\left(\E\left[w(x,z)\cdot xx^\top \right]\right)$.
Let $t>0$, and consider the event $\mathcal E_t = \{|x_u| \leq t\}$.
Conditional on this event, $w(x,z) \geq 1/(1+ (t\|u\|+z_\text{max})^2)$. Hence,
\begin{align}\label{eq:E_w_xx_lower_temp}
\E\left[w(x,z) \cdot xx^\top\right] &\succeq \PP[\mathcal E] \cdot \E\left[w(x,z)\cdot xx^\top \mid \mathcal E\right]
\succeq \frac{\PP[\mathcal E]}{1+ (t\|u\| + z_\text{max})^2} \E\left[xx^\top \mid \mathcal E\right] .
\end{align}
It is left to lower bound $\E\left[xx^\top \mid \mathcal E\right]$. Let $U = \tilde u \tilde u^\top$.
Since $x_\perp$ is independent of $x_u$, it is also independent of $\mathcal E$.
\cref{lem:xxt} thus implies
\begin{align*}
\E\left[xx^\top \mid \mathcal E\right] &=
I_d - \left(1 - \E\left[x_u^2 \mid \mathcal E\right]\right) U.
\end{align*}
Since $x_u \sim \mathcal N(0, 1)$, we have
\begin{align*}
\E\left[x_u^2 \mid \mathcal E\right] = \frac{\frac{1}{\sqrt{2\pi}} \int_{-t}^t x_u^2 e^{-x_u^2/2} dx_u}{\PP[\mathcal E]} = 1 - \sqrt{\frac{2}{\pi}} \frac{t\cdot e^{-\frac{t^2}{2}}}{\PP[\mathcal E]}.
\end{align*}
Hence,
\begin{align*}
\E\left[xx^\top \mid \mathcal E\right] &= I_d - \sqrt{\frac{2}{\pi}} \frac{t\cdot e^{-\frac{t^2}{2}}}{\PP[\mathcal E]} U.
\end{align*}
Plugging this equality into \eqref{eq:E_w_xx_lower_temp} and using $\PP[\mathcal E] = \PP[|x_u| \leq t] = 2\Phi(t)-1$ gives
\begin{align*}
\E\left[w(x,z)\cdot xx^\top \mid \mathcal E\right] &\succeq
\frac{1}{1+(t\|u\| + z_\text{max})^2} \left((2\Phi(t)-1) \cdot I_d - \sqrt{\frac{2}{\pi}} t\cdot e^{-\frac{t^2}{2}} \cdot U \right) .
\end{align*}
The RHS is, up to scaling, a rank one perturbation of the identity matrix. Hence, its smallest singular value is
\begin{align*}
\sigma_\text{min}\left( (2\Phi(t)-1) I_d - t\sqrt{\frac{2}{\pi}} e^{-t^2/2} U \right) = 2\Phi(t)-1 - t\sqrt{\frac{2}{\pi}}  e^{-t^2/2}.
\end{align*}
The lower bound in \eqref{eq:E_wxx_bound} follows by picking $t=3$, as $2\Phi(3)-1 - 3\sqrt{\frac{2}{\pi}} e^{-9/2} > 24/25$.

Finally, the upper bound in \eqref{eq:E_wxx_bound} follows trivially by $0 \leq w(x,z) \leq 1$ and $\E\left[xx^\top\right] = I_d$.
\end{proof}

\begin{proof}[Proof of \cref{lem:Dt}]
Fix some $t$, and denote the iterates at time steps $t$ and $(t+1)$ by $\beta_1 = \beta_1^{(t)}$ and $\beta_1^+ = \beta_1^{(t+1)}$, respectively.
Further denote $D = D_t$.
According to \eqref{eq:IRLS_beta},
\begin{align*}
\beta_1^+ &= \left(\sum_{i=1}^n w_{i,1} x_i x_i^\top \right)^{-1} \left(\sum_{i=1}^n w_{i,1} x_i y_i \right) .
\end{align*}
Recall that the cluster assignment of a sample, $c^*_i \in \{1,2\}$, is distributed as Bernoulli with probabilities $p_1, p_2$, independently of the sample $x_i$.
Let $r(x,c^*,\epsilon; \beta_1) = \left|x^\top \beta_1 - y(c^*, \epsilon)\right|$ and $w(x,c^*,\epsilon; \beta_1) = 1/\left(1 + \eta r(x,c^*,\epsilon; \beta_1)^2/R\right)$ where $y(c^*, \epsilon) = x^\top \beta^*_{c^*} + \epsilon$.
For simplicity of notation, from now on we suppress the dependencies of $r, w$ and $y$ on $x, c^*, \epsilon$ and $\beta_1$.

By the weak law of large numbers, as $n\to \infty$, the terms $\left(\sum_{i=1}^n w_i x_i x_i^\top \right)^{-1}$ and $ \left(\sum_{i=1}^n w_i x_i y_i \right)$ tend to $\E \left[w\cdot x x^\top \right]^{-1}$ and $\E[ w\cdot x y]$, respectively, with the expectation taken over $x$, $c^*$ and $\epsilon$.
Hence, in the population setting, the update of $\beta_1$ takes the form
\begin{align*}
\beta_1^+ \stackrel{n\to\infty}{\rightarrow}
\E \left[w\cdot x x^\top \right]^{-1} \E[ w\cdot x y] .
\end{align*}
Observe that the true regression vector $\beta_1^*$ can be written as
\begin{align*}
\beta_1^*
= \left(\E \left[w\cdot x x^\top \right] \right)^{-1} \E\left[ w\cdot x x^\top\right]  \beta_1^*
= \left(\E \left[w\cdot x x^\top \right] \right)^{-1} \E\left[ w\cdot x x^\top \beta_1^*\right].
\end{align*}
Hence, the distance of the next iterate $\beta_1^+$ from $\beta_1^*$ satisfies
\begin{align}
\|\beta_1^+ - \beta^*_1\| &= \left\|\left(\E \left[w\cdot x x^\top \right] \right)^{-1} \E\left[ w\cdot x (y - x^\top \beta_1^*)\right] \right\|
\leq \frac{\left\| \E\left[ w\cdot x (y - x^\top \beta_1^*)\right] \right\|}{\sigma_\text{min}\left(\E \left[w\cdot x x^\top \right] \right)}.
\label{eq:next_error}
\end{align}

We first analyze the denominator of the RHS in \eqref{eq:next_error}.
Denote $\delta = \beta_1 - \beta_1^*$, and for $k=1,2$ let $w_{c^*=k}$ be the weight conditional on the response $y$ having been generated from the $k$-th component. 
By the independence of $c^*$ from $x$ and $\epsilon$,
\begin{align}\label{eq:E_wxx_bound_temp}
\E\left[w \cdot xx^\top\right] &\succeq p_1 \E\left[w_{c^*=1} \cdot xx^\top \right] .
\end{align}
Now, by definition,
\begin{align*}
w_{c^*=1} = \frac{1}{1 + \eta (x^\top \beta_1 - x^\top \beta_1^* - \epsilon)^2/R} = \frac{1}{1 + \eta (x^\top \delta - \epsilon)^2/R} .
\end{align*}
Invoking Eq.~\eqref{eq:E_wxx_bound} of \cref{lem:E_bounds} with $u = \sqrt{\eta/R} \delta$, $z = \sqrt{\eta/R} \epsilon$ and $z_\text{max} = \sqrt{\eta/R} \xi \|\Delta\|$ gives
\begin{align}
\sigma_\text{min}\left(\E\left[w_{c^*=1} \cdot xx^\top\right]\right)
\geq \frac{24}{25} \frac{1}{1 + \eta (3\|\delta\| + \xi\|\Delta\|)^2/R}
\geq \frac{24}{25} \frac{1}{1 + \tilde\eta (3D + \xi)^2}.  \label{eq:sigmaMin_E_w1xx_bound}
\end{align}
Plugging this into \eqref{eq:E_wxx_bound_temp} yields
\begin{align} \label{eq:sigmaMin_bound}
\sigma_\text{min}\left(\E\left[w \cdot xx^\top\right]\right) \geq \frac{24p_1}{25} \frac{1}{1 + \tilde\eta (3D+\xi)^2}.
\end{align}
Next, we upper bound the numerator of the RHS in \eqref{eq:next_error}.
Conditional on the response $y$ having been generated from the second component $(c^*=2)$, the weight satisfies
\begin{align*}
w_{c^*=2} = \frac{1}{1 + \eta (x^\top \beta_1 - x^\top \beta_2^* - \epsilon)^2/R} = \frac{1}{1 + \eta (x^\top (\delta+\Delta) - \epsilon)^2/R} 
\end{align*}
where $\Delta = \beta_1^* - \beta_2^*$.
By the triangle inequality,
\begin{align*}
\|\E\left[w\cdot x (y - x^\top\beta^*_1)\right]\| &= \|p_1 \E\left[w\cdot x (y - x^\top\beta_1^*) \mid c^*=1] + p_2 \E[w\cdot x (y - x^\top\beta_1^*) \mid c^*=2 \right]\| \nonumber \\
&\leq p_1\|\E\left[w_{c^*=1}\cdot \epsilon x \right]\| + p_2 \|\E\left[w_{c^*=2}\cdot x (\epsilon - x^\top \Delta) \right]\| \nonumber \\
&\leq p_1 \|\E\left[w_{c^*=1}\cdot \epsilon x \right]\| + p_2 \|\E\left[w_{c^*=2}\cdot \epsilon x \right]\| + p_2 \|\E\left[w_{c^*=2}\cdot xx^\top \right] \| \cdot \| \Delta\| .
\end{align*}
We now employ \cref{lem:E_bounds} to bound each of the three terms on the RHS.
The first term is bounded using \eqref{eq:E_wx_bound} with $u = \sqrt{\eta/R} \delta$, $z = \sqrt{\eta/R} \epsilon$ and $z_\text{max} = \sqrt{\eta/R} \xi \|\Delta\|$. The second term is bounded using \eqref{eq:E_wx_bound} with $u = \sqrt{\eta/R} (\delta+\Delta)$ and the same $z, z_\text{max}$. The third term is bounded using \eqref{eq:E_wxx_bound} with the same $u, z, z_\text{max}$.
Putting everything together, we obtain
\begin{align*}
\|\E\left[w\cdot x (y - x^\top\beta^*_1)\right]\| &\leq \sqrt\frac{2}{\pi} (p_1+p_2) \xi \|\Delta\| + p_2 \|\Delta\| = \left(\sqrt\frac{2}{\pi} \xi + p_2 \right) \|\Delta\|.
\end{align*}
Plugging this, together with \eqref{eq:sigmaMin_bound}, into \eqref{eq:next_error}, gives
\begin{align*}
\|\beta_1^+ - \beta_1^*\| &\leq \frac{25}{24p_1} \left(\sqrt{2/\pi} \xi + p_2\right) \left(1 + \tilde \eta (3D+\xi)^2 \right) \|\Delta\|
\leq \frac{25}{24p_1} \left(4\xi/5 + p_2\right) \left(1 + \tilde \eta (3D+\xi)^2 \right) \|\Delta\| \\
&= \frac{5}{6} (q-\tilde\xi) \left(1 + \tilde \eta (3D+\xi)^2 \right) \|\Delta\|,
\end{align*}
where in the last equality we used the definition $q = (5p_2 + 4\xi)/(4p_1) + \tilde\xi$ \eqref{eq:bounded_balance_thm}.
\end{proof}

\subsection{Proof of \cref{lem:P}}
\begin{proof}[Proof of \cref{lem:P}]
Suppose $\|u\|/\|\Delta\| \leq q - \tilde\xi$ and $\|x\|^2 \leq R$.
Let $x_u = x^\top u/\|u\|$, so that $x^\top u = x_u \|u\|$.
Recall that $\tilde\xi = \xi/\sqrt R$ \eqref{eq:xi}.
Since $q > \tilde\xi$,
\begin{align*}
\frac{\left|x^\top u + \epsilon \right|}{\|\Delta\|} = \frac{\left|x_u \|u\| + \epsilon\right|}{\|\Delta\|} &\leq |x_u| (q-\tilde\xi) + \xi \leq \sqrt R (q-\tilde\xi) + \xi = \sqrt{R}q.
\end{align*}
Hence,
\begin{align*}
w(x,\epsilon) &\geq \frac{1}{1 + \eta q^2 \|\Delta\|^2} > w_\text{th},
\end{align*}
where the second inequality follows by \eqref{eq:wth_assumption}. This proves the first part of the lemma.

Next, suppose $\|u\|/\|\Delta\| \geq 1 - q + \tilde\xi$ and $\|x\|^2 \leq R$.
Let $\delta = 1 - |x_u|/\sqrt R \geq 0$.
Then
\begin{align*}
\frac{\left|x^\top u + \epsilon \right|}{\|\Delta\|} &= \frac{\left|x_u \|u\| + \epsilon\right|}{\|\Delta\|} \geq |x_u| (1-q+\tilde\xi) - \xi \\
&= \sqrt{R}(1-\delta) (1-q) - \delta \xi \\
&= \sqrt{R} \left((1-\delta)(1-q) - \delta \xi / \sqrt R \right),
\end{align*}
so that
\begin{align*}
w(x,\epsilon) \leq \frac{1}{1 + \eta\left((1-\delta)(1-q) - \delta \xi / \sqrt R \right)^2 \|\Delta\|^2} .
\end{align*}
The RHS is monotonically increasing in $\delta$.
For $\delta = 0$, we get $w(x) \leq 1/\left(1 + \eta (1-q)^2 \|\Delta\|^2\right)$.
Hence, for any $\zeta > 0$, there exists a sufficiently small $\delta>0$ such that
\begin{align*}
w(x,\epsilon) - \frac{1}{1 + \eta (1-q)^2 \|\Delta\|^2} < \zeta .
\end{align*}
Let $\zeta = w_\text{th} - 1/\left(1 + \eta(1-q)^2\|\Delta\|^2\right)$.
By \eqref{eq:wth_assumption}, $\zeta > 0$. Hence, there exists $\delta>0$ such that
\begin{align*}
w(x,\epsilon) - \frac{1}{1 + \eta (1-q)^2 \|\Delta\|^2} < w_\text{th} - \frac{1}{1 + \eta(1-q)^2\|\Delta\|^2},
\end{align*}
or, equivalently, $w(x,\epsilon) < w_\text{th}$.
The probability $P$ for this event satisfies
\begin{align*}
P = \PP[1 - |x_u|/\sqrt{R} \leq \delta \,\mid\, \|x\|^2 \leq R] = \PP[x_u^2 \geq (1-\delta)^2 R \,\mid\, \|x\|^2 \leq R] > 0.
\end{align*}
\end{proof}

\section{Additional Simulation and Experimental Details}\label{sec:additional_details}
All algorithms get as input a maximal number of iterations, and \MIRLST and \GD have additional parameters. The maximal number of iterations in \MIRLS, \MIRLST, \AltMin and \EM was set to $10^3$, and to $10^5$ in \GD.
The parameter $\rho$ of \MIRLS and \MIRLST was fixed at the value of $1$ in synthetic simulations and $2$ in real-world experiments.
To tune $\eta$ and $w_\text{th}$ of \MIRLST and the step size $\eta_\text{GD}$ of \GD, we run each simulation and experimental setting with a different set of values, and choose the best values out of 10 repetitions. The allowed values were: $\eta = \sqrt{\Phi^{-1}(0.75) / \nu} = \sqrt{0.6745 / \nu}$ where $\nu \in \{0.1, 0.5, 1, 2\}$, $w_\text{th} \in \{0.01, 0.1, 0.5, 0.75\}$, and $\eta_\text{GD} \in \{10^{-5}, 5\cdot 10^{-4}, 10^{-4}, \ldots, 5\cdot 10^{-1}, 10^{-1}\}$. In the untuned version of \MIRLS, we used the fixed values $\nu = 0.5$ and $w_\text{th} = 0.01$ for simulations, and $\nu = 1, w_\text{th} = 0.01$ for experiments.
\EM was initialized with noise levels $\sigma_1^{(0)}, \ldots, \sigma_K^{(0)}$ set to one as in \cite{diamandis2021wasserstein}, and with uniform mixture proportions $p^{(0)} = (1/K, \ldots, 1/K)$. We remark that initializing \EM with the exact noise levels does not improve its performance. In addition, as discussed the main text, \EM hardly improves given prior information of the true proportions $p$.

In all algorithms, we employed the same following stopping criterion: if the estimate does not change much between subsequent iterations,
\begin{align*}
\frac{\sum_{k=1}^K \|\beta_k^{(t)} - \beta_k^{(t-1)}\|^2}{\sum_{k=1}^K \|\beta_k^{(t)}\|^2} < \delta^2
\end{align*}
where $\delta$ is a tolerance constant, the algorithm is stopped.
The tolerance $\delta$ is set to $\tilde \delta \equiv \min(1, \max\{0.01 \sigma, 2 \epsilon_\text{mp}\})$ in \MIRLS, \AltMin and \EM, and to $0.01 \tilde\delta$ in \GD.

\paragraph{Additional simulation details.}
As described in the main text, the failure probability is defined as the percentage of runs with $F_\text{latent} > F_\text{th} \equiv 2\sigma$. Let us explain the choice of threshold $2\sigma$. As the numerical factor $2$ is arbitrary and the results are insensitive to its choice, we focus on the scaling with $\sigma$.
In well-defined standard linear regression (namely, with sample size above the information limit), the OLS error scales as $\sigma \sqrt{d/n}$. This quantity ignores the condition number of $XX^\top$, as it is close to one following our normality assumption \eqref{eq:gaussian_assumption}. However, scaling $F_\text{th}$ with $\sigma \sqrt{d/n}$ would make the failure probability invariant to the sample size $n$. Since we want to see how the different methods improve with increasing sample sizes, we set $F_\text{th}$ to be proportional to the OLS error at the information limit $n = d$, so that it scales as $\sigma$.
In MLR, our error measure \eqref{eq:objective_latent} scales as $(\sigma/K) \left(\sqrt{d/n_1} + \ldots \sqrt{d/n_K}\right)$, where $n_k$ is the number of sample that belong to component $k$. At the information limit $n = d/\text{min}(p) \equiv d/p_K$, we get the scaling $(\sigma/K) \cdot \sqrt{p_K} \left(1/\sqrt{p_1} + \ldots + 1/\sqrt{p_K}\right)$. This quantity is upper bounded by $\sigma$. Moreover, for all mixture proportions considered in our paper, this quantity is lower bounded by $\sigma/2$. Hence, also in MLR, the error threshold $F_\text{th}$ scales with $\sigma$.

We remark that empirically, this definition of $F_\text{th}$ agrees with the observed critical sample sizes. At a critical sample size, the median error undergoes a phase transition from failed to successful recovery: e.g., in \cref{fig:varying_n_prob}, the critical sample size is $n \approx 3500$ for \MIRLS and $n\approx 8000$ for \EM and \GD. In the various figures, the failure probability at the critical sample size is roughly $50\%$, implying the consistency of the definition of $F_\text{th}$; compare, for example, the two panels in \cref{fig:varying_n_K4} or in \cref{fig:varying_n_K5_moderateImbalance}.

\paragraph{Additional experimental details.}
In the real-data experiments (\cref{sec:experiments}), we ignore nominal fields and consider only numeric and ordinal ones. Nominal fields with two categories are considered ordinal. \Cref{table:datasets} details the number of samples and the dimension in each dataset.
The data is mean-centered and normalized as follows: $x'_i \leftarrow (x'_i - \bar x'_i) / \|x'_i - \bar x'_i\|$ and $y \leftarrow (y - \bar y) / \|y - \bar y\|$, where $x'_i$ is the $i$-th column of $X$ and $\bar u$ represents the mean of a vector $u$. In all datasets except fish market, a bias (intercept) term was added. In fish market, such a term makes no physical sense, as a fish of zero dimension must weight zero.

Finally, we remark that several authors proposed tensor-based initialization methods for MLR \cite{chaganty2013spectral,sedghi2016provable,yi2016solving,zhong2016mixed}.
In this work, we focus on random initialization, as it is more frequently used in real-data applications.
For completeness, we also run the initialization procedure proposed in \cite{zhong2016mixed} using the code generously provided to us by the authors.
However, in the settings explored in this paper with limited number of samples, this initialization did not seem to be more accurate than a random one and did not improve the recovery results of the algorithms.

\renewcommand{\arraystretch}{1.1}
\begin{table*}[b]
\begin{minipage}{\textwidth} 
\renewcommand\thempfootnote{\arabic{mpfootnote}}
\caption{The number of samples $n$ and the dimension $d$ in each of the datasets, ignoring NaN samples and nominal fields.}
\centering
 \label{table:datasets}
 \begin{tabular}{c  c  c}
 \hline
 dataset name & number of samples $n$ & dimension $d$ \\
  \hline
 medical insurance cost\footnote{\url{http://www.kaggle.com/datasets/mirichoi0218/insurance}} & 1338 & 7 \\
  \hline
 red wine quality\footnote{\url{http://www.kaggle.com/datasets/uciml/red-wine-quality-cortez-et-al-2009}} & 1599 & 12 \\
 \hline
 WHO life expectancy\footnote{\url{http://www.kaggle.com/datasets/kumarajarshi/life-expectancy-who}} & 1649 & 21 \\
 \hline
 fish market\footnote{\url{http://www.kaggle.com/datasets/aungpyaeap/fish-market}} & 159 & 5 \\
 \hline
 \end{tabular}
 \end{minipage}
\end{table*}

\section{Additional Simulation Results}\label{sec:additional_simulations}
In \cref{fig:varying_n} of the main text, we showed the failure probability of the algorithms as a function of the sample size $n$. \Cref{fig:varying_n_time} depicts the median runtime (in seconds) of the algorithms in the same simulation. Except for \GD, the different algorithms have comparable runtimes. \Cref{fig:varying_n_prob} shows the median error of the algorithms.

The results for mixtures with $K=4$ and $K=5$, including median error, failure probability and runtimes, are depicted in \cref{fig:varying_n_K4,fig:varying_n_K5_moderateImbalance,fig:varying_n_K45_time}. Specifically, \cref{fig:varying_n_K5_moderateImbalance} and \cref{fig:varying_n_K45_time}(right) show the results for a moderately imbalanced mixture. In all these settings, \MIRLS significantly outperforms the compared methods.

\begin{figure}[t]
\centering
	\subfloat{
		\includegraphics[width=0.5\linewidth]{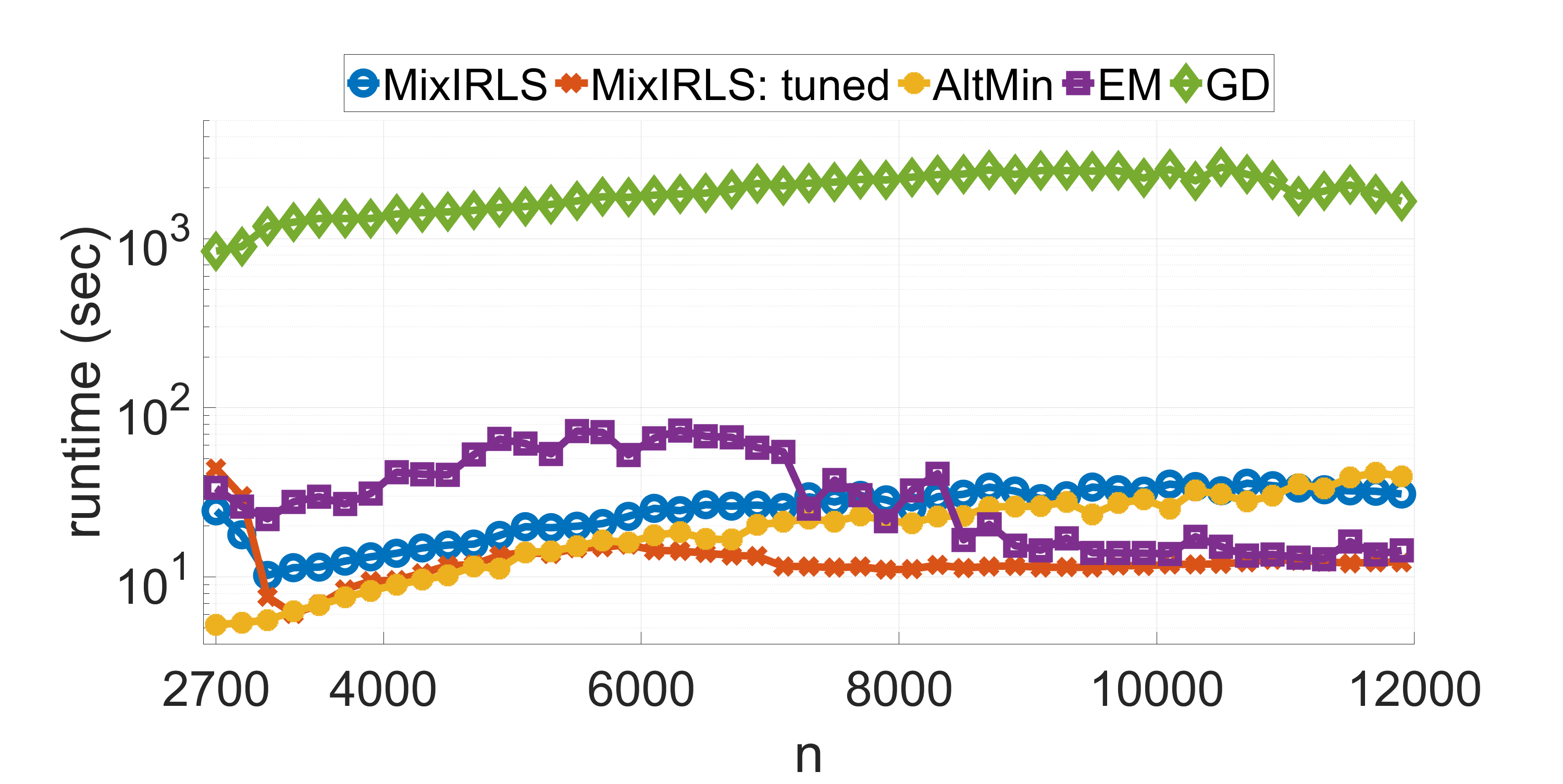}
	}
	\subfloat{
		\includegraphics[width=0.5\linewidth]{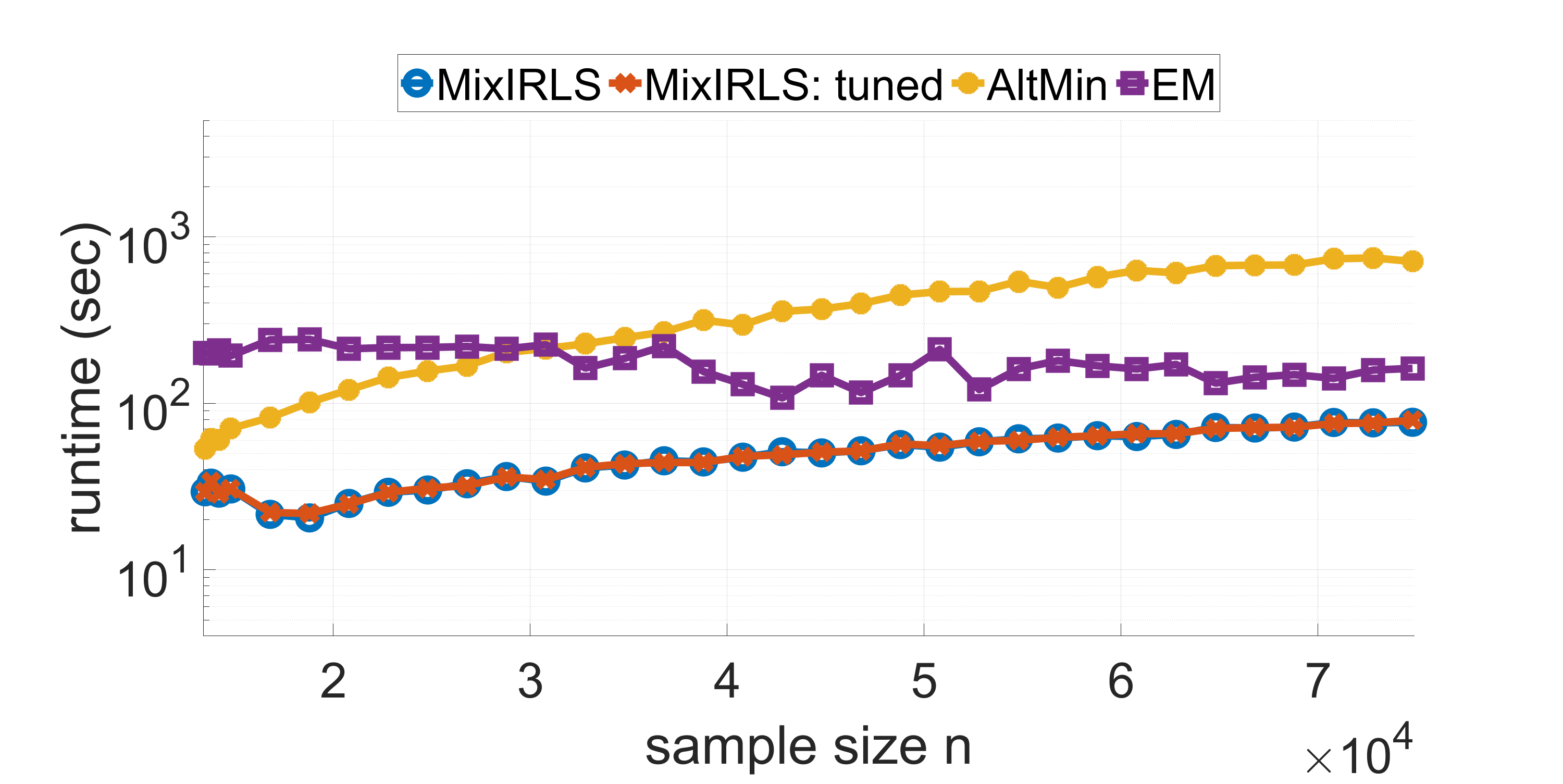}
	}
	\caption{Median runtime comparison in the setting of \cref{fig:varying_n}.}
\label{fig:varying_n_time}
\end{figure}

\begin{figure}[t]
\centering
	\subfloat{
		\includegraphics[width=0.5\linewidth]{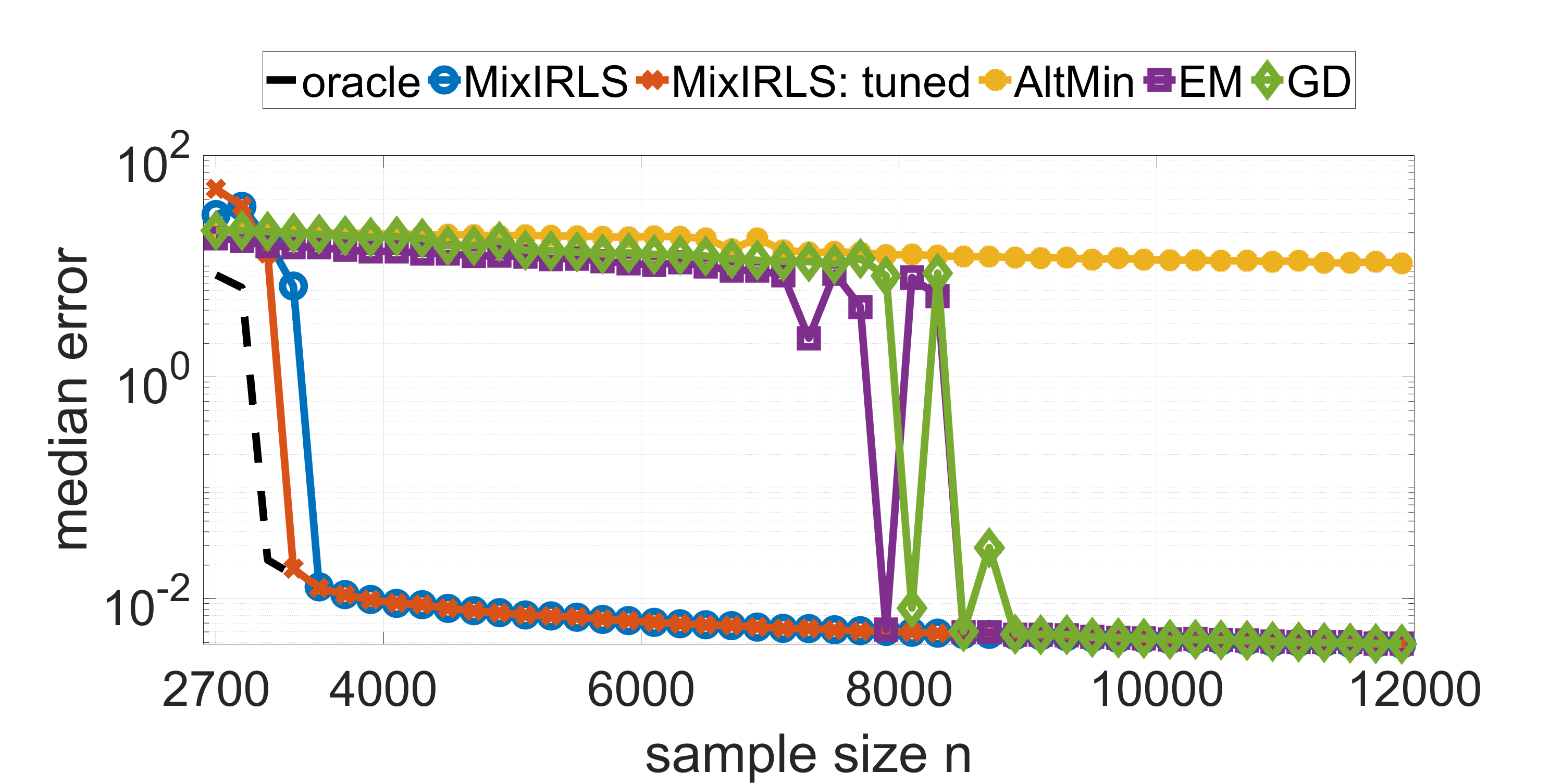}
	}
	\subfloat{
		\includegraphics[width=0.5\linewidth]{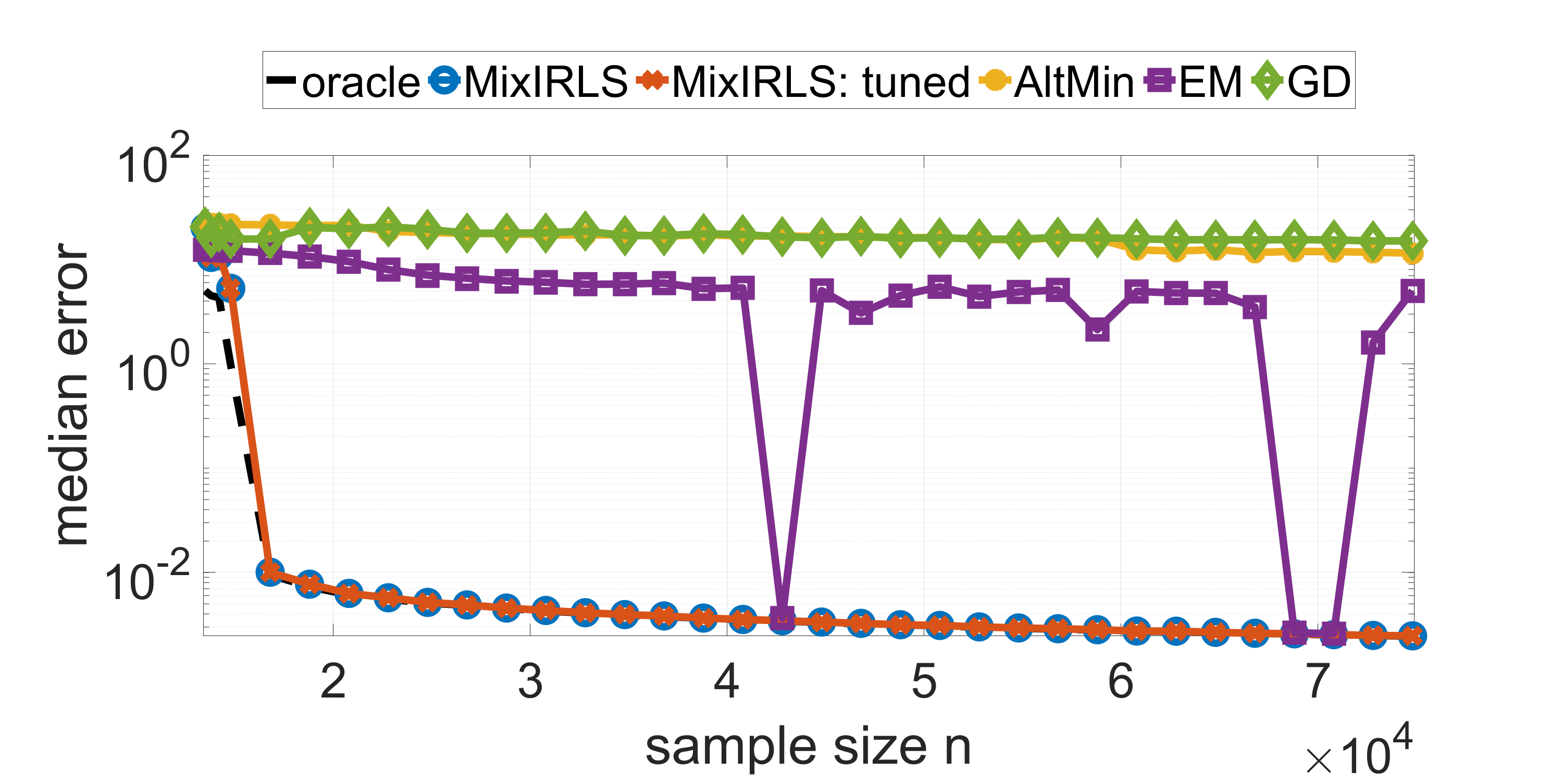}
	}
	\caption{Same setting as in \cref{fig:varying_n}, but with median errors instead of failure percentage.}
\label{fig:varying_n_prob}
\end{figure}

\begin{figure}[t]
\centering
	\subfloat{
		\includegraphics[width=0.5\linewidth]{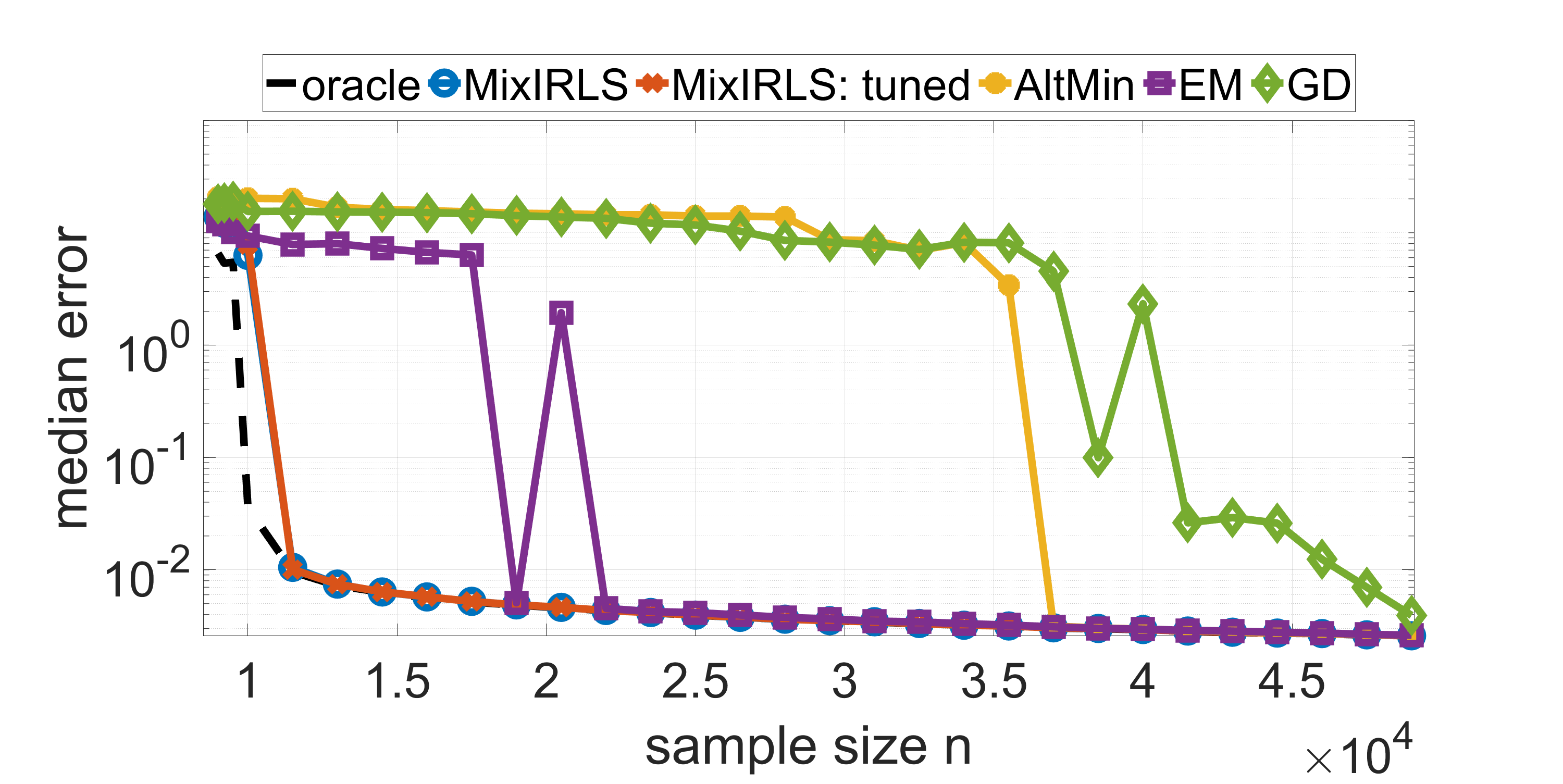}
	}
	\subfloat{
		\includegraphics[width=0.5\linewidth]{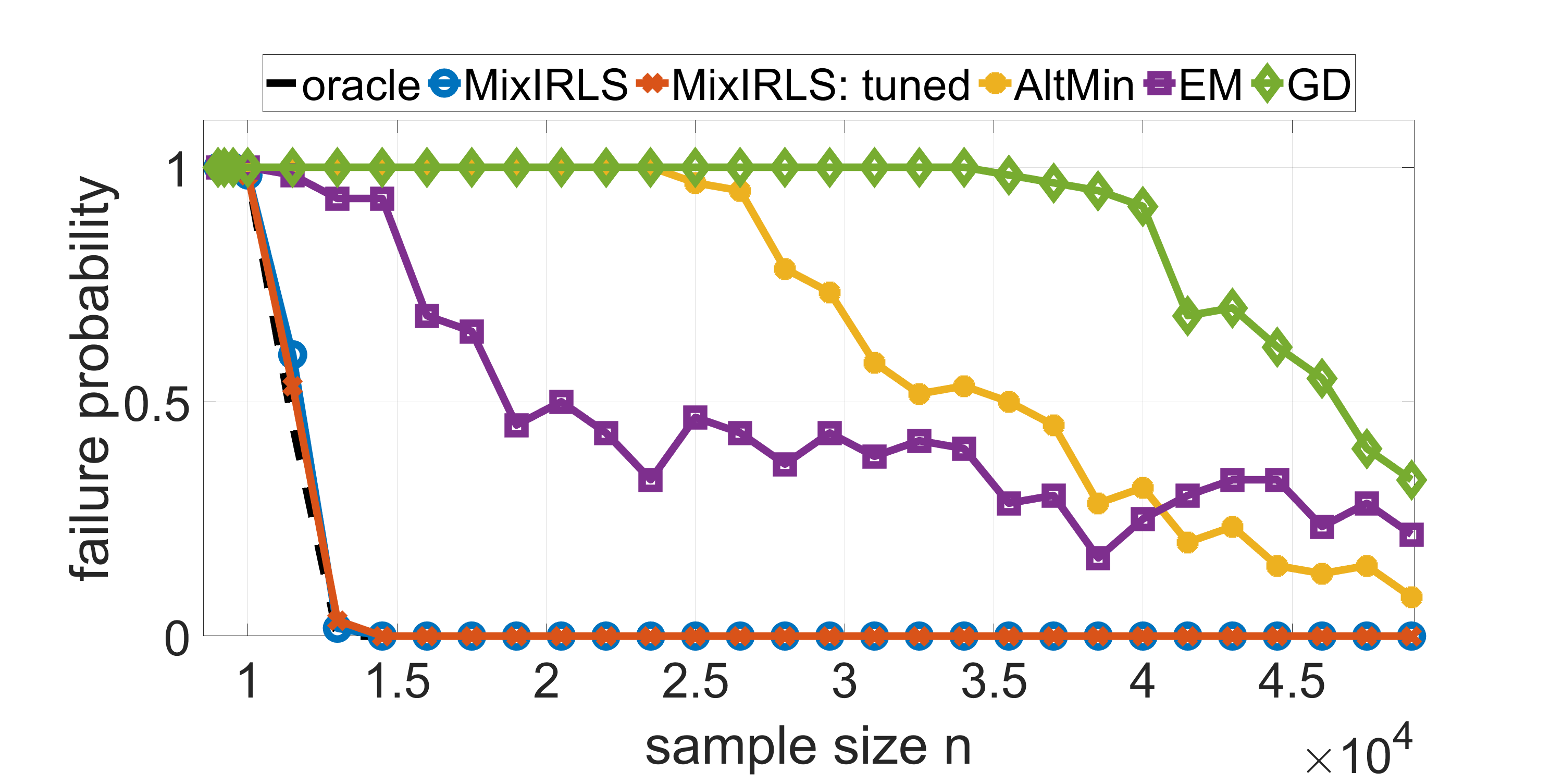}
	}
	\caption{Same setting as in \cref{fig:varying_n}, but with $K=4$ and $p = (0.67, 0.2, 0.1, 0.03)$.}
\label{fig:varying_n_K4}
\end{figure}

\begin{figure}[t]
\centering
	\subfloat{
		\includegraphics[width=0.5\linewidth]{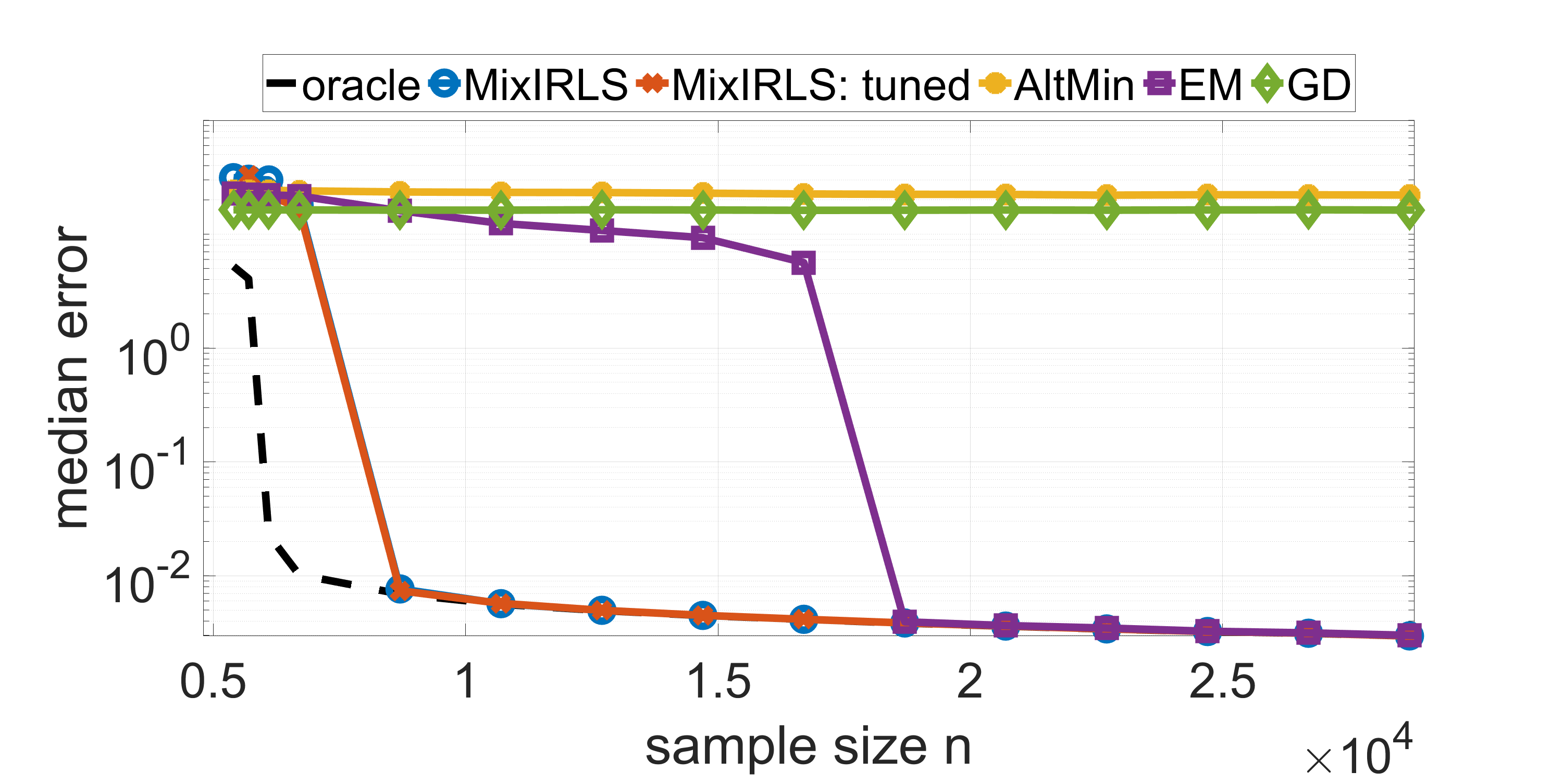}
	}
	\subfloat{
		\includegraphics[width=0.5\linewidth]{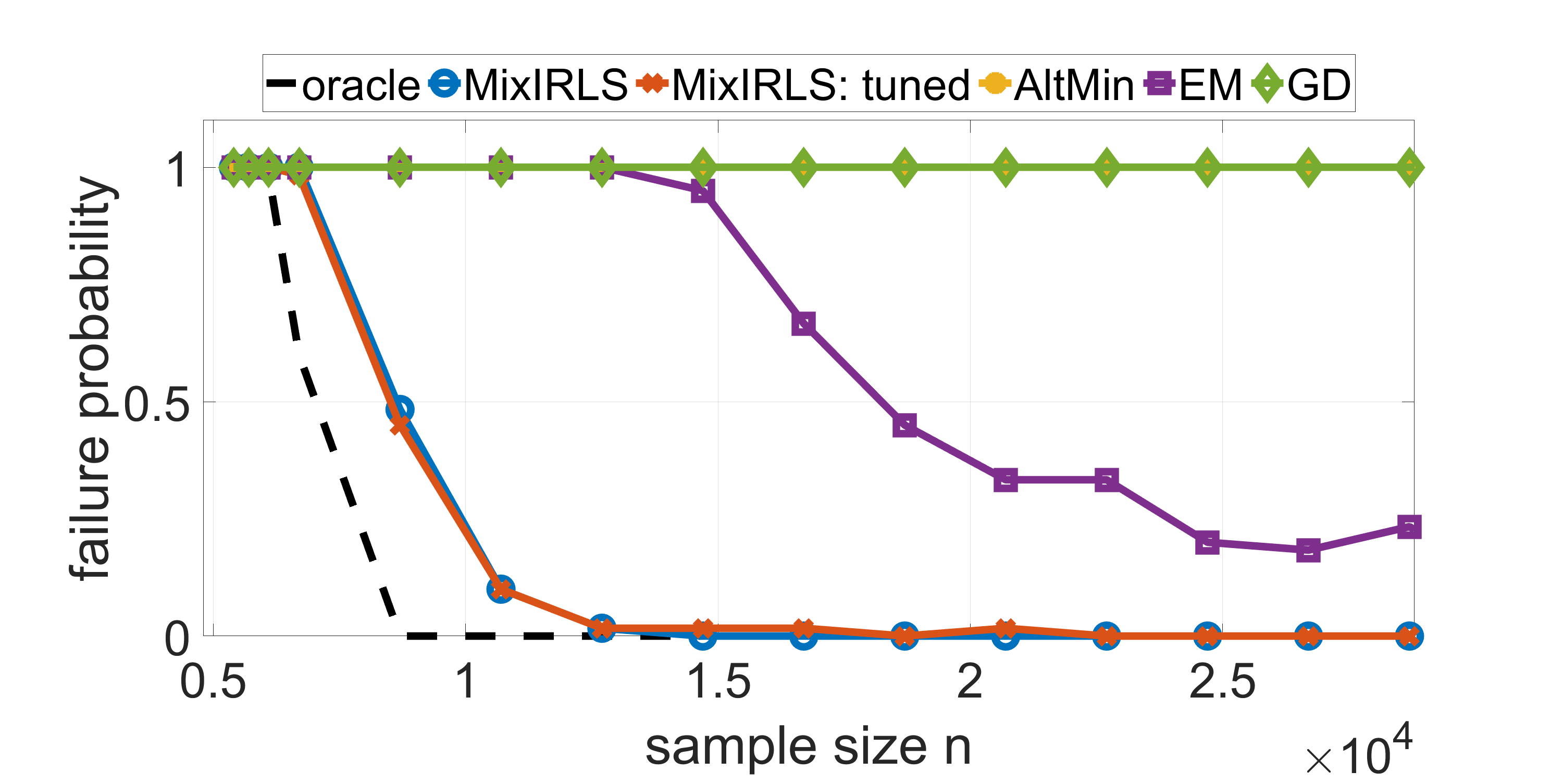}
	}
	\caption{Same setting as in \cref{fig:varying_n}, but with $K=5$ and $p = (0.4, 0.3, 0.15, 0.1, 0.05)$.}
\label{fig:varying_n_K5_moderateImbalance}
\end{figure}

\begin{figure}[t]
\centering
	\subfloat{
		\includegraphics[width=0.5\linewidth]{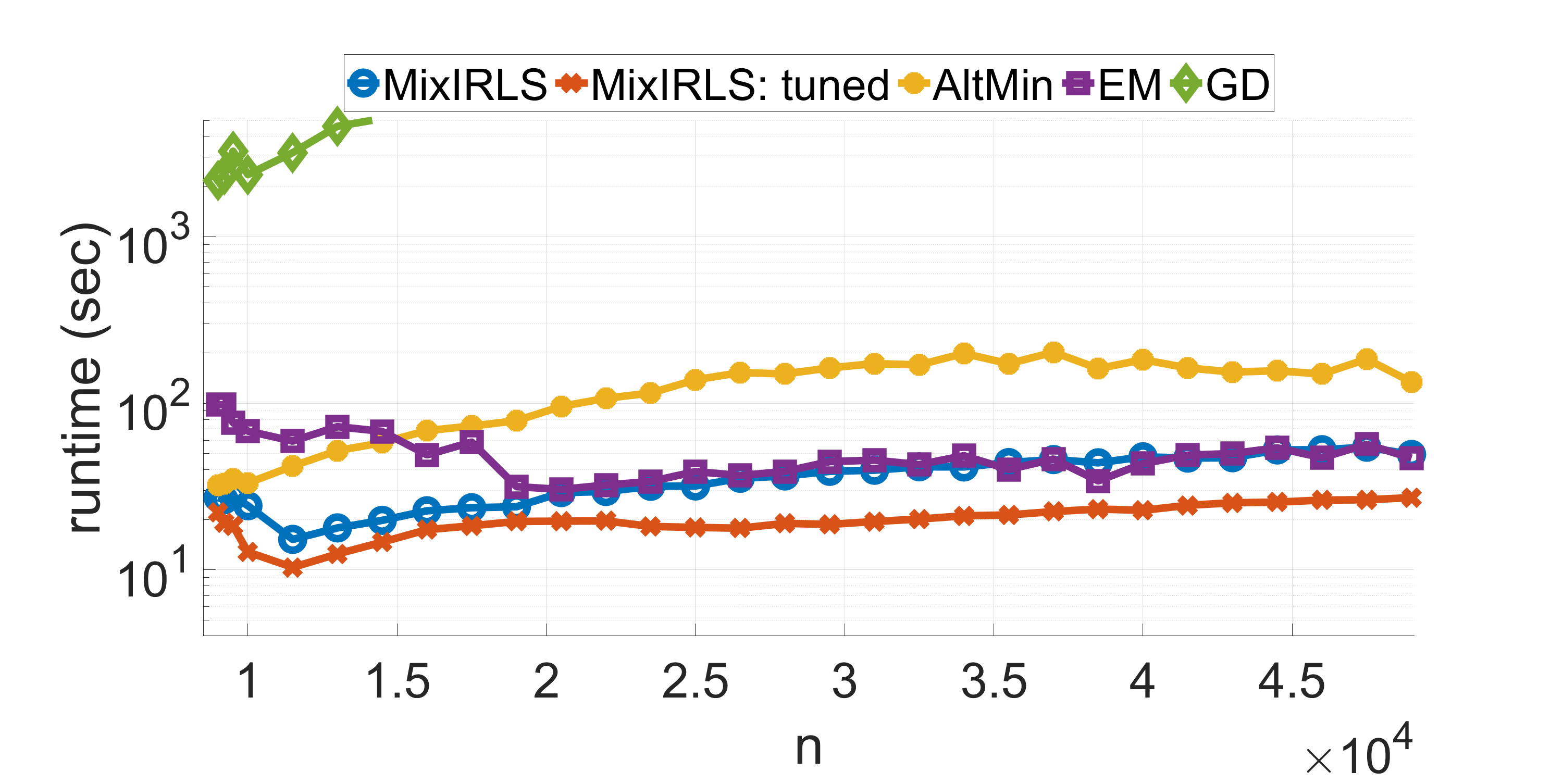}
	}
	\subfloat{
		\includegraphics[width=0.5\linewidth]{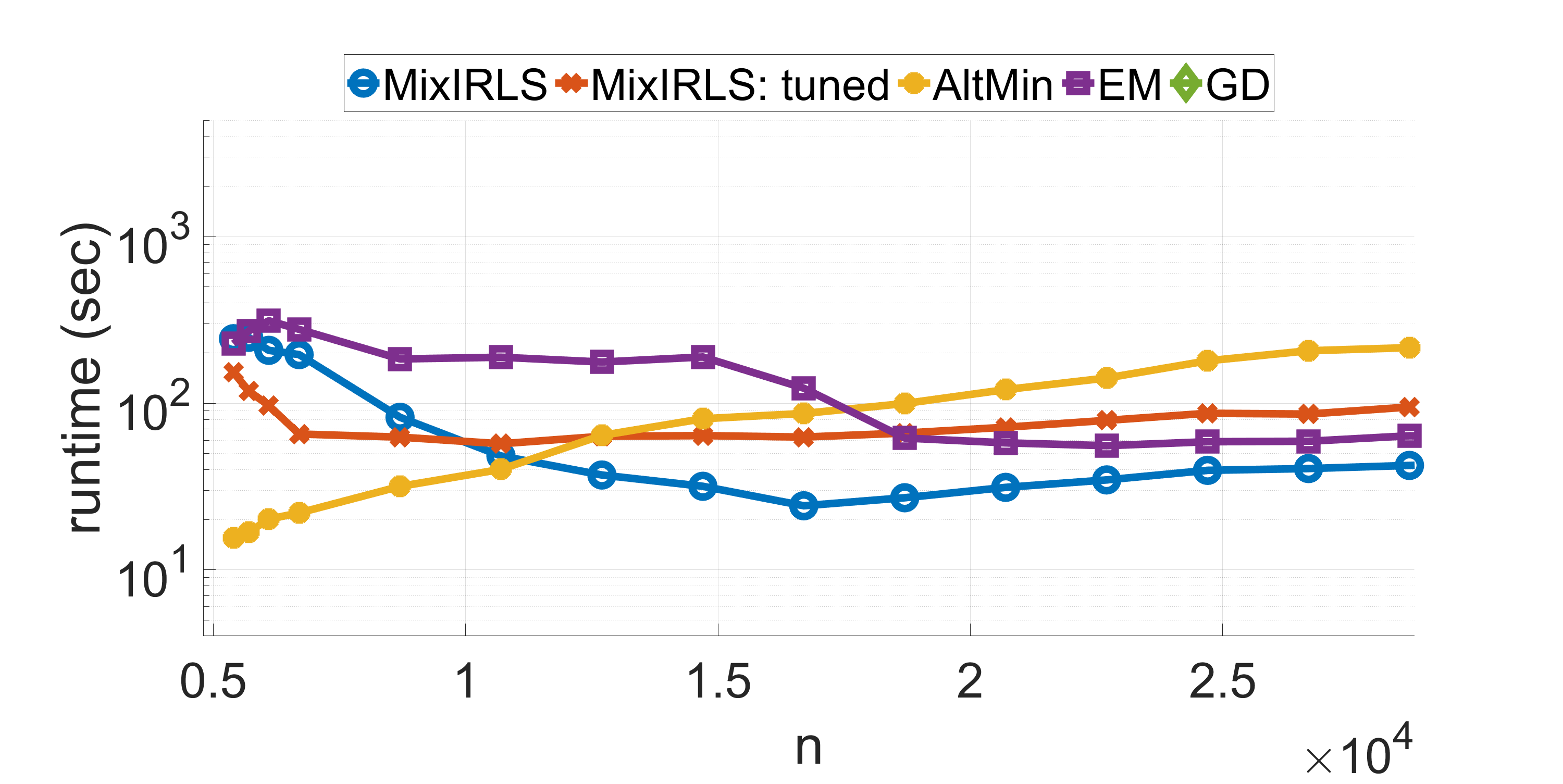}
	}
	\caption{Median runtimes comparison in the setting of \cref{fig:varying_n_K4} (left panel) and of \cref{fig:varying_n_K5_moderateImbalance} (right panel).}
	
\label{fig:varying_n_K45_time}
\end{figure}

A common approach to study performance of algorithms is a study of their phase transition from failure to success as a function of sample size and dimension.
\Cref{fig:heatmap} depicts the results of such a simulation, conducted in a noiseless setting $(\sigma = 0)$. The simulation covers a broad range of values on a 2D grid for both the sample size and the dimensions. As in previous simulations, \MIRLS recovers the linear models very close to the information limit, with negligible differences from the oracle's performance. In contrast, the other methods need much larger samples sizes to succeed in the recovery.
In this simulation, we additionally included our implementation of the \ILTS algorithm \cite{shen2019iterative}. In contrast to the other algorithms, \ILTS gets as input estimates for the mixture proportions $p_k$. In \cref{fig:heatmap}, we show the performance of an \texttt{ILTS:latent} version, which is supplied with the exact mixture proportions (this information is inaccessible to the other algorithms except for the oracle).

\begin{figure}[t]
\centering
	\subfloat{
		\includegraphics[width=0.5\linewidth]{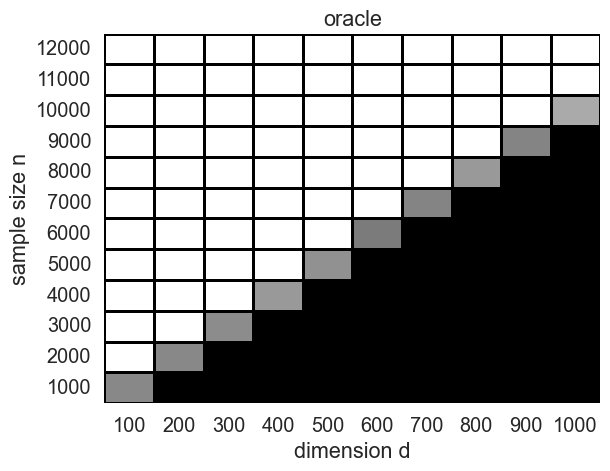}
	}
	\subfloat{
		\includegraphics[width=0.5\linewidth]{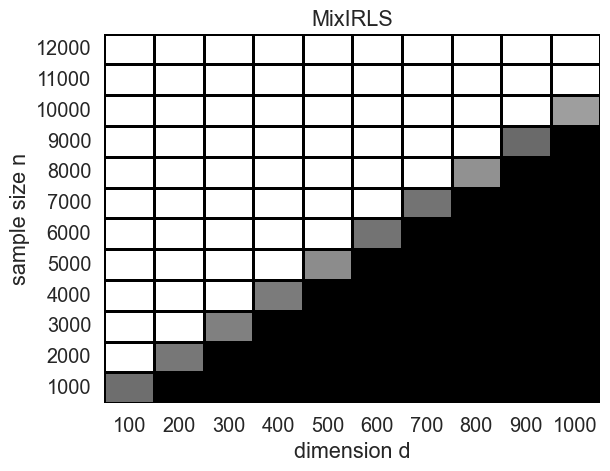}
	}
	\\
	\subfloat{
		\includegraphics[width=0.5\linewidth]{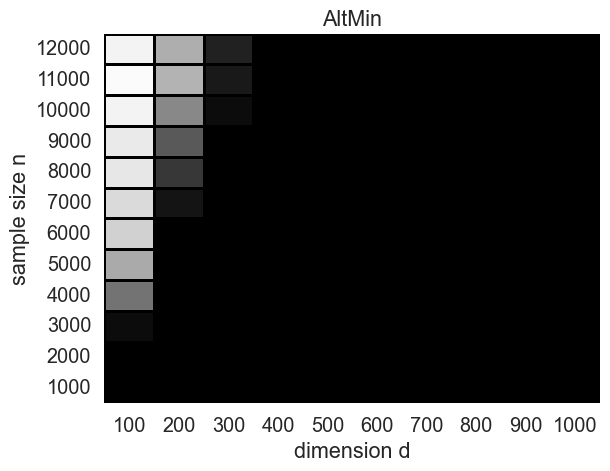}
	}
	\subfloat{
		\includegraphics[width=0.5\linewidth]{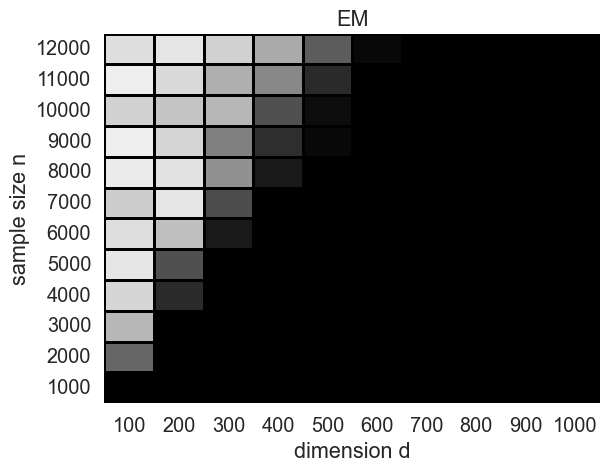}
	}
	\\
	\subfloat{
		\includegraphics[width=0.5\linewidth]{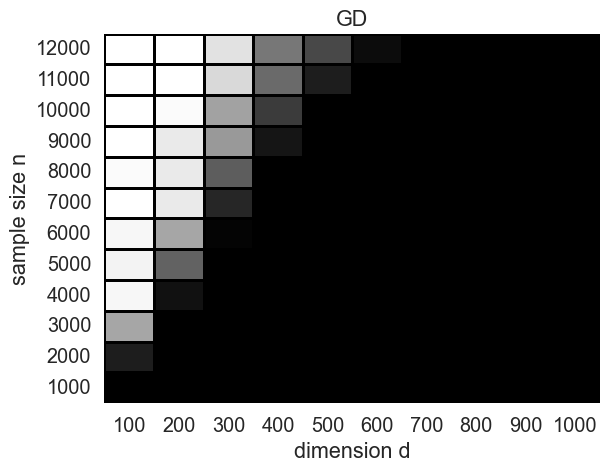}
	}
	\subfloat{
		\includegraphics[width=0.5\linewidth]{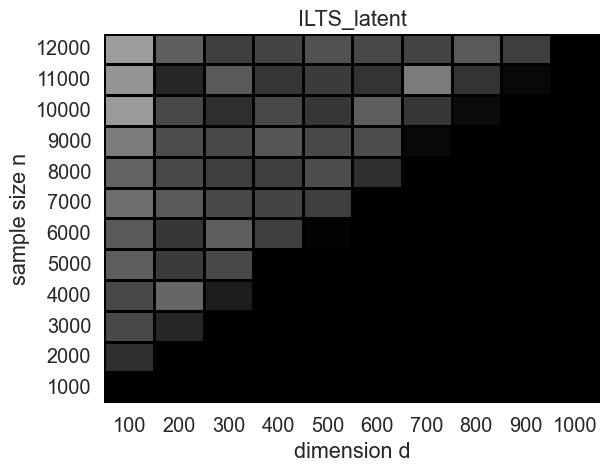}
	}
	\caption{The success percentage of various MLR algorithms, compared to an oracle, as a function of the dimension $d$ and the sample size $n$, with $K=3$, $p = (0.7, 0.2, 0.1)$ and no noise $\sigma=0$. White cell means 100\% success. Comparison of the top two panels show that \MIRLS recovery is nearly as good as that of the oracle, whereas other methods require many more samples to succeed. The result for \MIRLST is very similar to that of \MIRLS, and is thus omitted.}
\label{fig:heatmap}
\end{figure}

Next, we further explore the robustness of the algorithms beyond \cref{fig:robustness} in two ways.
First, \cref{fig:robustness} showed only the median error of the algorithms. In \cref{fig:robustness_prob}, we show the corresponding failure probability.
Second, \cref{fig:robustness} showed only the robustness to outliers and to overparameterization. In \cref{fig:noise}, we show the stability of the algorithms to varying noise levels.

\begin{figure}[t]
\centering
	\subfloat{
		\includegraphics[width=0.5\linewidth]{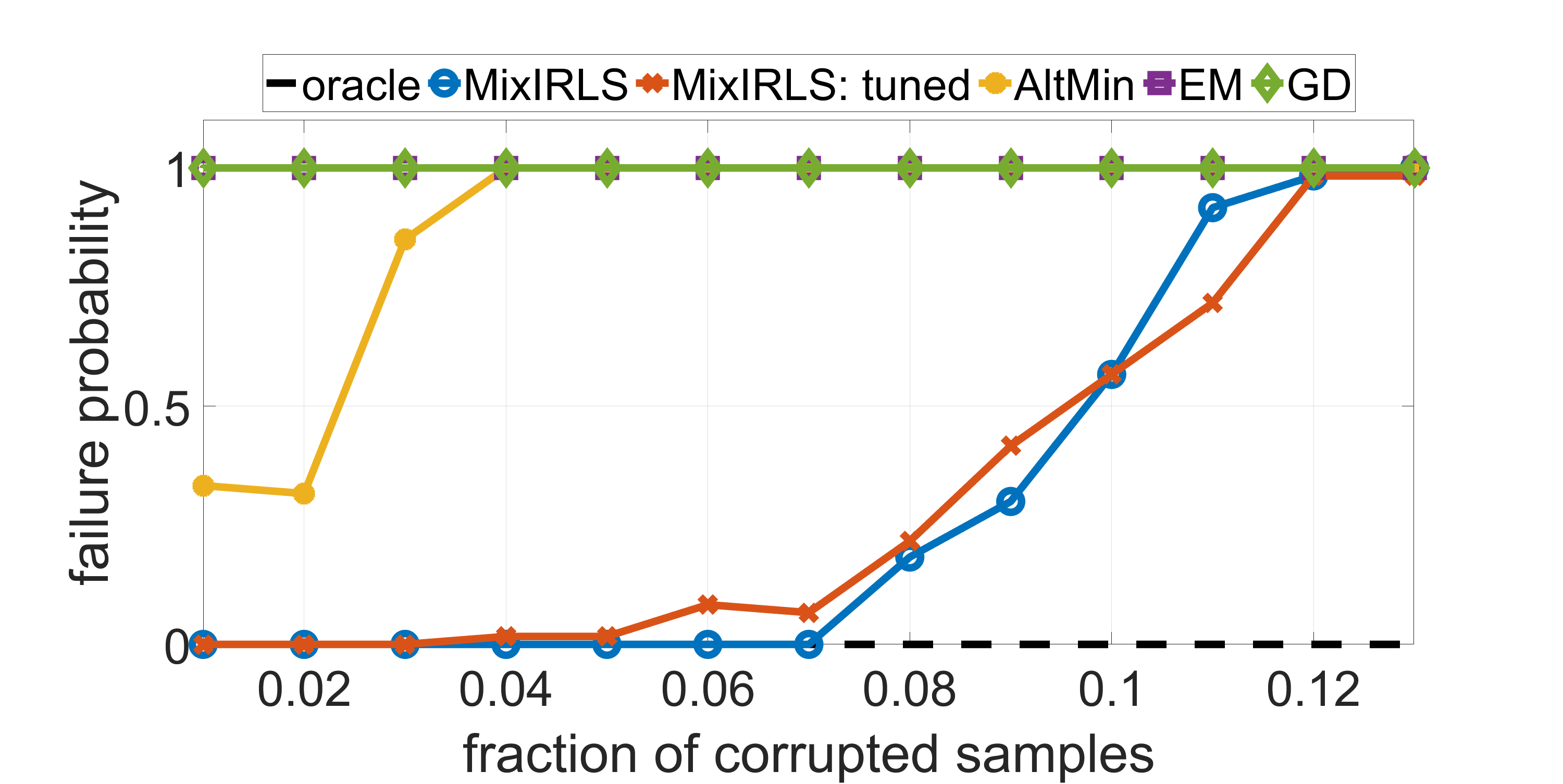}
	}
	\subfloat{
		\includegraphics[width=0.5\linewidth]{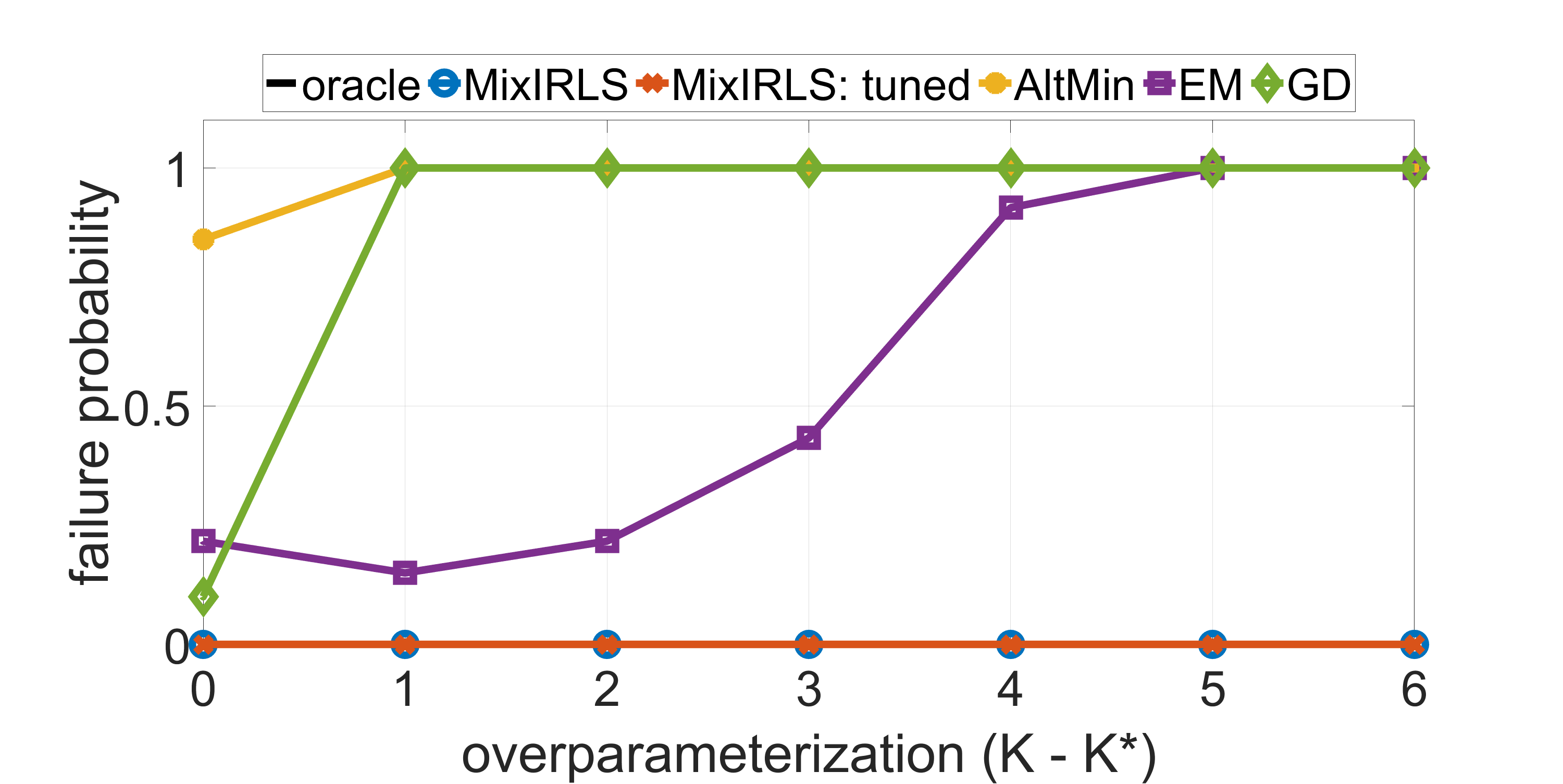}
	}
	\caption{Same setting as in \cref{fig:robustness}, but with y-axis showing the failure percentage instead of the median error.}
\label{fig:robustness_prob}
\end{figure}

\begin{figure}[t]
\centering
	\subfloat{
		\includegraphics[width=0.5\linewidth]{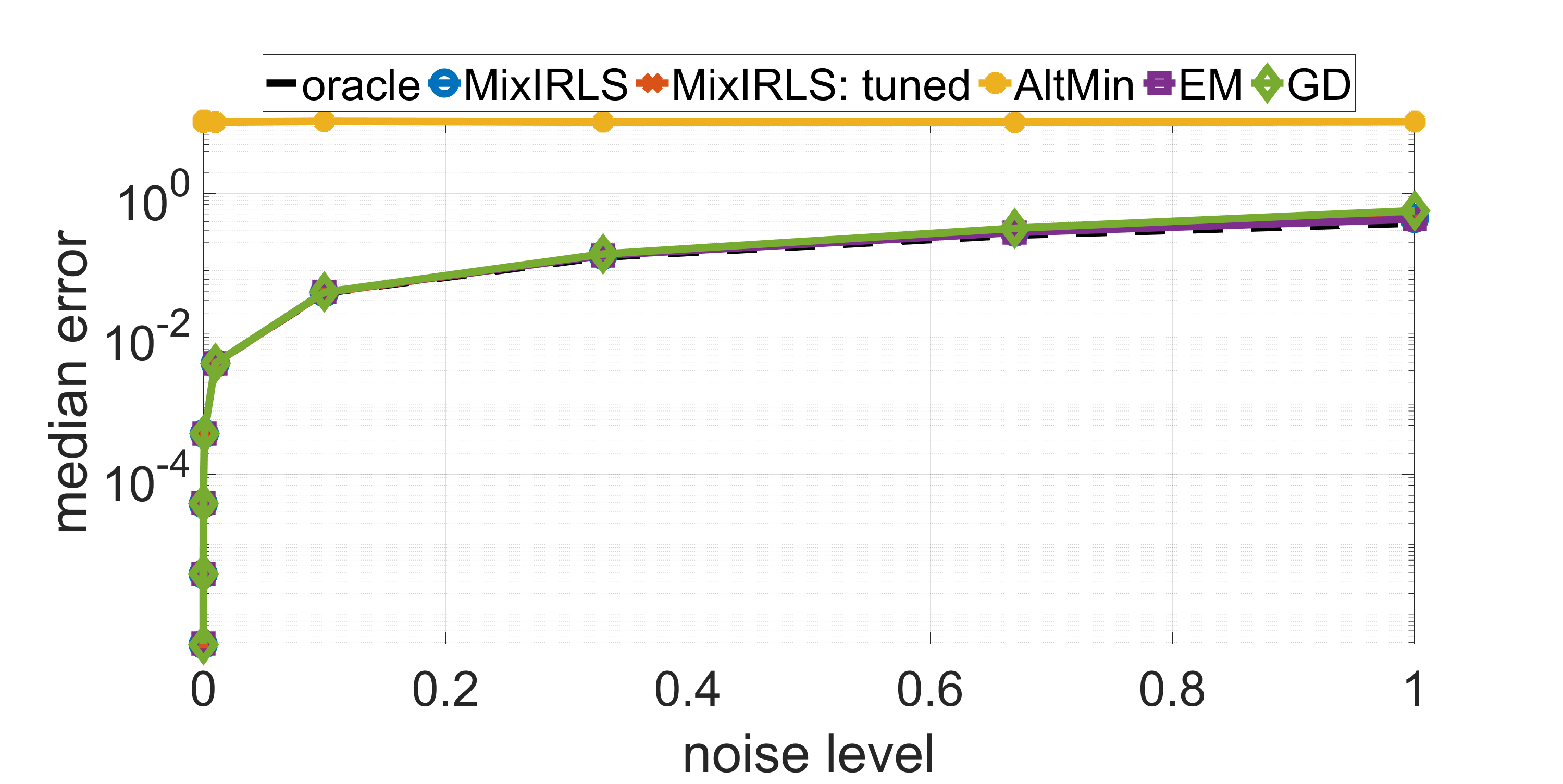}
	}
	\subfloat{
		\includegraphics[width=0.5\linewidth]{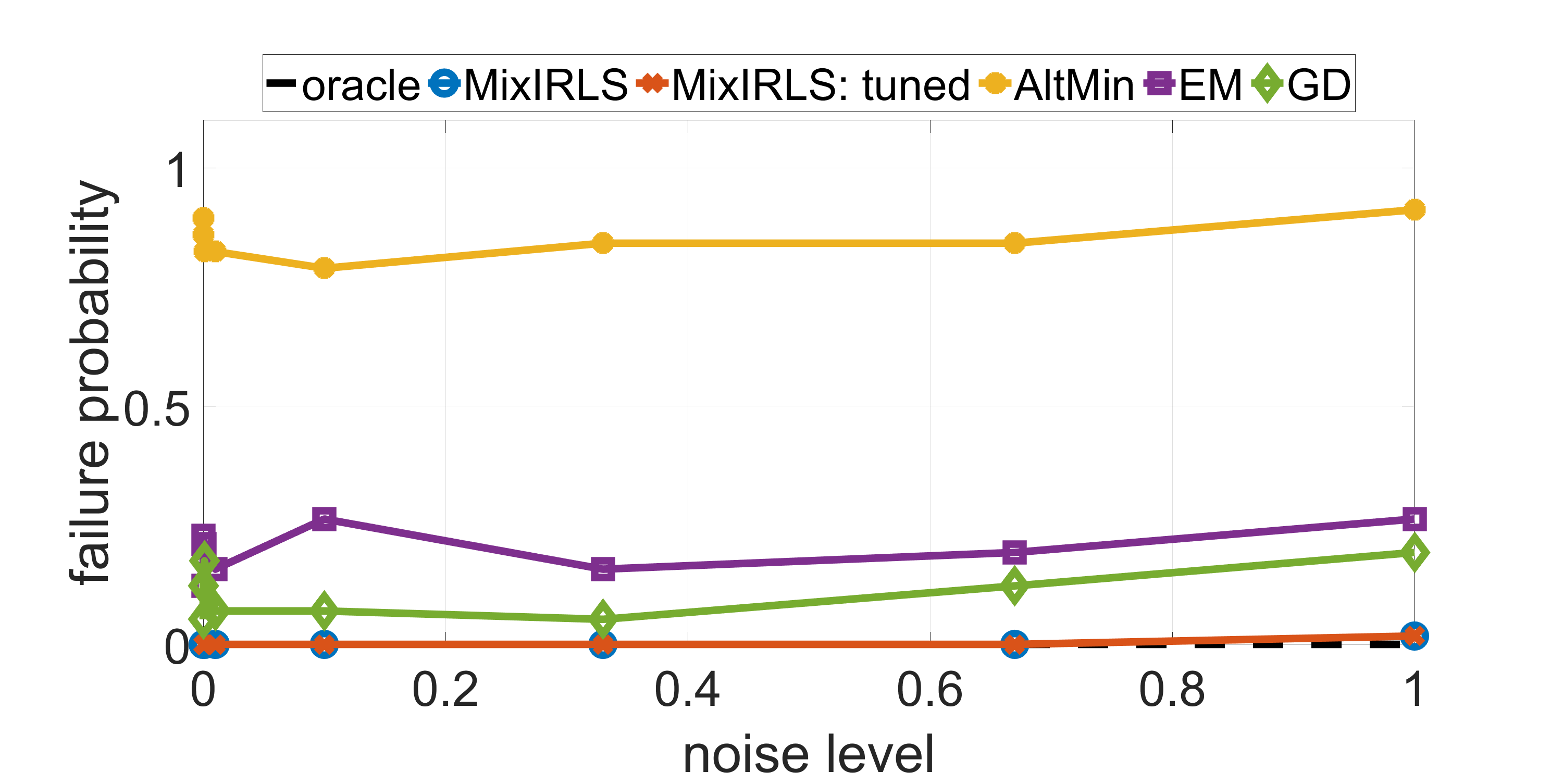}
	}
	\caption{Comparison of the stability of several MLR algorithms to additive Gaussian noise of mean 0 and varying standard deviation $\sigma \in [0,1]$, for the same values of $d, K$ and $p$ as in \cref{fig:robustness}.}
\label{fig:noise}
\end{figure}

Finally, we compare the performance of the algorithms on a perfectly balanced mixture, with $p = (1/3, 1/3, 1/3)$. The results, depicted in \cref{fig:balanced_varying_n,fig:balanced_corruptions,fig:balanced_overparam,fig:balanced_noise},
show that in this setting \MIRLS loses its advantage and performs comparably to other methods in terms of sample complexity, but holds its lead in terms of robustness to outliers and to overparameterization.

\begin{figure}[t]
\centering
	\subfloat{
		\includegraphics[width=0.5\linewidth]{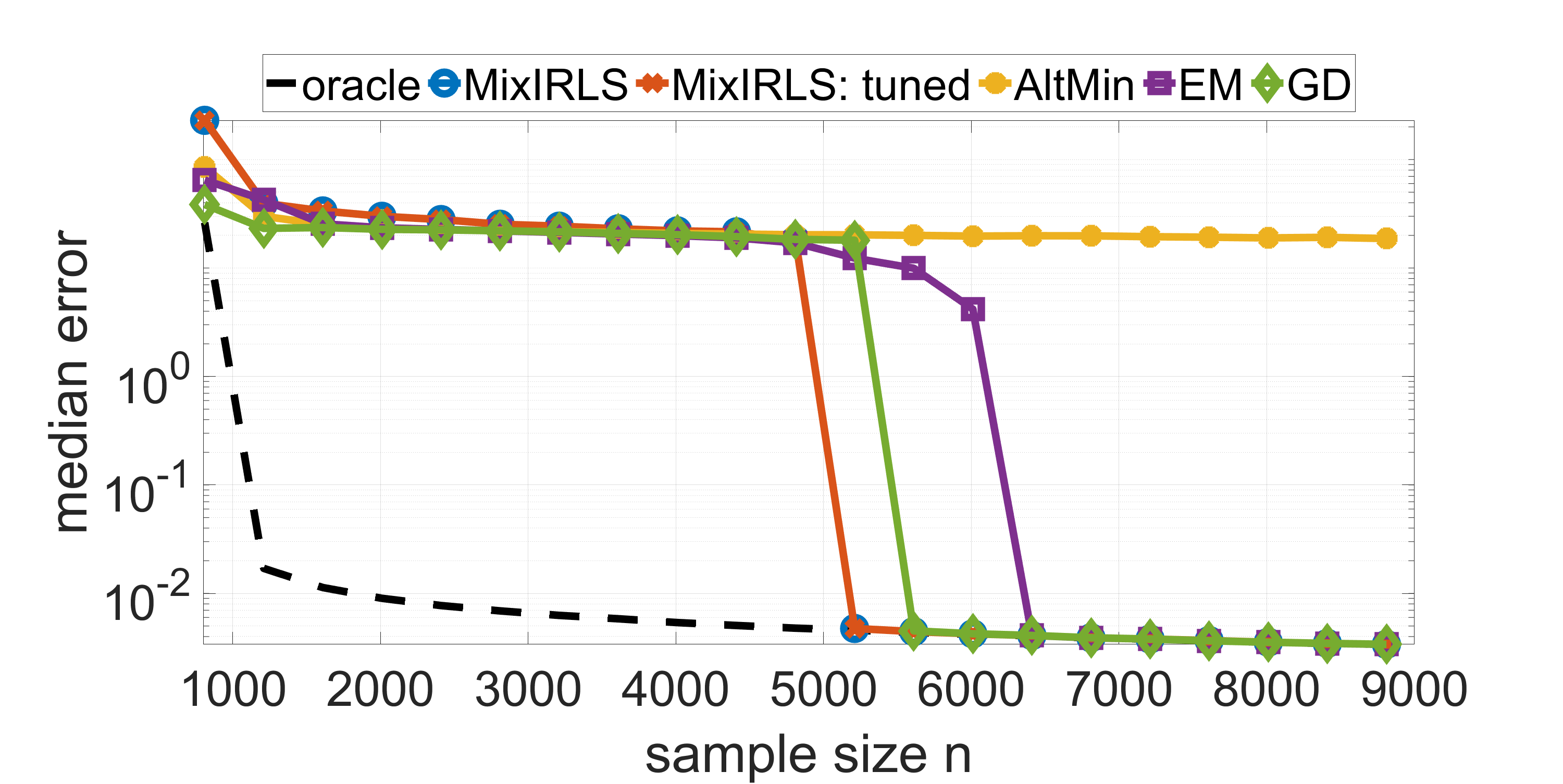}
	}
	\subfloat{
		\includegraphics[width=0.5\linewidth]{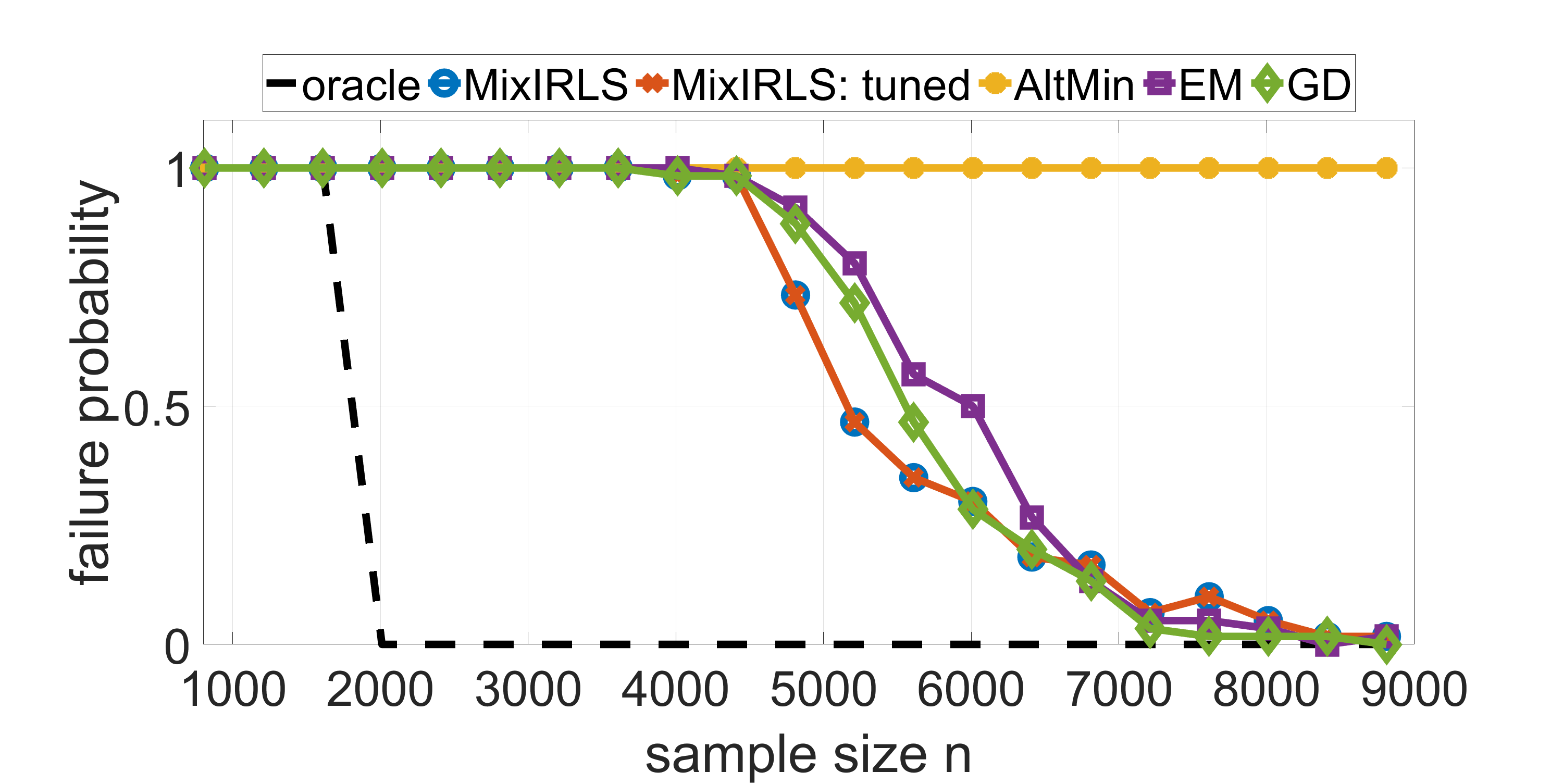}
	}
	\caption{Same setting as in \cref{fig:varying_n}, but with $p = (1/3, 1/3, 1/3)$.}
\label{fig:balanced_varying_n}
\end{figure}

\begin{figure}[t]
\centering
	\subfloat{
		\includegraphics[width=0.5\linewidth]{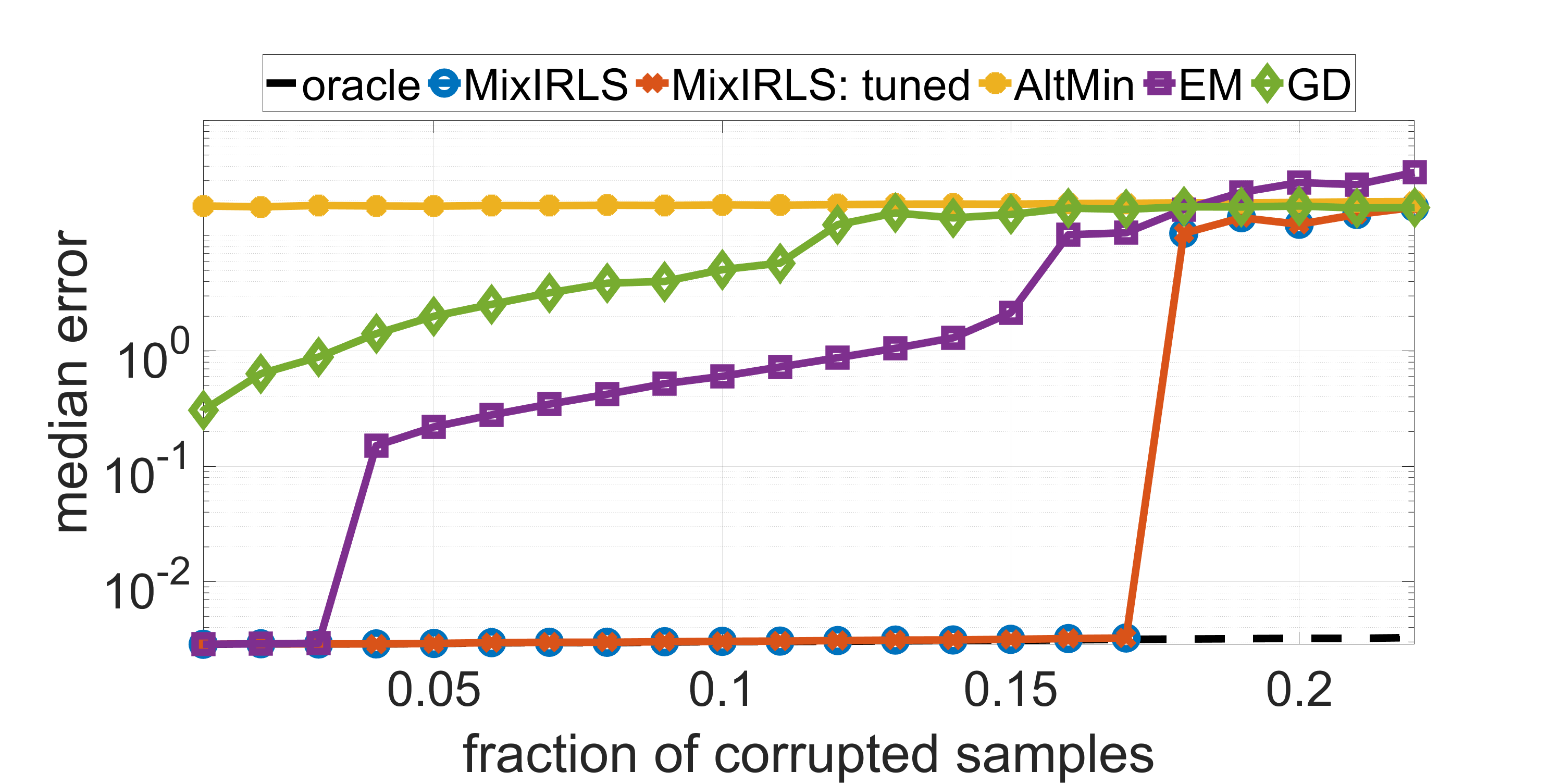}
	}
	\subfloat{
		\includegraphics[width=0.5\linewidth]{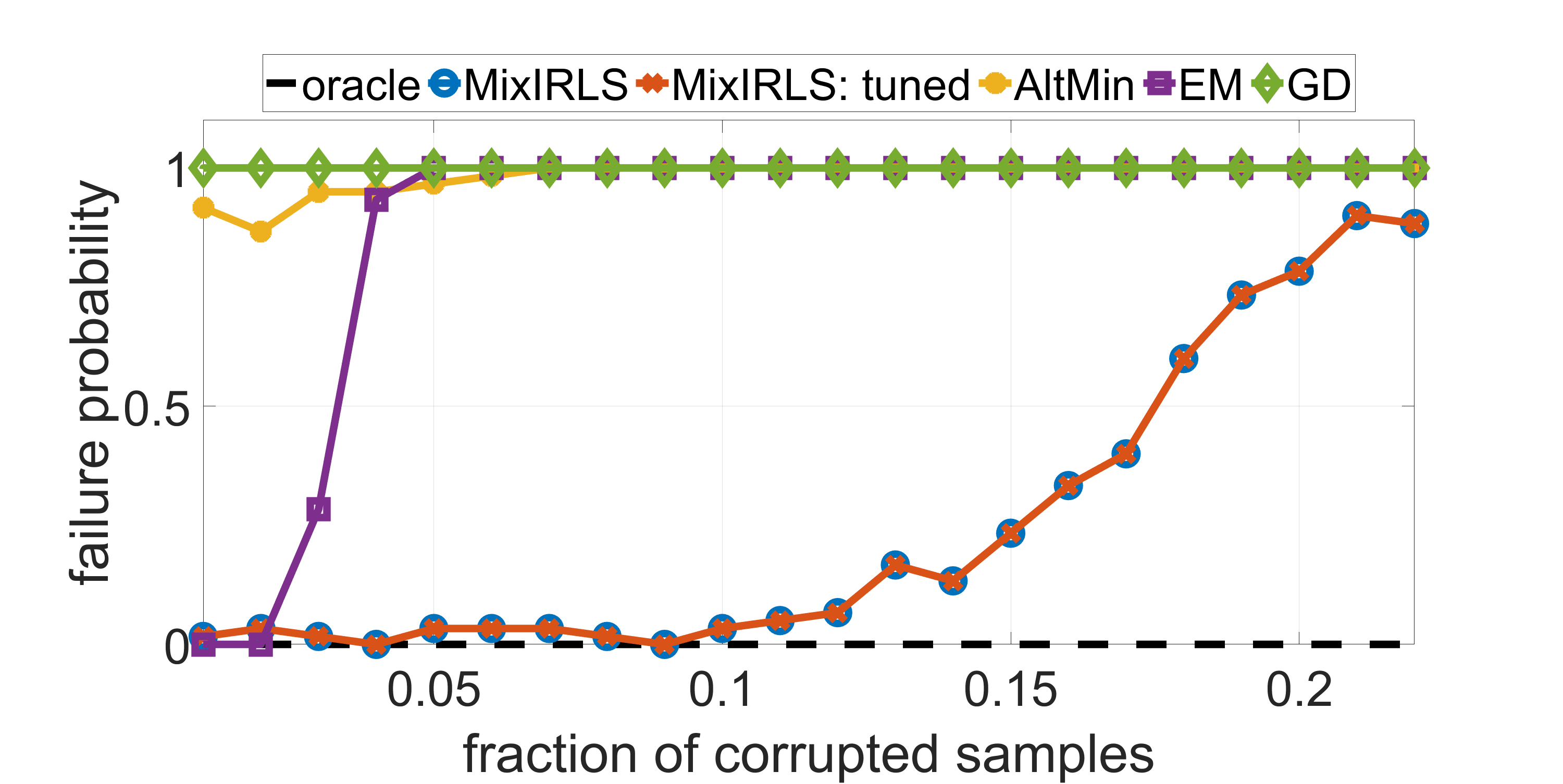}
	}
	\caption{Same setting as in \cref{fig:robustness}(left), but with $p = (1/3, 1/3, 1/3)$.}
\label{fig:balanced_corruptions}
\end{figure}

\begin{figure}[t]
\centering
	\subfloat{
		\includegraphics[width=0.5\linewidth]{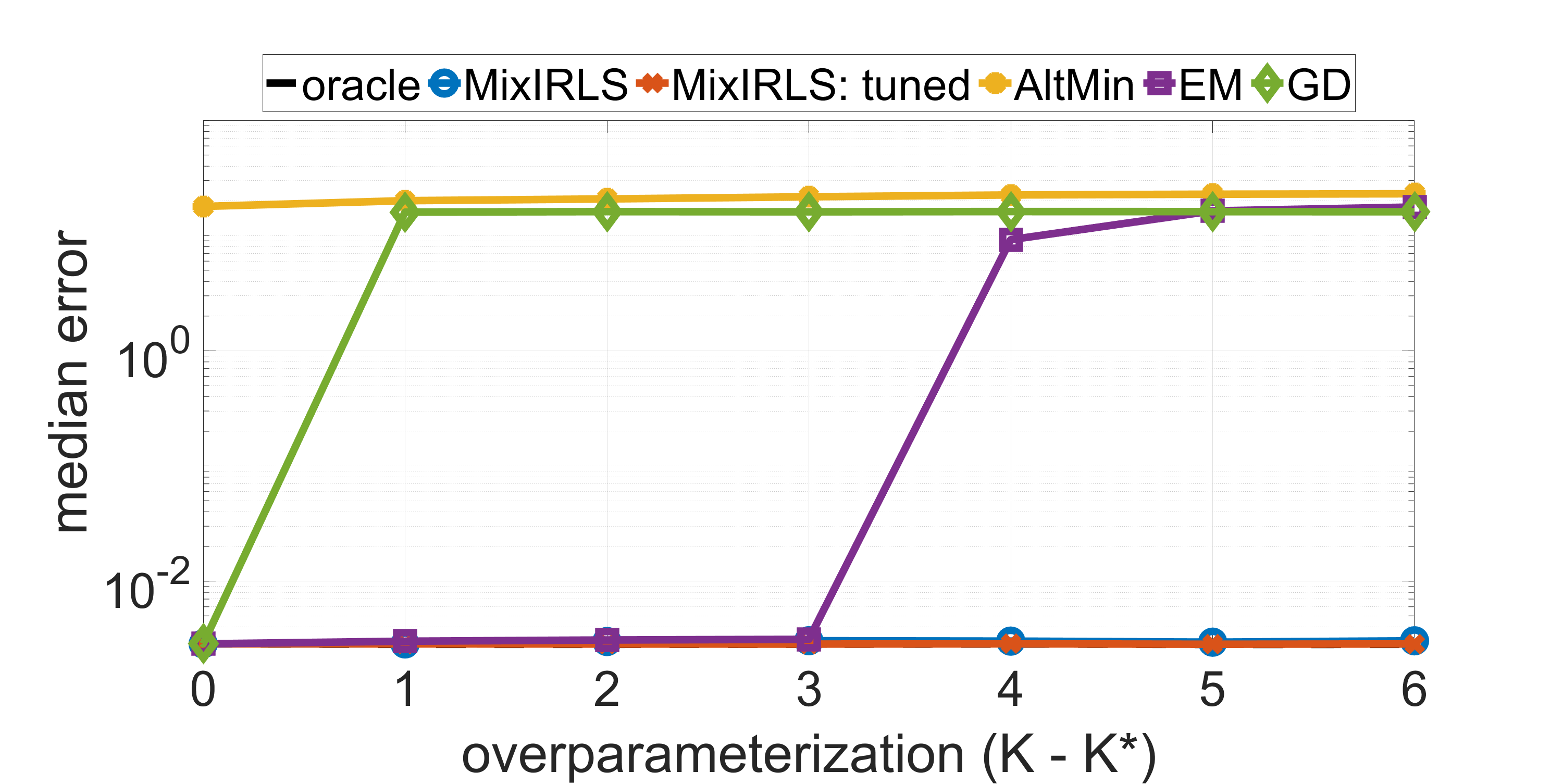}
	}
	\subfloat{
		\includegraphics[width=0.5\linewidth]{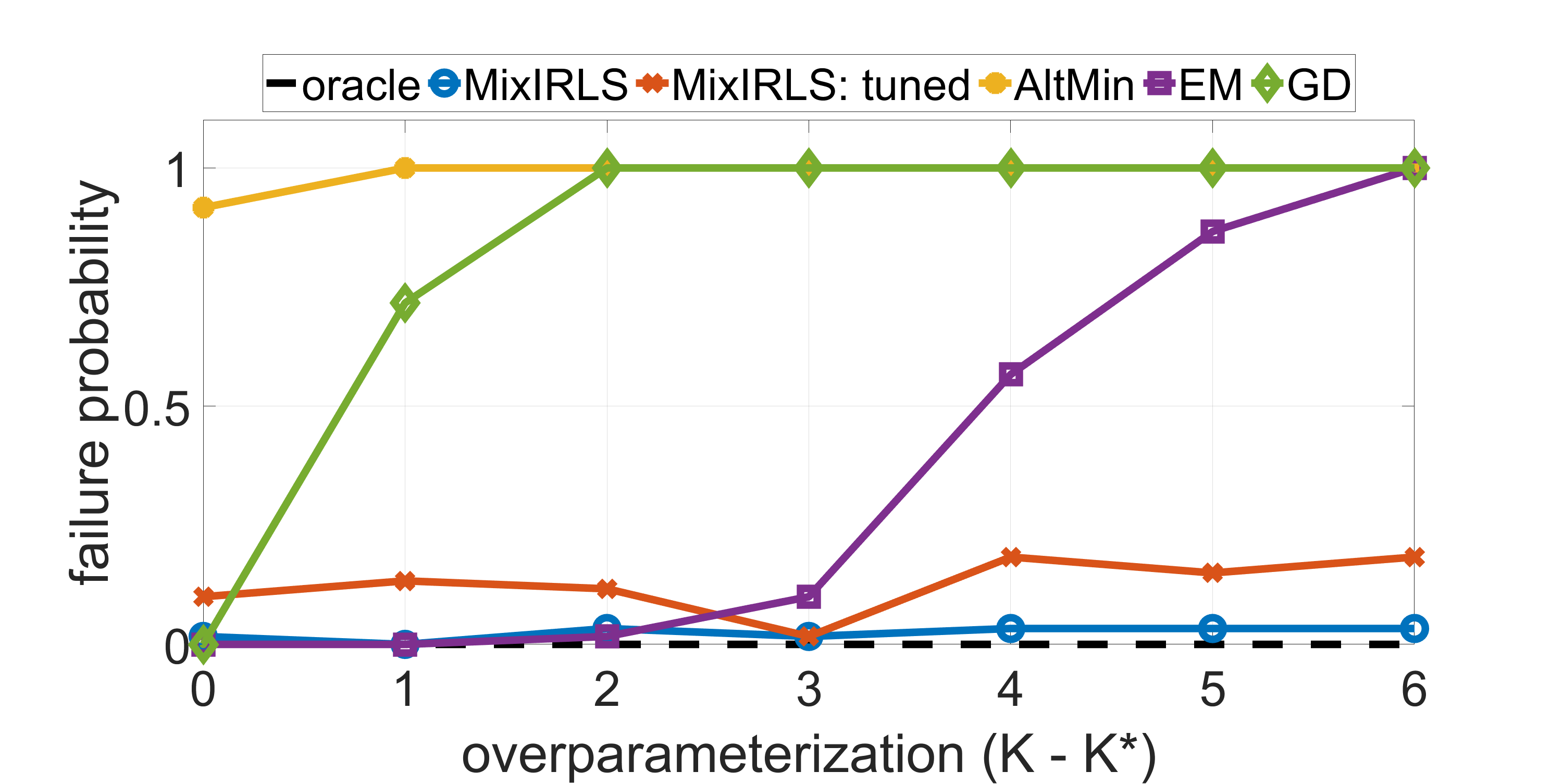}
	}
	\caption{Same setting as in \cref{fig:robustness}(right), but with $p = (1/3, 1/3, 1/3)$.}
\label{fig:balanced_overparam}
\end{figure}

\begin{figure}[t]
\centering
	\subfloat{
		\includegraphics[width=0.5\linewidth]{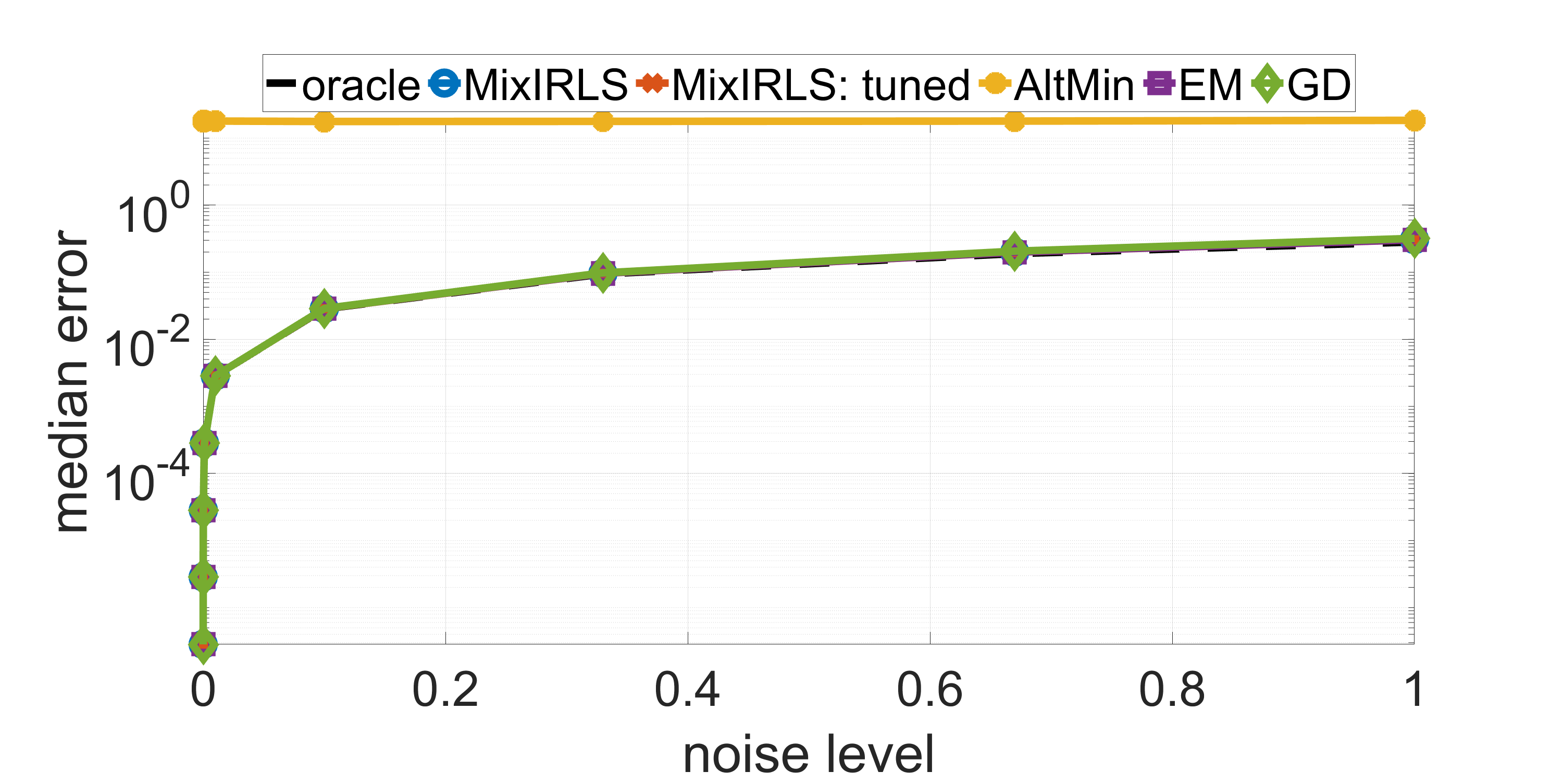}
	}
	\subfloat{
		\includegraphics[width=0.5\linewidth]{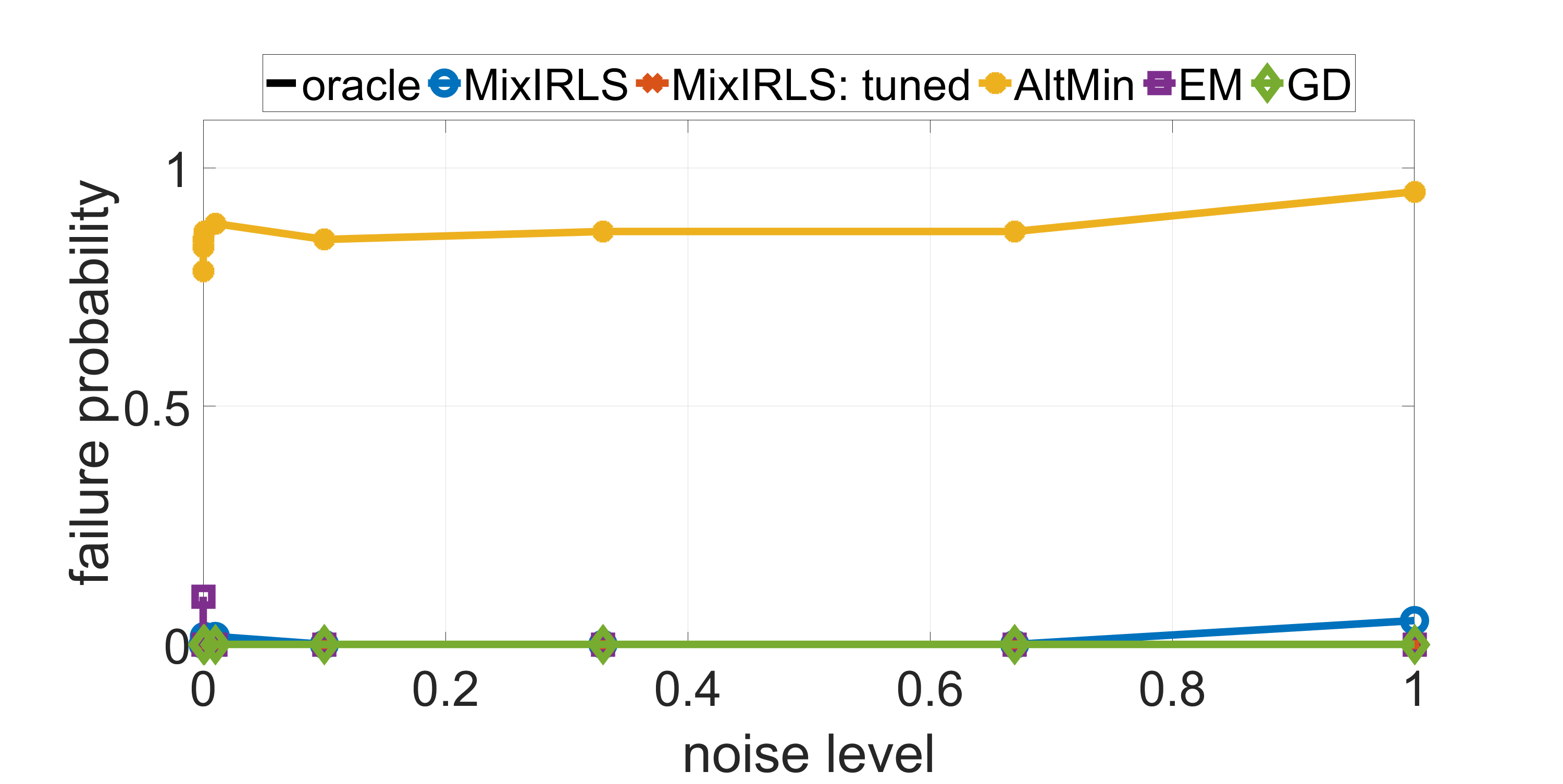}
	}
	\caption{Same setting as in \cref{fig:noise}, but with $p = (1/3, 1/3, 1/3)$.}
\label{fig:balanced_noise}
\end{figure}

\section{Additional Experimental Results}\label{sec:additional_experiments}
As discussed in the main text, \MIRLS finds $K=2$ components in the tone perception experiment given its default parameters (detailed in \cref{sec:additional_details}). With the value of the sensitivity parameter $w_\text{th}$ set to $0.1$ instead of $0.01$, \MIRLS finds $K=3$ components, as depicted in \cref{fig:tone_perception_outliersK3}(right). With $w_\text{th} = 0.5$, \MIRLS already finds $K=4$ components. \Cref{fig:tone_perception_outliersK3}(left) shows that equipped with a prior knowledge of $f = 3\%$ corruptions, \MIRLS identifies reasonable outliers. We note that this value of $f$ was chosen arbitrarily, and we do not know the true number of outliers in Cohen's data.

In \cref{fig:medical_wine}, we showed the median estimation errors of the algorithms for the medical insurance cost and the wine quality datasets. \Cref{fig:life_fish} shows the results for the WHO life expectancy and the fish market datasets. In addition, \cref{table:experiments_minErrors} shows the minimal estimation error across different random initializations for all four datasets. Interestingly, even though \AltMin performs consistently worse than \EM in terms of the median error, it outperforms it in terms of the minimal error. Moreover, in the special case of $K=2$, \AltMin achieves the lowest minimal error in all four datasets. However, in general, the tuned variant of \MIRLS outperforms the compared methods, including \AltMin, also in terms of the minimal error.

\begin{figure*}[ht]
\centering
	\subfloat{
		\includegraphics[width=0.5\linewidth]{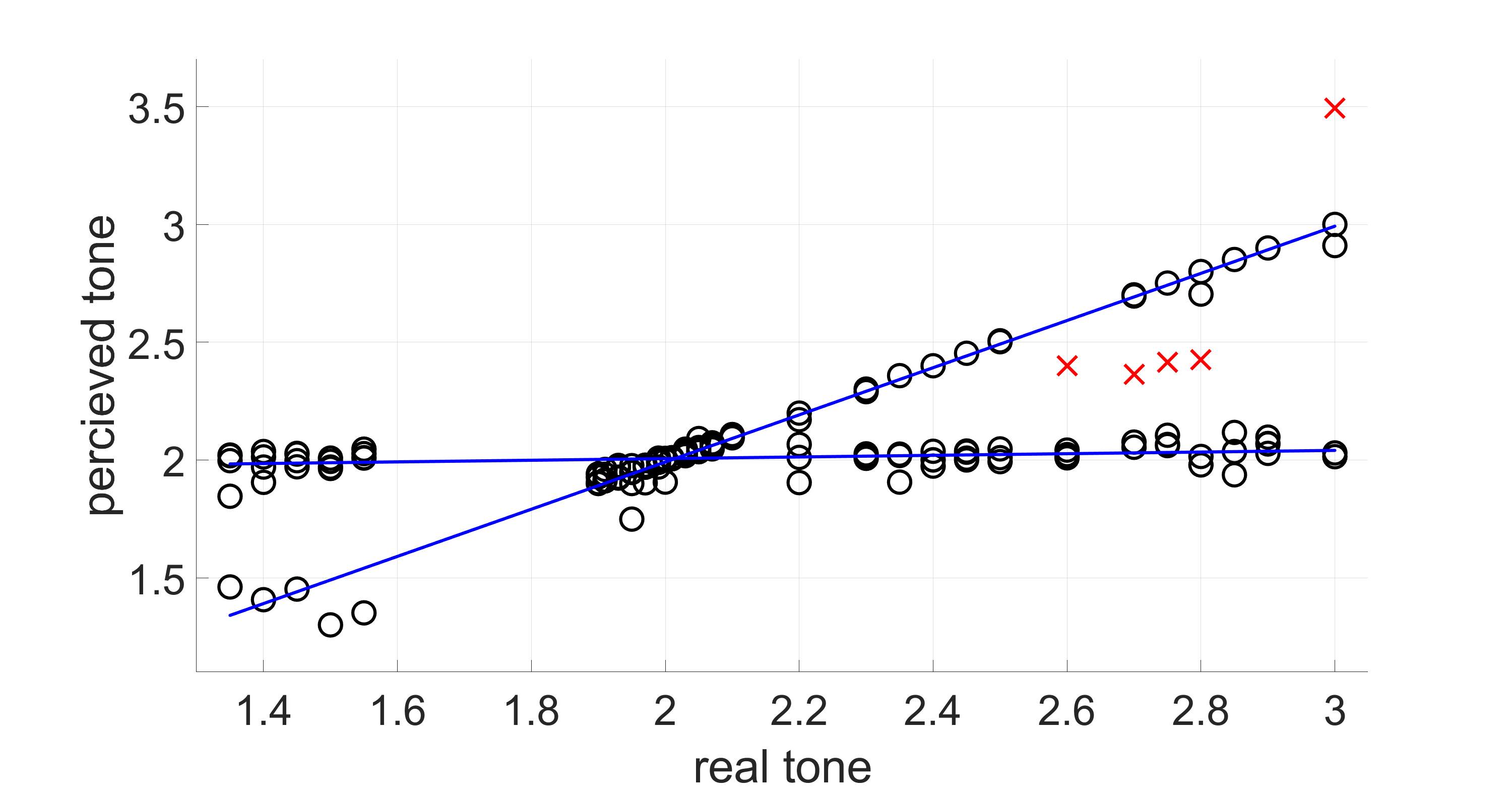}
	}
	\subfloat{
		\includegraphics[width=0.5\linewidth]{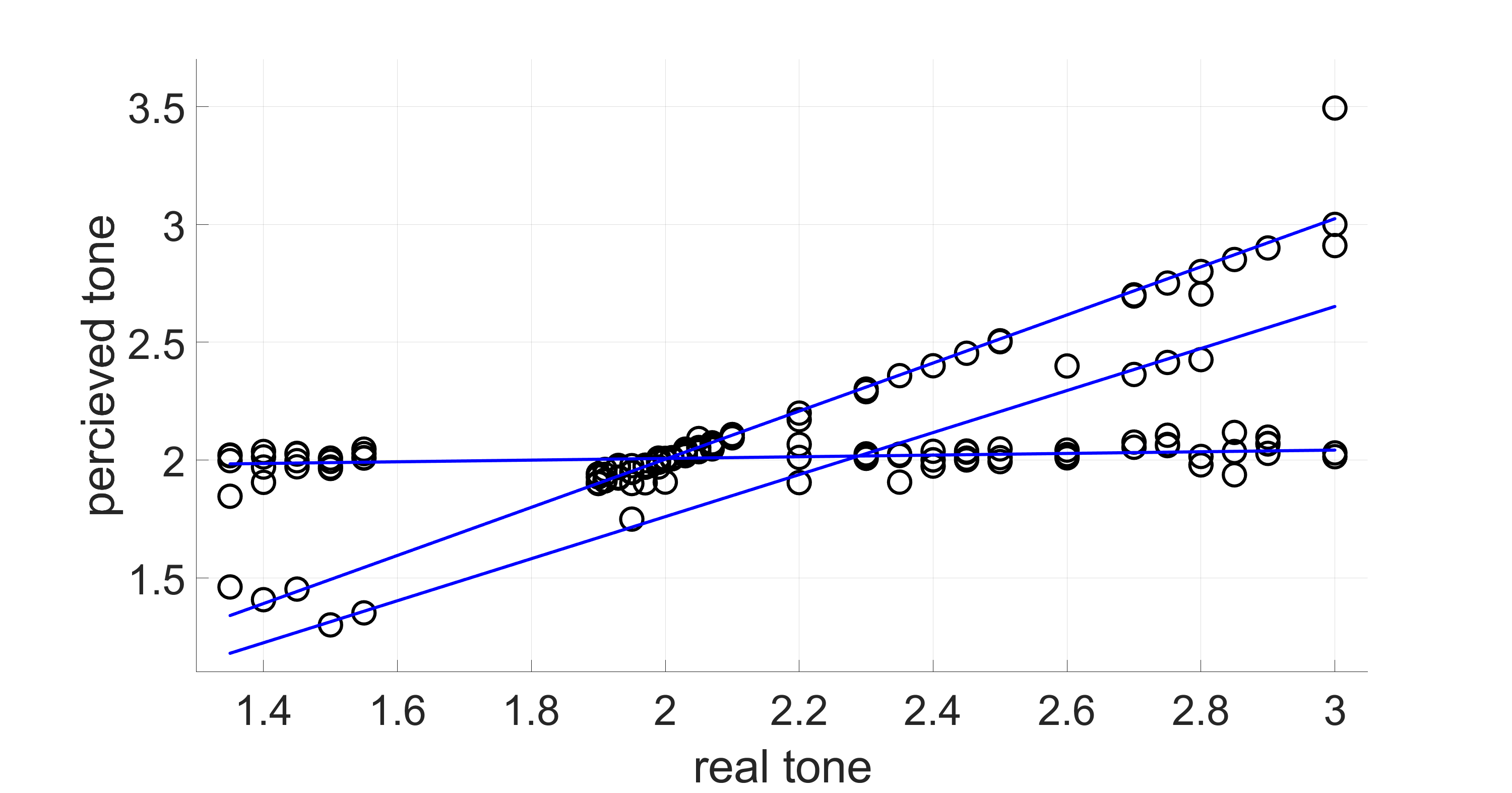}
	}
	\caption{Same as \cref{fig:tone_perception}, but with a given corruption fraction of $0.03$ (left panel), and with an increased value of the parameter $w_\text{th}$ (right panel). Marked with red X are outliers identified by \MIRLS.}
\label{fig:tone_perception_outliersK3}
\end{figure*}

\begin{figure}[t]
\centering
	\subfloat{
		\includegraphics[width=0.5\linewidth]{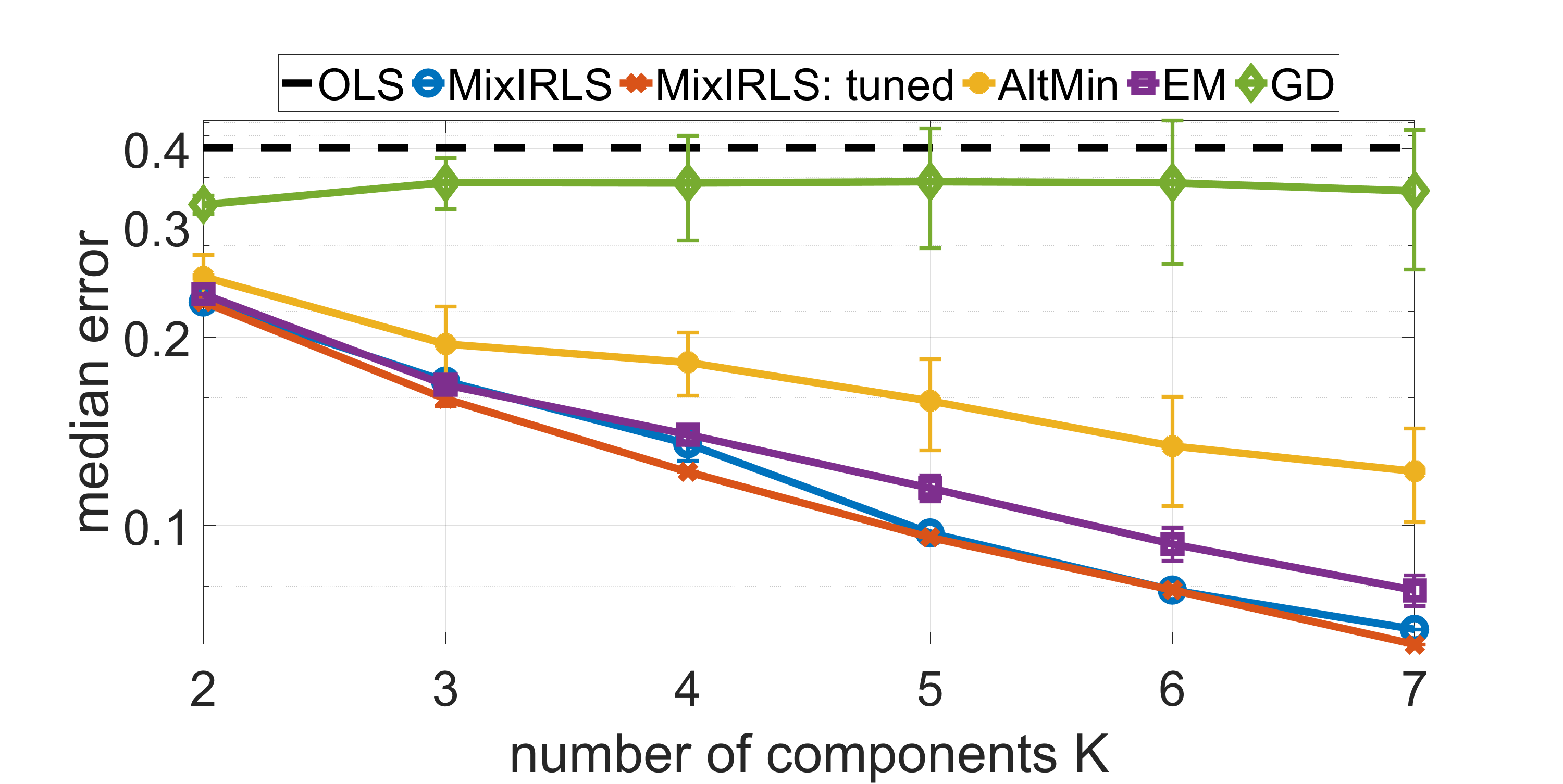}
	}
	\subfloat{
		\includegraphics[width=0.5\linewidth]{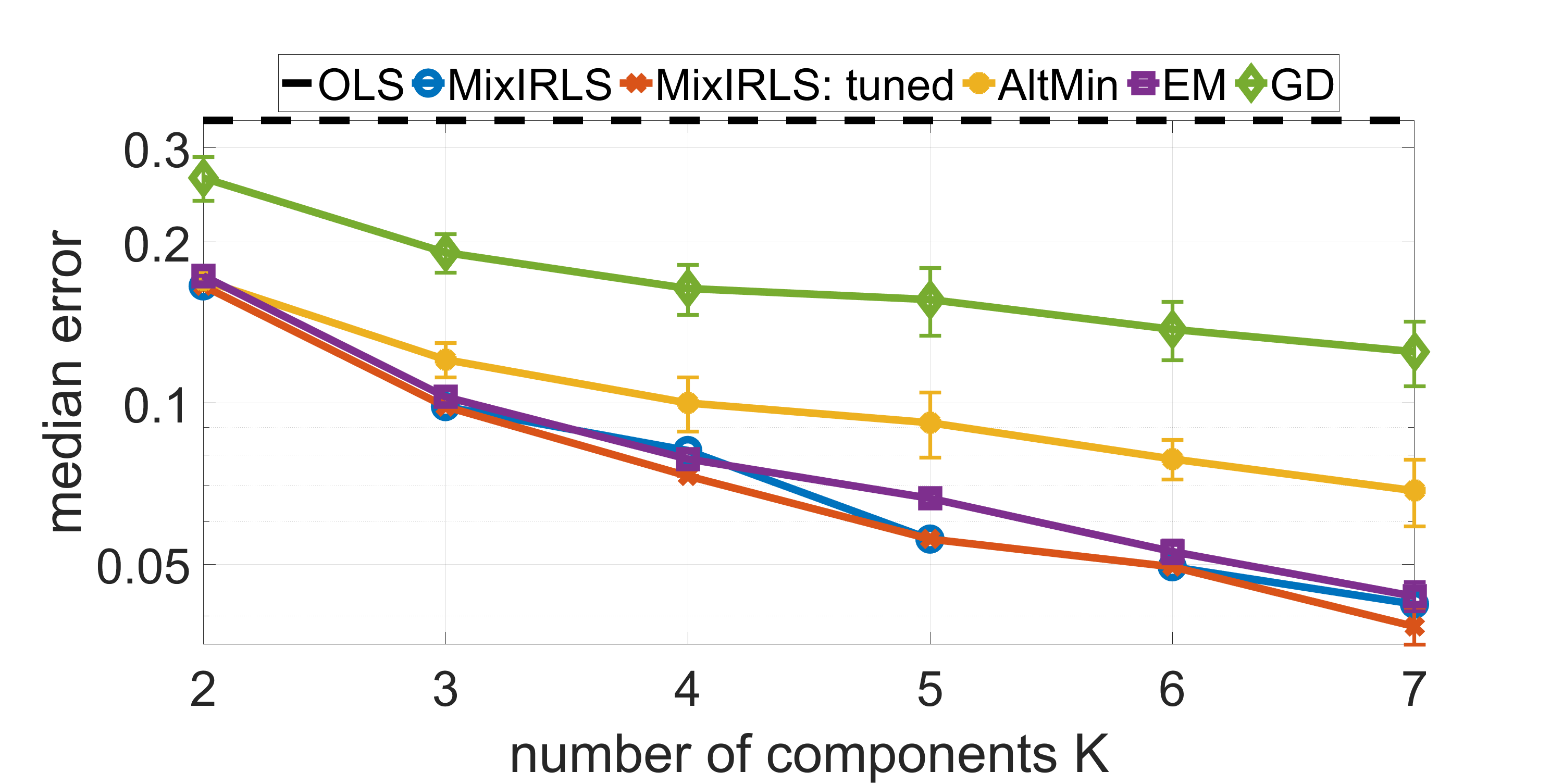}
	}
	\caption{Same as in \cref{fig:medical_wine}, but on the WHO life expectancy and the fish market datasets.}
\label{fig:life_fish}
\end{figure}

\renewcommand{\arraystretch}{1.1}
\begin{table*}[ht]
\caption{The minimal error, as defined in \eqref{eq:objective_observed}, achieved by several MLR algorithms, across 50 realizations of the experiment, each with a different random initialization. The corresponding median errors are depicted in \cref{fig:medical_wine,fig:life_fish}.}
\centering
 \begin{tabular}{c || c | c c c c c c}
Dataset & Algorithm & $K=2$ & $K=3$ & $K=4$ & $K=5$ & $K=6$ & $K=7$ \\
 \hline
 medical & 
	\MIRLS   & 0.1594 & 0.1086 & \textbf{0.0725} & 0.0579 & 0.0450 & \textbf{0.0428} \\ &
	\MIRLST           & 0.1594 & \textbf{0.0905} & \textbf{0.0725} & \textbf{0.0538} & \textbf{0.0435} & 0.0430 \\ &
	\AltMin          & \textbf{0.1591} & 0.0950 & 0.0900 & 0.0653 & 0.0638 & 0.0628 \\ &
	\EM              & 0.1598 & 0.1230 & 0.0817 & 0.0699 & 0.0589 & 0.0438 \\ &
	\GD              & 0.2676 & 0.2567 & 0.3063 & 0.2767 & 0.2458 & 0.2920 \\ 
 \hline
 wine & 
	\MIRLS          & 0.4827 & 0.3836 & 0.2856 & 0.0802 & 0.0587 & 0.0108 \\ &
	\MIRLST           & 0.4827 & \textbf{0.2974} & \textbf{0.1437} & \textbf{0.0764} & 0.0438 & \textbf{0.0000} \\ &
	\AltMin          & \textbf{0.4747} & \textbf{0.2974} & \textbf{0.1437} & 0.0776 & \textbf{0.0311} & \textbf{0.0000} \\ &
	\EM              & 0.5593 & 0.3490 & 0.1857 & 0.1100 & 0.0485 & 0.0233 \\ &
	\GD              & 0.5852 & 0.5043 & 0.4592 & 0.4048 & 0.3974 & 0.3391 \\ 
 \hline
 WHO & 
	\MIRLS   & 0.2276 & 0.1604 & \textbf{0.1201} & 0.0973 & 0.0789 & 0.0663 \\ &
	\MIRLST           & 0.2276 & \textbf{0.1517} & 0.1213 & \textbf{0.0928} & 0.0789 & \textbf{0.0646} \\ &
	\AltMin          & \textbf{0.2246} & 0.1610 & 0.1272 & 0.1042 & 0.0974 & 0.0830 \\ &
	\EM              & 0.2315 & 0.1604 & 0.1228 & 0.0984 & \textbf{0.0776} & 0.0653 \\ &
	\GD              & 0.2990 & 0.2682 & 0.2281 & 0.2400 & 0.2008 & 0.1946 \\ 
 \hline
 fish & 
	\MIRLS   & 0.1656 & \textbf{0.0985} & 0.0808 & \textbf{0.0557} & \textbf{0.0452} & \textbf{0.0318} \\ &
	\MIRLST           & 0.1656 & \textbf{0.0985} & \textbf{0.0731} & \textbf{0.0557} & \textbf{0.0452} & 0.0327 \\ &
	\AltMin          & \textbf{0.1637} & \textbf{0.0985} & 0.0784 & 0.0703 & 0.0551 & 0.0504 \\ &
	\EM              & 0.1729 & 0.1027 & 0.0771 & 0.0596 & 0.0473 & 0.0362 \\ &
	\GD              & 0.1814 & 0.1511 & 0.1188 & 0.1221 & 0.0974 & 0.0887 \\ 
 \hline
 \end{tabular}
 \label{table:experiments_minErrors}
\end{table*}


\end{document}